\documentclass[twoside]{article}

\usepackage[accepted]{aistats2021}
\usepackage[round]{natbib}

\usepackage[utf8]{inputenc} 
\usepackage[T1]{fontenc}    
\usepackage{hyperref}       
\usepackage{url}            
\usepackage{booktabs}       
\usepackage{amsfonts}       
\usepackage{nicefrac}       
\usepackage{microtype}      

\usepackage{graphicx}
\usepackage{subcaption}
\usepackage{amsmath}
\usepackage{amssymb}
\usepackage{color}
\usepackage{amsthm}

\usepackage{algorithm}
\usepackage{algorithmicx}
\usepackage[noend]{algpseudocode}

\usepackage[group-separator={,}, group-minimum-digits={3}]{siunitx}

\usepackage{thmtools} 
\usepackage{thm-restate}

\newtheorem{definition}{Definition}
\newtheorem{assumption}{Assumption}

\newtheorem{fact}{Fact}
\newtheorem{lemma}{Lemma}
\newtheorem{corollary}{Corollary}

\newcommand{\hatZunsafe}{\hat{Z}_{\text{unsafe}}}
\newcommand{\hatZsafe}{\hat{Z}_{\text{safe}}}
\newcommand{\Zsafe}{Z_{\text{safe}}}
\newcommand{\Zexplore}{{Z}_{\text{explore}}}
\newcommand{\Znext}{{Z}_{\text{next}}}
\newcommand{\Zreturn}{{Z}_{\text{return}}}
\newcommand{\Zreturnable}{{Z}_{\text{returnable}}}
\newcommand{\Zclosed}{{Z}_{\text{closed}}}
\newcommand{\Zreachable}{{Z}_{\text{reachable}}}
\newcommand{\Zcandidate}{{Z}_{\text{candidate}}}

\newcommand{\Zbad}{{Z}_{\text{bad}}}
\newcommand{\tildeZgood}{\tilde{Z}_{\text{good}}}
\newcommand{\tildeZbad}{\tilde{Z}_{\text{bad}}}
\newcommand{\Zedge}{{Z}_{\text{edge}}}
\newcommand{\optZgoal}{\opt{Z}_{\text{goal}}}

\newcommand{\Hcom}{H_{\text{com}}}
\newcommand{\picom}{\pi_{\text{com}}}

\newcommand{\Mgoal}{M_{\text{goal}}}
\newcommand{\Mexplore}{M_{\text{explore}}}
\newcommand{\Mswitch}{M_{\text{switch}}}
\newcommand{\optMgoal}{\opt{M}_{\text{goal}}}
\newcommand{\optMexplore}{\opt{M}_{\text{explore}}}
\newcommand{\optMswitch}{\opt{M}_{\text{switch}}}
\newcommand{\optTgoal}{\opt{T}_{\text{goal}}}
\newcommand{\pisafe}{{\pi}_{\text{safe}}}

\newcommand{\pireach}{{\pi}_{\text{reach}}}
\newcommand{\pireturn}{{\pi}_{\text{return}}}
\newcommand{\pivisit}{{\pi}_{\text{visit}}}

\newcommand{\pigoal}{\pi_{\text{goal}}}

\newcommand{\optpiexplore}{\opt{\pi}_{\text{explore}}}
\newcommand{\optpigoal}{\opt{\pi}_{\text{goal}}}
\newcommand{\optpiswitch}{\opt{\pi}_{\text{switch}}}
\newcommand{\optpi}{\opt{\pi}}
\newcommand{\optrhogoal}{\opt{\rho}_{\text{goal}}}
\newcommand{\sinit}{s_{\text{init}}}

\newcommand{\optQexplore}{\opt{Q}_{\text{explore}}}
\newcommand{\optQgoal}{\opt{Q}_{\text{goal}}}
\newcommand{\optQswitch}{\opt{Q}_{\text{switch}}}

\newcommand{\Rexplore}{{R}_{\text{explore}}}
\newcommand{\Rgoal}{{R}_{\text{goal}}}
\newcommand{\Rswitch}{{R}_{\text{switch}}}

\renewcommand{\Pr}{\mathrm{P}}

\makeatletter
\newcommand{\oset}[3][0.2ex]{%
  \mathrel{\mathop{#3}\limits^{
    \vbox to#1{\kern-2\ex@
    \hbox{$#2$}\vss}}}}
\makeatother
\newcommand{\del}[1]{\oset{\scriptscriptstyle\Delta}{#1}}
\newcommand{\opt}[1]{\overline{#1}} 
\newcommand{\delT}{\del{T}}

\newcommand\numberthis{\addtocounter{equation}{1}\tag{\theequation}}
\usepackage{accents}

\usepackage{multicol}

\newcommand{\pluseq}{\mathrel{{+}{=}}}

\begin{document}

%

%

\twocolumn[

\aistatstitle{Provably Safe PAC-MDP Exploration Using Analogies}

\aistatsauthor{ Melrose Roderick \And Vaishnavh Nagarajan \And J. Zico Kolter }

\aistatsaddress{ Carnegie-Mellon University } ]

\begin{abstract}
A key challenge in applying reinforcement learning to safety-critical domains is understanding how to balance exploration (needed to attain good performance on the task) with safety (needed to avoid catastrophic failure).
Although a growing line of work in reinforcement learning has investigated this area of ``safe exploration,'' most existing techniques either 1) do not guarantee safety during the actual exploration process; and/or 2) limit the problem to a priori known and/or deterministic transition dynamics with strong smoothness assumptions. 
Addressing this gap, we propose Analogous Safe-state Exploration (ASE), an algorithm for provably safe exploration in Markov Decision Processes (MDPs) with unknown, stochastic dynamics. Our method exploits analogies between state-action pairs to safely learn a near-optimal policy in a PAC-MDP (Probably Approximately Correct-MDP) sense.
Additionally, ASE also guides exploration towards the most task-relevant states, which empirically results in significant improvements in terms of sample efficiency, when compared to existing methods.
Source code for the experiments is available at \url{https://github.com/locuslab/ase}.
\end{abstract}

\section{Introduction}
Imagine you are Phillipe Petit in 1974, about to make a tight-rope walk between two thousand-foot-tall buildings.
There is no room for error.
You would want to be certain that you could successfully walk across without falling.
And to do so, you would naturally want to practice walking a tightrope on a similar length of wire, but only a few feet off the ground, where there is no real danger.

This example illustrates one of the key challenges in applying reinforcement learning (RL) to safety-critical domains, such as autonomous driving or healthcare, where a single mistake could cause significant harm or even death.
While RL algorithms have been able to significantly improve over human performance on some tasks in the average case, most of these algorithms do not provide any guarantees of safety either during or after training, making them too risky to be used in real-world, safety-critical domains.

Taking inspiration from the tight-rope example, we propose a new approach to safe exploration in reinforcement learning.
Our approach, Analogous Safe-state Exploration (ASE), seeks to explore state-action pairs that are analogous to those along the path to the goal, but are guaranteed to be safe.
Our work fits broadly into the context of a great deal of recent work in safe reinforcement learning, but compared with past work, our approach is novel in that 1) it guarantees safety during exploration in a stochastic, unknown environment (with high probability), 2) it finds a near-optimal policy in a PAC-MDP (Probably Approximately Correct-MDP) sense, and 3) it guides exploration to focus only on state-action pairs that provide necessary information for learning the optimal policy.
Specifically, in our setting we assume our agent has access to a set of initial state-action pairs that are guaranteed to be safe and a function that indicates the similarity between state-action pairs.
Our agent constructs an optimistic policy, following this policy only when it can establish that this policy won't lead to a dangerous state-action pair.
Otherwise, the agent explores state-action pairs that inform the safety of the optimistic path.

In conjunction with proposing this new approach, we make two main contributions.
First, we prove that ASE guarantees PAC-MDP optimality, and also safety of the {\em entire} training trajectory, with high probability.
To the best of our knowledge, this is the first algorithm with this two-fold guarantee in 
 stochastic environments.
Second, we evaluate ASE on two illustrative MDPs, and show empirically that our proposed approach substantially improves upon existing PAC-MDP methods, either in safety or sample efficiency, as well as existing methods modified to guarantee safety.




\section{Related Work}

\textbf{Safe reinforcement learning.}
Many safe RL techniques either require sufficient prior knowledge to guarantee safety a priori or promise safety only during deployment and not during training/exploration.
Risk-aware control methods \citep{fleming95risk, blackmore10control, ono15riskaware}, for example, can compute safe control policies even in situations where the state is not known exactly, but require the dynamics to be known a priori.
Similar works \citep{perkins02lyapunov, hans08safeexpl} allow for learning unknown dynamics, but assume sufficient prior knowledge to determine safety information before exploring.
Constrained-MDPs (C-MDPs) \citep{altman1999constrained, achiam17cpo, taleghan2018efficient} and Robust RL \citep{wiesemann2013robust, lim2013reinforcement, nilim2005robust, ostafew15robust, aswani13provably}, on the other hand, are able to learn the dynamics and safety information, but promise safety only during deployment.
Additionally, there are complications with using C-MDPs, such as optimal policies being stochastic and the constraints only holding for a subset of states \citep{taleghan2018efficient}.

Other works \citep{ wachi18safeexpl,berkenkamp17stability, turchetta16safeexpl, akametalu14reachability, moldovan12safeexpl} do consider the problem of learning on unknown environments while also ensuring safety throughout training.
Indeed,
both our work and these works
 rely on a notion of similarity between state-action pairs in order to gain critical safety knowledge. For example, \citep{ wachi18safeexpl,berkenkamp17stability, turchetta16safeexpl, akametalu14reachability}
make assumptions about the regularity of the transition or safety functions of the environment,
which allows them to model the uncertainty in these functions using Gaussian Processes (GPs).
Then, by examining the worst-case estimate of this model, they guarantee safety on continuous environments. 
We also refer the reader to a rich line of work outside of the safety literature that has studied similarity metrics in RL, 
in order to improve computation time of planning and sample complexity of exploration \citep{givan2003equivalence, taylor2009bounding, abel2017near, kakade2003exploration,taiga2018approximate} (see Appendix~\ref{sec:metrics} for more discussion).

However, there are two key differences between the above works \citep{ wachi18safeexpl,berkenkamp17stability, turchetta16safeexpl, akametalu14reachability,moldovan12safeexpl}  and ours.
First, (with the exception of \cite{wachi18safeexpl}), these approaches are not reward-directed,
and instead focus on only exploring the state-action space as much as possible. 
Second, although the GP-based methods 
\citep{ wachi18safeexpl,berkenkamp17stability, turchetta16safeexpl, akametalu14reachability}
help capture uncertainty,
they do not model inherent stochasticity in the environment and instead assume the true transition function to be deterministic.
Accounting for this stochasticity presents significant algorithmic challenges (see Appendix~\ref{sec:methodology}).
To the best of our knowledge, \citet{moldovan12safeexpl} is the only work that tackles learning on environments with unknown, stochastic dynamics.
Their method guarantees safety by ensuring there always exists a policy to return to the start state.
They do, however, assume that the agent knows a-priori a function that can compute the transition dynamics given observable attributes of the states.
Our method, however, does not assume known transition functions. 
Instead it learns the dynamics of state-actions that it has established to be safe, and extends this knowledge to potentially unsafe
state-actions.


\textbf{PAC-MDP learning.}
Sample efficiency bounds for RL fall into two main categories: 1) regret \citep{jaksch2010near} and 2) PAC (Probably Approximately Correct) bounds \citep{strehl2009reinforcement, fiechter1994efficient}.
For our analysis we use the PAC, specifically PAC-MDP \citep{strehl2009reinforcement}, framework.
PAC-MDP bounds bound the number of $\epsilon$-suboptimal steps taken by the learning agent.
PAC-MDP bounds have been shown for many popular exploration techniques, including R-Max \citep{brafman2002r} and a slightly modified Q-Learning \citep{strehl2006pac}.
While R-Max is PAC-MDP, it explores the state-action space exhaustively, which can be inefficient in large domains.
Another PAC-MDP algorithm, Model-Based Interval Estimation (MBIE) \citep{strehl2008analysis}, outperforms the sample efficiency of R-Max by only exploring states that are potentially along the path to the goal.
Our work seeks to extend this algorithm to safety-critical domains and guarantee safety during exploration.

\section{Problem Setup}

We model the environment as a Markov Decision Process (MDP), a 5-tuple $\langle S, A, R, T, \gamma \rangle$ with discrete, finite sets of states $S$ and actions $A$, a {\em known, deterministic} reward function $R: S \times A \to \mathbb{R}$,  an {\em unknown, stochastic} dynamics function $T: S \times A \to \mathcal{P}_{S}$ which maps a state-action pair to a probability distribution over next states, and discount factor $\gamma \in (0,1)$.
We assume the environment has a fixed initial state and denote it by $\sinit$.
We also assume that the rewards are known a-priori and bounded between $-1$ and $1$; the rewards that are negative denote dangerous state-actions.\footnote{Note that our work can be easily extended to have a separate safety and reward function, but we have combined them in this work for convenience.}
This of course means that the agent knows a priori what state-actions are ``immediately'' dangerous; but we must emphasize that the agent is still faced with the non-trivial challenge of learning about other a priori unknown state-actions that can be dangerous in the long-term -- we elaborate on this in the ``Safety'' section below.
Also note that while most RL literature does not assume the reward function is known, we think this is a reasonable assumption for many real-world problems where reward functions are constructed by engineers.
Moreover, this assumption is not uncommon in RL theory \citep{szita2010model, lattimore2012pac}.

\textbf{Analogies.}
As highlighted in \citet{turchetta16safeexpl}, some prior knowledge about the environment is required for ensuring the agent never reaches a catastrophic state.
In prior work, this knowledge is often provided as some notion of similarity between state-action pairs; for example, kernel functions used in previous work employing GPs to model dynamics or Lipschitz continuity assumptions placed on the dynamics.
Intuitively, such notions of similarity can be exploited to 
learn about unknown (and potentially dangerous) state-action pairs, by exploring a ``proxy'' state-action pair that is sufficiently similar and known to be safe.
Below, we define a notion of similarity, but the key difference between this and previous formulations is that ours also applies to stochastic environments. Specifically, we introduce the notation of \emph{analogies} between state-action pairs.
More concretely, the agent is given {\em an analogous state function}  
$\alpha: (S \times A \times S) \times (S \times A) \to S$ and a {\em pairwise state-action distance mapping} $\Delta: (S \times A) \times (S \times A) \to [0,2]$ such that, for any $(s,a,\tilde{s},\tilde{a}) \in (S \times A) \times (S \times A)$
\begin{align*}
    \sum_{s' \in S} |T(s,a,s') - T(\tilde{s},\tilde{a}, \alpha(\cdots, s'))| \leq 
    \Delta((s,a), (\tilde{s},\tilde{a}))
\end{align*}
where $\alpha(\cdots, s') = \alpha(s, a, s', \tilde{s}, \tilde{a})$ represents, intuitively, the next state that is ``equivalent'' to $s'$ for $(\tilde{s},\tilde{a})$.
\footnote{
Note that although $\alpha$ is defined for every $(s,a,s',\tilde{s},\tilde{a})$ tuple, not all such pairs of state-actions need be analogous to each other.
In such cases, we can imagine that $\alpha(s,a,s',\tilde{s},\tilde{a})$ maps to a dummy state and $\Delta((s,a), (\tilde{s},\tilde{a})) = 1$.
}
In other words, for any two state-action pairs, we are given a bound on the $L_1$ distance between their dynamics: one that is based on a mapping between analogous next states.
The hope is that $\alpha$ can provide a much more useful analogy than a naive identity mapping between the respective next states.


To provide some intuition, recall the tight-rope walker example mentioned in the introduction:
The tight-rope walker agent must cross a dangerous tight-rope, but wants to guarantee it can do so safely.
In this situation, the agent has a ``practice'' tight-rope (that's only a few feet off the ground) and a ``real'' tight-rope.
In this example, if the agent is in a particular position and takes a particular action on either tight-rope, the change in its position will be the same regardless of what rope it was on.\footnote{For this example, we assume the dynamics of the agent on the practice and real tight-ropes are identical, but small differences could be captured by making the $\Delta(\cdot, \cdot, \cdot, \cdot)$ function non-zero.}
This analogy between the two ropes can be mathematically captured as follows.
Consider representing the agent's state as a tuple of the form
$(\text{prac}, x)$ or $(\text{real}, x)$ where the first element denotes which of the two ropes the agent is on, and the second element denotes the position within that rope. Then for any action $a$, we can say that $\alpha((\text{prac}, x), a, (\text{prac}, x'), (\text{real}, x), a) = (\text{real}, x')$  and $\Delta(((\text{prac}, x), a), ((\text{real}, x), a) = 0$.  This would thus imply that by learning a model of the dynamics in the practice setting, and the corresponding optimal policy, an agent would still be able to learn a policy that guaranteed safety even in the real setting.



\textbf{State-action sets.}
For simplicity, for any set of state-action pairs $Z \subset S \times A$, we say that $s \in Z$ if there exists any $a \in A$ such that $(s,a) \in Z$.
Also, we say that $(s,a)$ is an {\bf edge} of $Z$ if $(s,a) \notin Z$ but $s \in Z$.
We use $\Pr[ \cdot \vert \pi]$ to denote the probability of an event occurring while following a policy $\pi$.
%
%
\begin{definition}
We say that $Z \subset S \times A$ is \textbf{closed} if for every $(s,a) \in Z$ and for every next $s'$ for which $T(s,a,s') > 0$, there exists $a'$ such that $(s',a') \in Z$.
\end{definition}
Intuitively, if a set $Z$ is closed, then we know that if the agent starts at a state in $Z$ and follows a policy $\pi$ such that for all $s \in Z$, $(s,\pi(s)) \in Z$, then we can guarantee that the agent never exits $Z$ (see Fact~\ref{fact:closed_policy} in Appendix~\ref{sec:facts}).
We will use $\pi \in \Pi(Z)$ to denote that, for all $s \in Z$, $(s,\pi(s)) \in Z$.
%



\begin{definition}
\label{def:communicating}
A subset of state-action pairs, $Z \subset S \times A$ is said to be \textbf{communicating} if  $Z$ is closed and for any $s' \in Z$, there exists a policy $\pi_{s} \in \Pi(Z)$ such that
$\forall s, \ \Pr[\exists  t, \ s_t = s' \ | \pi_{s'}, s_0 = s] = 1.$
\end{definition}

In other words, every two states in $Z$ must be reachable through a policy that never exits the subset $Z$.
Note that this definition is equivalent to the standard definition of communicating when $Z$ is the set of all state-action pairs in the MDP (see Appendix~\ref{sec:facts}).

\textbf{Safe-PAC-MDP.}
One of the main objectives of this work is to design an agent that, with high probability (over all possible trajectories the agent takes), learns an optimal policy (in the PAC-MDP sense) while also never taking dangerous actions (i.e. actions with negative rewards) at any point along its arbitrarily long, trajectory.
This is a very strong notion of safety, but critical for assuring safety for long trajectories.
Such a strong notion is necessary in many real-world applications such as health and self-driving cars where a dangerous action spells complete catastrophe.

We formally state this notion of Safe-PAC-MDP below. The main difference between this definition and that of standard PAC-MDP is (a) the safety requirement on all timesteps and (b) instead of competing against an optimal policy (which could potentially be unsafe), the agent now competes with a ``safe-optimal policy'' (that we will define later). To state this formally, as in \citet{strehl2006pac}, let the trajectory of the agent until time $t$ be denoted by $p_t$ and let the value of the algorithm $\mathcal{A}$ be denoted by $V^{\mathcal{A}}(p_t)$ -- this equals the cumulative sum of rewards in expectation over all future trajectories (see Def~\ref{def:alg-value} in Appendix~\ref{sec:facts}). 

\begin{definition}
\label{def:safe_pac_mdp}
We say that an algorithm $\mathcal{A}$ is \textbf{Safe-PAC-MDP} if, for any $\epsilon, \delta \in (0, 1]$, with probability at least $1 - \delta$, $R(s_t, a_t) \geq 0$ for all timesteps $t$ and additionally, the sample complexity of exploration i.e., the number of timesteps $t$ for which $V^{\mathcal{A}}(p_t) > V^{\pisafe^*}(p_t) - \epsilon$,
is bounded by a polynomial in the relevant quantities, $(|S|, |A|, 1/\epsilon, 1/\delta, 1/(1-\gamma), 1/\tau, H_{\text{com}})$. 
Here, $\pisafe^*$ is the safe-optimal policy defined in Def.~\ref{def:pisafestar}, $\tau$ is the minimum non-zero transition probability (see Assumption \ref{as:negligible_transitions}) and $H_{\text{com}}$ is ``communication time'' (see Assumption~\ref{as:h_communicating}).
\end{definition}

We must emphasize that this notion of safety must {\em not} be confused with the weaker notion where one simply guarantees safety with high probability {\em at every step} of the learning process.
In such a case, 
for sufficiently long training trajectories, the agent is guaranteed to take a dangerous action i.e., with probability $1$, the trajectory taken by the agent will lead it to a dangerous action as $t \rightarrow \infty$.

\textbf{Safety.} 
Since our agent is provided the reward mapping, the agent knows a priori which  state-action pairs are ``immediately'' dangerous (namely, those with negative rewards). However, the agent is still faced with the challenge of determining
which actions may be dangerous {\em in the long run}: an action may momentarily yield a non-negative reward, but by taking that action, the agent may be doomed to a next state (or a future state) where all possible actions have negative rewards.
For example, at the instant when a tight-rope walker loses balance, they may experience a zero reward, only to eventually fall down and receive a negative reward.
In order to avoid such ``delayed danger'', below we define a natural notion of a safe set: a closed set of non-negative reward state-action pairs; as long as the agent takes actions within such a safe set, it will never find itself in a position where its only option is to take a dangerous action. Our agent will then aim to learn such a safe set; note that accomplishing this is non-trivial despite knowing the rewards, because of the unknown stochastic dynamics.


\begin{definition}
\label{def:safeset}
We say that $Z \subset S \times A$ is a \textbf{safe set} if $Z$ is {\em closed} and for all $(s,a) \in Z$, $R(s,a) \geq 0$. Informally, we also call every $(s,a) \in Z$ as a safe state-action pair.
\end{definition}


\subsection{Assumptions}


We will dedicate a fairly large part of our discussion below detailing the assumptions we make. Some of these are strong and we will explain why they are in fact required to guarantee PAC-MDP optimality in conjunction with the strong form of safety that we care about (being safe on all actions taken in an infinitely long trajectory) in an environment with unknown stochastic dynamics. 

First, in order to gain any knowledge of the world safely, the agent must be provided some prior knowledge about the safety of the environment.
Without any such knowledge (either in the form of a safe set, prior knowledge of the dynamics, etc) it is impossible to make safety guarantees about the first and subsequent steps of learning.
We provide this to the agent in the form of an initial safe set of state-action pairs, $Z_0$, that is also communicating.
We note that this kind of assumption is common in safe RL literature \citep{berkenkamp17stability, biyik19unknown}.
Additionally, we chose to make this set communicating so that the agent has the freedom to roam and try out actions inside $Z_0$ without getting stuck.

\begin{assumption}[Initial safe set]
\label{as:z_0}
The agent is initially given a safe, communicating set $Z_0 \subset S \times A$ such that $\sinit \in Z_0$.
\end{assumption}

In the PAC-MDP setting, we care about how well the agent's policy compares to an optimal policy over the whole MDP. However, in our setting, that would be an unfair benchmark since such an optimal policy might potentially travel through unsafe state-actions. To this end, we will first suitably characterize a safe set $\Zsafe$ and then set our benchmark to be the optimal policy confined to $\Zsafe$.


We begin by defining $\Zsafe$ to be the set of state-action pairs from which there exists some (non-negative-reward) return path to $Z_0$. 
Indeed, \textit{returnability} is a key aspect in safe reinforcement learning -- it has been similarly assumed in previous work \citep{moldovan12safeexpl, turchetta16safeexpl} and is also very similar to the notion of stability used to define safety in other works \citep{berkenkamp17stability, akametalu14reachability}.
Defining $\Zsafe$ in terms of returnability ensures that $\Zsafe$ does not contain any ``safe islands'' i.e., are safe regions that the agent can venture into, but without a safe way to exit.
At a high-level, this criterion helps prevent the agent from getting stuck in a safe island and acting sub-optimally forever.
The reasoning for why we need this assumption and traditional PAC-MDP algorithm do not is a bit nuanced and we discuss this in detail in Appendix \ref{sec:safe_islands}.

\begin{definition} \label{def:Zsafe}
We define $\Zsafe$ to be the set of state-action pairs $(s,a)$ such that $\exists \pireturn$ for which:
\[
\Pr\left[ \; \substack{\exists t \geq 0 \; s.t. (s_t, a_t) \in Z_0 \; \\ \& \; \forall t, R(s_t, a_t) \geq 0} |\pireturn, (s_0,a_0) = (s,a)\right]=1
\]
\end{definition}

Note that it follows that $\Zsafe$ is a safe set  (see Fact~\ref{fact:Zsafe} in Appendix~\ref{sec:facts}). 
Additionally, we will assume that $\Zsafe$ is communicating; note that given that all actions in $\Zsafe$ satisfy returnability to $Z_0$, this assumption is equivalent to assuming that all actions in $\Zsafe$ are also \textit{reachable} from $Z_0$. This is reasonable since we care only about the space of trajectories beginning from the initial state, which lies in $Z_0$. Having characterized $\Zsafe$ this way, we then define the safe optimal policy using $\Zsafe$.

\begin{assumption} [Communicatingness of safe set]
\label{as:zsafe_communicating} 
We assume $\Zsafe \subset S \times A$ is communicating. 
\end{assumption}





\begin{definition}
\label{def:pisafestar}
$\pisafe^*$ is a \textbf{safe-optimal} policy in that
\[\pisafe^* \in \arg\max_{\pi \in \Pi(\Zsafe)} \mathbb{E} \left[ \sum_{t=0}^\infty \gamma^t R(s_t,a_t) | \pi, s_0 = \sinit \right].\]
\end{definition}


Besides returnability, another important aspect in the safe PAC-MDP setting turns out to be the time it takes to travel between states.
In normal (unsafe) reinforcement learning settings, the agent can gather information from any state-action by experiencing it directly.
In this setting, on the other hand, not all state-actions can be experienced safely so the agent must indirectly gather information on a state-action pair of interest by experiencing an analogous state-action.
Thus, the agent must be able to visit the informative state-action and return to that state-action pair of interest in polynomial time.


To formalize this, we will assume that within any communicating subset of state-action pairs, we can ensure polynomial-time reachability between states, with non-negligible probability.
While this assumption is not made in the normal (unsafe) PAC-MDP setting, we emphasize that this assumption applies to a wide variety of real-world problems and is only violated in contrived examples, such as in a random-walk setting.
Specifically, to violate this assumption, there must be two state-action pairs in the safe set where the expected number of steps to move from one to the next is exponential in the state-action size.
This can only happen in random-walk-like scenarios where moving backwards has an equal (or higher) probability than moving forward, which occur very rarely in the real-world.
To make this last statement concrete, imagine a 1D grid where any action the agent takes leads it to either of the adjacent states with equal probability of $1/2$. Then, to reach a state that is $n$ steps away with probability $1/2$, it would take the agent, in expectation, an exponential number of steps in $n$, thus not satisfying Assumption 3.
However, if the probabilities of moving forward was $3/4$ for one action and moving backward was $1/4$ for the other action, then the agent could reach a state that is $n$ steps away with probability $1/2$ in less than $2n$ steps, satisfying Assumption 3.
\begin{assumption}(Poly-time communicating)
\label{as:h_communicating}
There exists $\Hcom  = \text{poly}(|S|, |A|)$ such that, 
for any communicating set $Z \subset S \times A$, and
 $\forall s' \in Z$, there exists a policy $\pi_{s'} \in \Pi(Z)$ for which
\[
\forall s, \ \Pr_M [\exists t \leq \Hcom , s_t = s' | \pi_{s'}, s_0 = s] \geq 1/2.
\]
\end{assumption}

Our next assumption is about the transition dynamics: we assume that we know a constant such that any transition either has zero probability or is larger than that known constant.
This assumption is {\em necessary} to perform learning under our strict safety constraints in {\em unknown, stochastic} dynamics.
Specifically, given this assumption, we can use finitely many samples to determine the support of a particular state-action pair's next state distribution. This is critical since
the agent cannot take an action unless it knows every possible next state that it could land in.
Note that, in a finite MDP, such a $\tau$ always exists, we simply assume we have a lower-bound on it.

\begin{assumption} [Minimum transition probability]
\label{as:negligible_transitions}
There exists a {\em known} $\tau > 0$ such that $\forall s,s' \in S$ and $a \in A$, $T(s,a,s') \in [0] \cup [\tau, 1]$. 
\end{assumption}

Finally, we make an assumption that will help the agent expand its current estimate of the safe-set along its edges. Specifically, note that to establish safety of an edge state-action pair, it is necessary to establish a return path from it to the current safe-set (see Appendix~\ref{sec:safe_islands} for a more in-depth reasoning behind this).
Motivated by this we assume the following. Consider any safe subset of state-actions $Z$ and a state-action $(s,a)$ at the edge of $Z$ that also belongs to $\Zsafe$. 
Then, we assume that for every state-action pair that is on the path that starts from this edge and returns to $Z$, there exists an element in $Z$ that is sufficiently similar to that pair. In other words, for \textit{any} safe subset $Z$ of $\Zsafe$, this assumption provides us hope that $Z$ can be expanded by exploring suitable state-actions inside $Z$.

The fact that this allows expansion of any safe subset $Z$ might seem like a stringent assumption. But we must emphasize this is  required to show PAC-MDP optimality: to learn a policy is near-optimal with respect to $\smash{\pisafe^*}$, intuitively, our algorithm must {\em necessarily} establish the safety of the set of state-actions that $\smash{\pisafe^*}$ could visit. Now, depending on what the (unknown) $\smash{\pisafe^*}$ looks like, this set could be as large as (the unknown set) $\Zsafe$ itself. To take this into account, our assumption essentially allows the agent to use analogies to expand its safe set to a set as large as $\Zsafe$ if the need arises. Without this assumption, the agent might at some point not be able to expand the safe set and end up acting sub-optimally forever. 
We note earlier works \citep{turchetta16safeexpl,berkenkamp17stability,akametalu14reachability} do not make this assumption because they crucially didn't need to establish PAC-MDP optimality.


\begin{assumption}[Similarity of return paths]
\label{as:sim}
For any safe set $Z$ such that $Z_0 \subset Z$ and for any $(\tilde{s},\tilde{a}) \in \Zsafe$ such that $\tilde{s} \in Z$ and $(\tilde{s},\tilde{a}) \notin Z$, we know from Assumption~\ref{as:zsafe_communicating} that the agent can return to $Z_0$ (and by extension, to $Z$)  from $(\tilde{s},\tilde{a})$ through at least one $\pireturn \in \Pi(\Zsafe)$. Let $\smash{\tilde{Z}}$ denote
the set of state-action pairs $(s,a)$ visited by $\pireturn$ before reaching $Z$, in that,
\[
\Pr_M\left[\left. \exists t \geq 0 \ \substack{(s_t, a_t) = (s,a) \\ \forall t' < t \ s_{t'} \notin Z} \ \right| (s_0, a_0) = (\tilde{s}, \tilde{a}), \pireturn\right] > 0.
\]
We assume that 
for all $(s,a) \in \tilde{Z} \setminus Z$, there exists $(s',a') \in Z$ such that $\Delta((s,a), (s',a')) \leq \tau/4$.
\end{assumption}

To provide some intuition behind this assumption, consider again the tight-rope walker example.
For this assumption to be satisfied, the agent must be able to safely learn how to return from any point along the real tight-rope.
Since the agent has access to a similar practice tight-rope, the agent can learn to safely cross the practice tight-rope, turn around, and return safely.
Once the agent has learned this policy, it has established a safe return path from any point along the real tight-rope.
Thus, this assumption is satisfied.

\section{Analogous Safe-state Exploration}

\begin{algorithm}[t]
\footnotesize
\caption{Analogous Safe-state Exploration $(\alpha, \Delta, m, \delta_T, R, \gamma, \gamma_{\text{explore}}, \gamma_{\text{switch}}, \tau)$}
\label{alg:ase}
\begin{algorithmic}
\State Initialize: $\hatZsafe \gets Z_0$; $n(s,a),n(s,a,s') \gets 0$;
\State \hspace{14mm} $\optZgoal \gets S \times A$; $s_0 \gets \sinit$.
\State Compute confidence intervals using Alg~\ref{alg:tighter_ci} (Appendix~\ref{sec:methodology}) with parameter $\delta_T$ and analogy function $\alpha, \Delta$.
\State Compute $\optpigoal, \optZgoal$, $\Zexplore$ using Alg~\ref{alg:z_goal} and~\ref{alg:z_explore} with parameters $\gamma, \tau$ and reward function $R$.
\State Compute $\optpiexplore, \optpiswitch$ using value iteration (Appendix~\ref{sec:opt_policies}) with parameters $\gamma_{\text{explore}}, \gamma_{\text{switch}}$.
\For{$t = 1, 2, 3, \ldots$} 
    \State $a_t \gets \begin{cases} \optpigoal(s_t) & \textbf{if } s_t \in \optZgoal$ \& $\optZgoal \subset \hatZsafe \\
    \optpiexplore(s_t) & \textbf{if } \optZgoal \not\subset \hatZsafe \\
    \optpiswitch(s_t) & \text{\bf otherwise}.
    \end{cases}$
    \State Take action $a_t$ and observe next state $s_{t+1}$.
    \If{$n(s_t, a_t) < m$}
        \State $n(s_t, a_t) \pluseq 1$,  $n(s_t, a_t, s_{t+1}) \pluseq 1$.
        \State Recompute confidence intervals
        \State Expand $\hatZsafe$ using Alg~\ref{alg:compute_safe_set} with parameter $\tau$.
        \State Recompute $\optpigoal, \optZgoal$, $\Zexplore$, $\optpiexplore, \optpiswitch$.
    \EndIf
\EndFor

\end{algorithmic}
\end{algorithm}

Given these assumptions, we now detail the main algorithmic contribution of the paper, the Analogous Safe-state Exploration (ASE) algorithm, which we later prove is Safe-PAC-MDP.
In addition to safety and optimality, we also do not want to exhaustively explore the state-action space, like R-Max \citep{brafman2002r}, as that can be prohibitively expensive in large domains.
We want to guide our exploration, like MBIE  \citep{strehl2008analysis}, to explore only the state-action pairs that are needed to find the safe-optimal policy. MBIE does this by maintaining confidence intervals of the dynamics of the MDP, and then by following an
``optimistic policy'' computed using the most optimistic model of the MDP that falls within the computed confidence intervals. 
We build on this standard MBIE approach and equip it with a significant amount of machinery to meet our three objectives simultaneously:  safety, guided exploration, and optimality in the PAC-MDP sense.


\textbf{Policies maintained by ASE.} ASE maintains and updates three different policies: (a) an optimistic policy $\optpigoal$ that seeks to maximize reward -- this is the same as the optimistic policy as in standard MBIE computed on $M$ (except some minor differences), (b) an exploration policy $\optpiexplore$ that guides the agent towards states in a set called $\Zexplore$ (described shortly) and finally (c) a ``switching'' policy $\optpiswitch$ that can be thought of as a policy that aids the
agent in switching from $\optpiexplore$ to $\optpigoal$ (by carrying it from $\Zexplore$ to $\optZgoal$ as explained shortly)
See Appendix~\ref{sec:opt_policies} for details on these policies.

\textbf{Sets maintained by ASE.} ASE also maintains and updates three major subsets of state-action pairs.
1) a safe set $\smash{\hatZsafe}$ which is initialized to $Z_0$ and gradually expanded over time (using Alg~\ref{alg:compute_safe_set}).
2) an ``optimistic trajectory'' set $\smash{\optZgoal}$ (computed in Alg~\ref{alg:z_goal}), which contains all state-action pairs that we expect the agent would visit if it were to follow $\optpigoal$ from $\sinit$ under \textit{optimistic} transitions.
3) an ``exploration set'' $\Zexplore$  (computed using Algs~\ref{alg:z_goal} and \ref{alg:z_explore}) which contains state-action pairs that, when explored, will provide information critical to expand the safe set.
Besides these state-action sets, the algorithm also maintains a set of $L_1$ confidence intervals
as detailed in Alg~\ref{alg:tighter_ci} and Appendix~\ref{sec:ci}.
The key detail here is that the interval for a given state-action pair is not only updated using its samples, but also by exploiting the samples seen at any other well-explored state-action pair that is sufficiently similar to the given pair, according to the given analogies.

\textbf{How ASE schedules the policies.}
We discuss how, at any timestep, the agent chooses between one of the above three policies to take an action.
First, the agent follows $\optpigoal$ whenever it can establish that doing so would be safe.
Specifically, whenever (a) $\smash{\optZgoal \subset \hatZsafe}$ and (b) the current state $s_t$ belongs to $\smash{\optZgoal}$, it is easy to argue that following $\optpigoal$ is safe (see proof of safety in Theorem~\ref{thm:pac_and_safe}).
On the other hand, when (a) does not hold, the agent follows $\optpiexplore$.
In doing so, the hope is that, it can explore $\Zexplore$ well and use analogies to expand $\smash{\hatZsafe}$ until it is large enough to subsume $\smash{\optZgoal}$ (which means (a) would hold then). As a final case, 
assume (a) holds, but (b) does not i.e., 
$\smash{s_t \notin \optZgoal}$.
This could happen if the agent has just explored a state-action pair far away from $\optZgoal$, which subsequently helped establish  (a) i.e., $\smash{\optZgoal \subset \hatZsafe}$.
Here, we use $\optpiswitch$ to carry the agent back to $\smash{\optZgoal}$. Once carried there, both (a) and (b) hold, so it can switch to following 
$\smash{\optpigoal}$.

\textbf{Guided exploration.}  We would like the agent to explore only relevant states by using the optimistic policy as a guide, like in MBIE. This is automatically the case whenever the agent explores using the optimistic goal policy $\optpigoal$. 
As a key addition to this, we use the optimistic policy to also guide the exploration policy $\optpiexplore$. As explained below, we do this by only conservatively populating $\Zexplore$ based on the optimistic trajectory set $\smash{\optZgoal}$.
Recall that the hope from exploring $\Zexplore$ is that it can help $\smash{\hatZsafe}$ expand in a way that it is large enough to satisfy $\smash{\optZgoal \subset \hatZsafe}$. 
Keeping this in mind, the naive way to set $\Zexplore$ would be to add all of $\smash{\hatZsafe}$ to it; this will force us to do a brute-force exploration of the safe set, and consequently aggressively expand the safe set in all directions. Instead of doing so, roughly speaking, we compute $\Zexplore$ in a way that it can help establish the safety of only those actions that (a) are on the edge of $\smash{\hatZsafe}$ and  (b) also belong to $\smash{\optZgoal}$. This will help us conservatively expand the safe set in ``the direction of the optimistic goal policy''. (Note that all this entails non-trivial algorithmic challenges since we operate in an unknown stochastic environment. Due to lack of space we discuss these challenges in Appendix~\ref{app:guided-exploration})




\section{Theoretical Results}

Below we state our main theoretical result, that ASE is Safe-PAC-MDP. We must however emphasize that this result does not intend to provide a tight sample complexity bound for ASE; nor does it intend to compete with existing sample complexity results of other (unsafe) PAC-MDP algorithms. In fact, our sample complexity result does not capture the benefits of our guided exploration techniques -- instead, we use practical experiments to demonstrate that these benefits are significant (see Section \ref{sec:experiments}). 
The goal of this theorem is to establish that ASE is indeed PAC-MDP and safe, which in itself is highly non-trivial as ASE has much more machinery than existing PAC-MDP algorithms like MBIE. 
In the interest of space we will give a brief overview of the proof \& algorithm.
A more detailed proof outline can be found in Appendix~\ref{sec:proof_outline}.
The full (lengthy) proof is included in Appendix~\ref{sec:proof}.

\begin{restatable}{theorem}{pacmdp}
\label{thm:pac_and_safe}
%
For any constant $c \in (0,1/4]$, $\epsilon,\delta \in (0,1]$, MDP $M = \langle S, A, T, \mathcal{R}, \gamma \rangle$, for $\delta_T = \delta/(2 |S||A|m)$, $\gamma_{\text{explore}} = \gamma_{\text{switch}} = c^{1/H}$, and $m = O\left((|S|/\tilde{\epsilon}^2) + (1/\tilde{\epsilon}^2) \ln \left(|S| | A |/\tilde{\epsilon}\right) \right)$ where $\tilde{\epsilon} = \min\left(\tau, {\epsilon}(1-\gamma)^2, 1/H^2  \right)$ and
$H = O\left( \max \left \{ \Hcom \log \Hcom,  (1/(1-\gamma)) \ln (1/\epsilon (1-\gamma)) \right \}\right)$ ASE is Safe-PAC-MDP with a sample complexity bounded by $O \left (H m |S| |A|  (1 / \epsilon ( 1 - \gamma)) \ln (1/\delta) \right )$.
%
\end{restatable}

To prove safety, 
we show in Lemma~\ref{lem:hatZsafe_safe} that Alg~\ref{alg:compute_safe_set}, which computes $\smash{\hatZsafe}$, always ensures that $\smash{\hatZsafe}$ is a safe set.
Then, in the main proof of Theorem~\ref{thm:pac_and_safe}, we argue that the agent always picks state-actions inside $\smash{\hatZsafe}$. 
So, it follows that the agent always experiences only positive rewards.
Next, in order to prove PAC-MDP-ness, while we build on the core ideas from the proof for PAC-MDP-ness of MBIE \citep{strehl2008analysis}, our proof is a lot more involved. This is because we need to show that all the added machinery in ASE work in a way that (a) the agent never gets ``stuck'' and (b) whenever the agent takes a series of sub-optimal actions (e.g., while following $\optpiexplore$ or $\optpiswitch$), it can ``make progress'' in some form. 
As an example of (a), Lemma~\ref{lem:hatZsafe_communicating} shows that $\smash{\hatZsafe}$ is always a communicating set, so the agent can always freely move between the states in $\smash{\hatZsafe}$.
This is critical to show that when the agent follows $\optpiexplore$ (or $\optpiswitch$), it can reach $\Zexplore$ (or $\smash{\optZgoal}$) without being stuck anywhere (see Lemma~\ref{lem:escape}).
As an example of (b), Lemma~\ref{lem:correctness_z_explore} and Lemma~\ref{lem:correctness_z_goal} together show that only informative state-action pairs are added to $\Zexplore$  i.e., when explored, they \textit{will} help us expand $\smash{\hatZsafe}$.


\section{Experiments}
\label{sec:experiments}

\begin{figure}[t]
    \centering
    \begin{subfigure}[b]{\columnwidth}
    \centering
        \includegraphics[width=\textwidth, trim=40 0 40 0, clip]{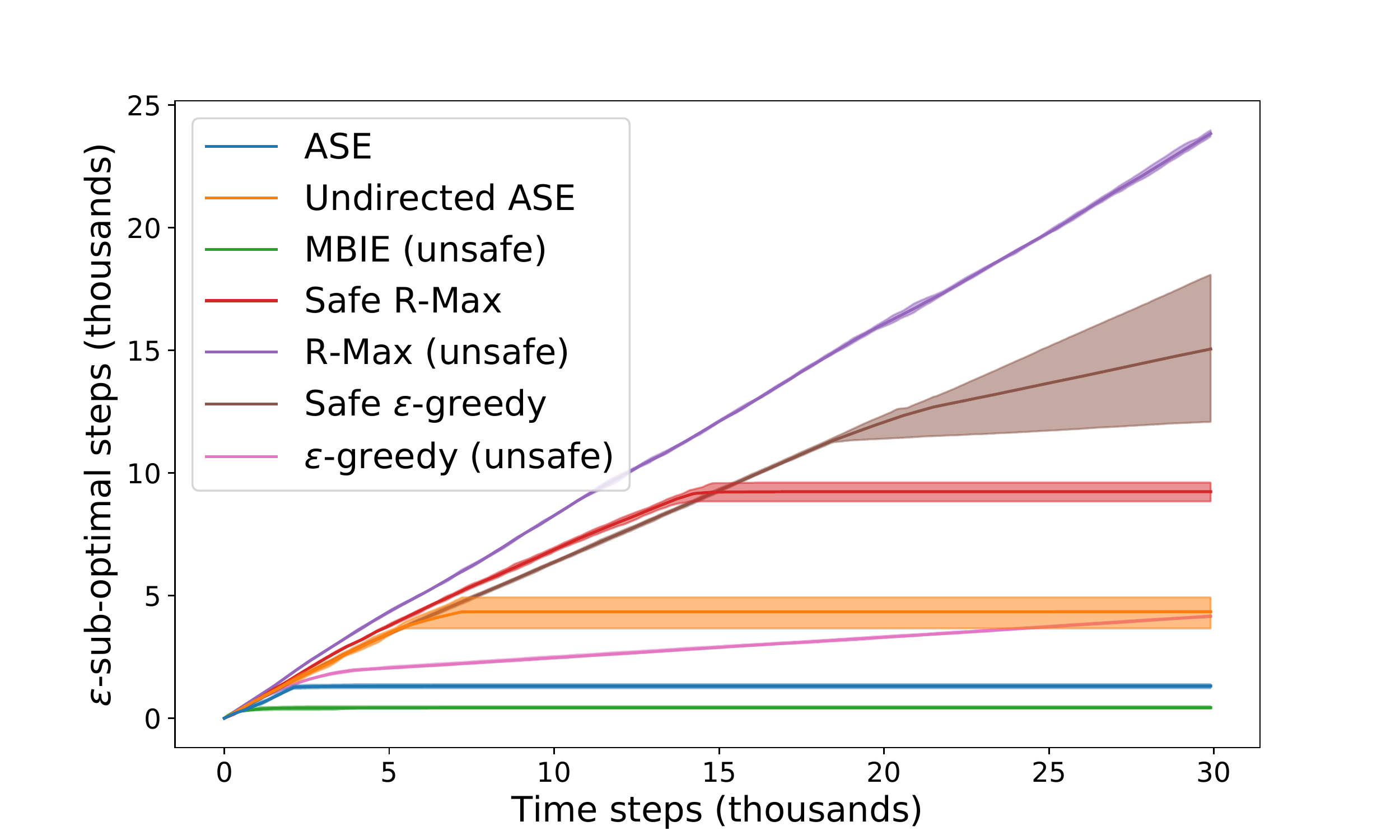}
        \caption{Unsafe Grid World}
        \label{fig:sub_opt:gridworld}
    \end{subfigure}
    \begin{subfigure}[b]{\columnwidth}
        \centering
        \includegraphics[width=\textwidth, trim=40 0 40 0, clip]{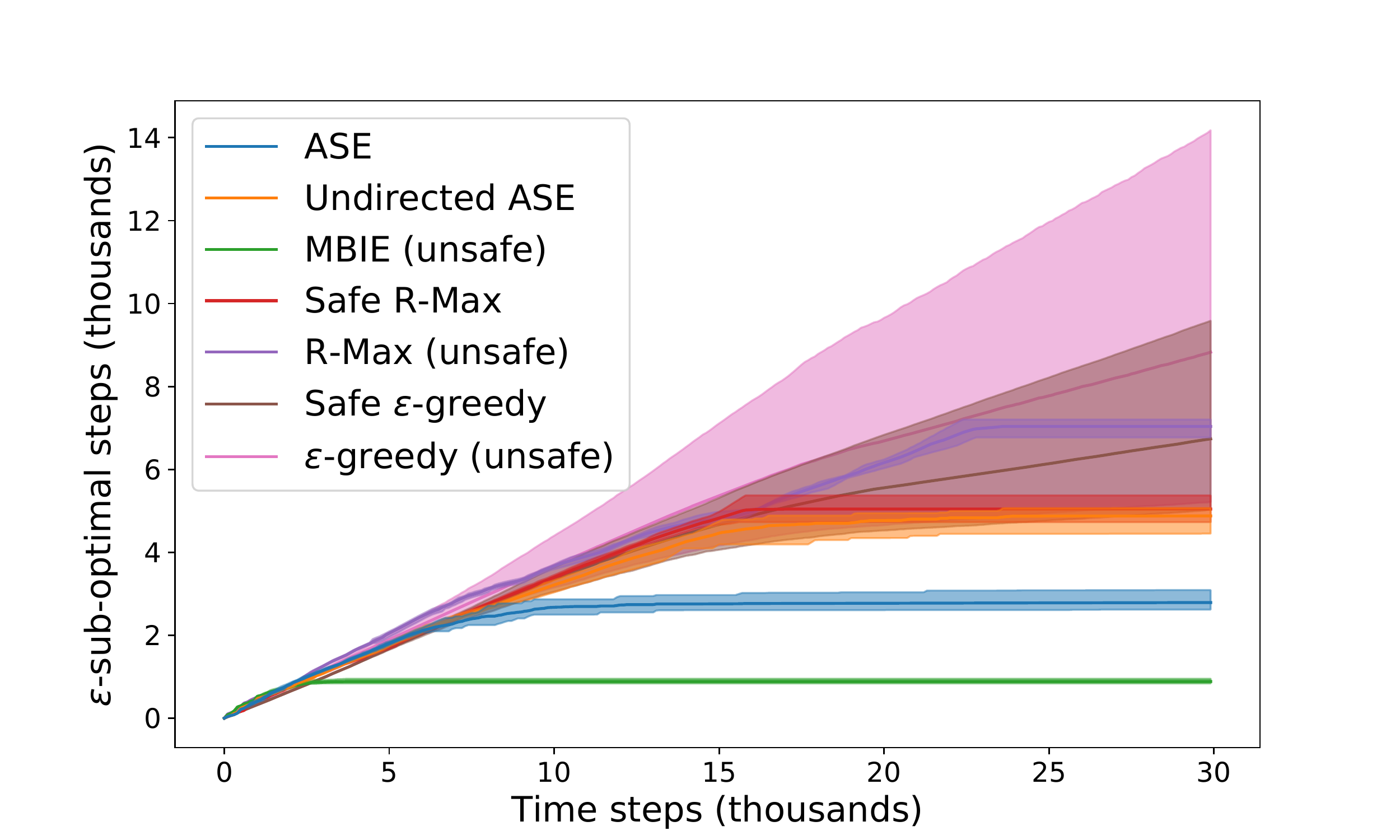}
        \caption{Discrete Platformer}
        \label{fig:sub_opt:platformer}
    \end{subfigure}
    \caption{Number of $\epsilon$-sub-optimal steps taken by each agent throughout training.
    Lines denote averages over five trials and shaded regions mark the max and min.
    }
    \label{fig:sub_opt}
\end{figure}

\begin{figure}[t]
  \centering
  \begin{subfigure}[b]{\columnwidth}
    \includegraphics[width=\textwidth, height=20mm]{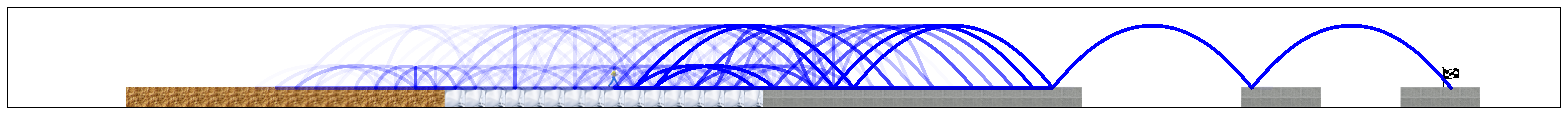}
    \caption{ ASE }
  \end{subfigure}
  \begin{subfigure}[b]{\columnwidth}
    \includegraphics[width=\textwidth, height=20mm]{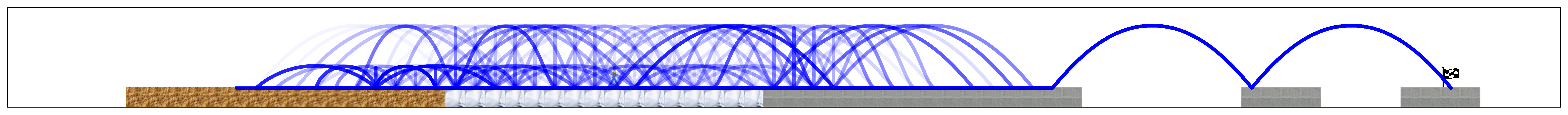}
    \caption{ Safe R-Max }
  \end{subfigure}
  \begin{subfigure}[b]{\columnwidth}
    \includegraphics[width=\textwidth, height=20mm]{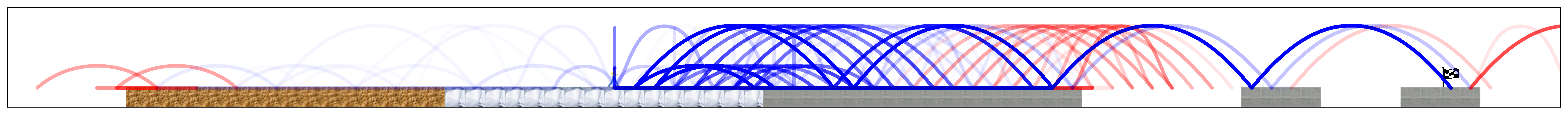}
    \caption{ MBIE (unsafe) }
  \end{subfigure}
  \caption{
  All trajectories of different agents on the Discrete Platformer domain.
  Unsafe trajectories are drawn in red.
  The brown, white, and grey squares correspond to the different surface types: sand, ice, and concrete, respectively.
  The agent starts in the center of the leftmost island.
  The flag represents the goal state.
  }
  \label{fig:heat_map}
\end{figure}

Through experiments, we aim to show that (a) ASE effectively guides exploration, requiring significantly less exploration than exhaustive exploration methods, and (b) the agent indeed never reaches a dangerous state under realistic settings of the parameters (namely $m$ and $\delta_T$ in Alg~\ref{alg:ase}). 
Below, we outline our experiments, deferring the details to Appendix~\ref{appendix:experiment}.
For our experiments, we consider two environments.
The first is a stochastic grid world containing five islands of grid cells surrounded by ``dangerous states'' (i.e., states where all actions result in a negative reward).
The agent can take actions that allow it to jump over dangerous states to transition between islands and reach the goal state.
The second is a stochastic platformer game where certain actions can doom the agent to eventually reaching a dangerous state by jumping off the edge of a platform.
In the grid world environment, two state-actions are analogous if the actions are equivalent and the states are near each other ($L_{\infty}$ distance), and in the stochastic platformer game, if the actions and all attributes but the horizontal position in the state are equivalent and the two states are on the same ``surface type''.


We compare the behavior of our algorithm against both ``unsafe'' and ``safe'' approaches to learning reward-based policies.
For the unsafe baselines, we consider the original (unsafe) MBIE algorithm \citep{strehl2008analysis}, R-Max \citep{brafman2002r}, and $\epsilon$-greedy, all adapted to use the analogy function (without which, exploring would take prohibitively long).
For safe baselines, unfortunately, there is no existing algorithm because no prior work has simultaneously addressed the two objectives of provably safe exploration and learning a reward-based policy in environments with unknown stochastic dynamics.
To this end, we create safe versions of R-Max and $\epsilon$-greedy (by restricting the allowable set of actions the agent can take to $\smash{\hatZsafe}$, and using analogies to expand $\smash{\hatZsafe}$), and also consider an ``Undirected ASE,''  which is a naïver version of ASE that expands $\smash{\hatZsafe}$ in all directions (not just along the goal policy).

Source code for the experiments is available at \url{https://github.com/locuslab/ase}.


\textbf{Results.}
To measure efficiency of exploration, we count the number of $\epsilon$-sub-optimal steps taken by each agent.
To calculate this, we first compute the true safe-optimal $Q$-function, $Q_M^{\pisafe^*}$.
We then count the number of $\epsilon$-sub-optimal actions taken by the agent, namely the number of times the agent is at a state $s_t$ and takes an action $a_t$ such that $Q_M^{\pisafe^*}(s_t, a_t) < \max_{a \in A} Q_M^{\pisafe^*}(s_t, a) - \epsilon$, where $\epsilon = 0.01$.
Figure~\ref{fig:sub_opt} shows our algorithm takes far fewer $\epsilon$-sub-optimal actions before it converges compared to all other \textit{safe} algorithms.
As for safety, during our experiments, we observe that, in both domains, the safe algorithms do not reach any unsafe states.
In the unsafe grid world domain, the MBIE, R-Max, and $\epsilon$-greedy algorithms encounter an average of $\num{85}$, $\num{5016}$, and $\num{915}$ unsafe states, respectively, and in the discrete platformer game encounter $\num{83}$, $\num{542}$, and $\num{768}$ unsafe states.

In the platformer domain, as we can see from Figure~\ref{fig:heat_map}, our method explores only the necessary parts of the initial safe-set, the right side, unlike the Safe R-Max algorithm. 
Although standard MBIE also directs exploration, it has many trajectories that end in unsafe states, which ASE avoids.

\section{Conclusion}

We introduced Analogous Safe Exploration (ASE), an algorithm for safe and guided exploration in unknown, stochastic environments using analogies.
We proved that, with high probability, our algorithm never reaches an unsafe state and converges to the optimal policy, in a PAC-MDP sense.
To the best of our knowledge, this is the first provably safe and optimal learning algorithm for stochastic, unknown environments (specifically, safe during exploration).
Finally, we illustrated empirically that ASE explores more efficiently than other non-guided methods.
Future directions for the this line of work include extensions to continuous state-action spaces, combining the handling of stochasticity we present here with common strategies in these domains such as kernel-based nonlinear dynamics.

\subsubsection*{Acknowledgements}
Melrose Roderick and Vaishnavh Nagarajan were supported by a grant from the Bosch Center for AI.
This material is based upon work supported by the National Science Foundation Graduate Research Fellowship under Grant No.~DGE1745016.

\bibliography{references}
\bibliographystyle{unsrtnat}
\newpage
\onecolumn
\appendix
\section{Notations, definitions and other useful facts}
\label{sec:facts}

In this section, we define certain standard notations and state some facts that were used in the main paper. 


\subparagraph{Policy of an algorithm.} In PAC-MDP models, an algorithm is considered to be a non-stationary policy $\mathcal{A}$ which, at any instant $t$, takes as input the path taken so far $p_t := s_0, a_0, \hdots, a_{t-1}, s_t$ and outputs an action. More formally, $\mathcal{A}: \left\lbrace S \times A \right\rbrace^* \times S \to A$. Note that since the algorithm  is already given the true reward function, we do not provide rewards as input to this non-stationary policy.  Then the value of the policy is formally defined as given below.

\begin{definition}\label{def:alg-value}
For any $p_t$, we define the value of the non-stationary policy $\mathcal{A}$ of our algorithm on the MDP $M$ as:
\[
V^{\mathcal{A}}_{M}(p_t) = \mathbb{E}\left[ \left. \sum_{t' = t}^{\infty} R(s_{t'},a_{t'}) \right| p_t, \mathcal{A} \right].
\]
For any $H > 0$, we denote the truncated value function of $\mathcal{A}$ as:
\[
V^{\mathcal{A}}_{M}(p_t, H) = \mathbb{E}\left[ \left. \sum_{t' = t}^{t+H} R(s_{t'},a_{t'}) \right| p_t, \mathcal{A} \right].
\]
\end{definition}

\subparagraph{Following $\pi \in \Pi(Z)$.} Below we state the fact that for a closed set $Z$, by following $\pi \in \Pi(Z)$, the agent always remains in $Z$ with probability $1$.

\begin{fact}
\label{fact:closed_policy}
For any closed set of state-action pairs $Z$, and for any policy $\pi \in \Pi(Z)$ and for any initial state $s_0 \in Z$:
\[
\Pr[\forall t, s_t \in Z \ | \ \pi, s_0 \in Z] = 1
\]
\end{fact}

\begin{proof}
 For any $t$, if $s_t \in Z$ (and this is true for $t=0$), we have  that $(s_t,\pi(s_t)) \in Z$ since $\pi \in \Pi(Z)$. Since $Z$ is closed, this means that $s_{t+1} \in Z$. Hence, by induction the above claim is true.
\end{proof}

\subparagraph{Communicatingness.}
We now discuss the standard notion of communicating and argue that it is equivalent to the notion defined in the paper (Definition~\ref{def:communicating}).
Recall that the standard notion of communicatingness of an MDP \citep{puterman2014markov} is that, for any pair of states in the MDP, there exists a stationary policy that takes the agent from one to the other with positive probability in finite steps.
This can be easily generalized to a subset of closed states as follows:

\begin{definition}
\label{def:communicating-standard}
A closed subset of state-action pairs, $Z \subset S \times A$ is said to be \textbf{communicating} if for any two states, $s, s' \in Z$, there exists a stationary policy $\pi_{s \to s'} \in \Pi(Z)$ such that for some $n \geq 1$
\[
\Pr[s_n = s' \ | \pi_{s \to s'}, s_0 = s] > 0
\]
\end{definition}

There are two key differences between this definition and Definition~\ref{def:communicating}.
(1) Recall that in Definition~\ref{def:communicating}, we defined communicating to mean that for a particular destination state, there exists a single stationary policy that can take the agent from {\em any} state inside $Z$ to that destination state, whereas in Definition~\ref{def:communicating-standard} there is a specific policy for every pair of states.
(2) Definition~\ref{def:communicating-standard} only requires the probability of reaching $s'$ to be positive as opposed to $1$.
Below, we note why these definitions are equivalent:





\begin{fact} \label{fact:communicating_single_policy}
Definition~\ref{def:communicating} and Definition~\ref{def:communicating-standard} are equivalent.
\end{fact}
\begin{proof}
Informally, we need to show that ``for any pair of states in $Z$, there exists a policy that takes the agent from one to the other with positive probability'' if and only if
``there exists a single policy that reaches a destination state from anywhere in $Z$ with probability 1''.

The sufficient direction is clearly true: if such a $\pi_{s'}$ and $t$ from Definition~\ref{def:communicating} exist, then we know that for any $s \in S$, we can set $\pi_{s \to s'}$ and $n$ in Definition~\ref{def:communicating-standard} to be $\pi_{s'}$ and $t$ to prove communicatingness (since if it holds with probability $1$, it also holds with positive probability).

For the necessary direction, we will show that for a communicating $Z = S \times A$, for any $s' \in Z$, the optimal policy of another MDP satisfies the requirements of $\pi_{s'}$.
For a communicating $Z$ with Definition~\ref{def:communicating-standard}, we know that for any $s, s' \in S$, there exists a policy $\pi_{s \to s'}$ and $n \geq 1$ such that $\Pr [s_n = s' | \pi_{s \to s'}, s_0 = s] > 0$.

Now we construct the MDP.
For a given $s'$, define an MDP $M_{s'} = \langle S, A, T, R_{s'}, \gamma_{s'} \rangle$ where $\gamma_{s'} = 1$, $R(s) = \mathbf{1} \{s = s' \}$, and $s'$ is terminal.
Note that, since $R(s') = 1$ and is zero everywhere else and $s'$ is terminal, for any policy $\pi$, the state value function $V_{M_{s'}}^{\pi}(s) = \Pr[\exists \ t, \ s.t. \ s_t = s' | \pi, s_0 = s]$.
Now define $\pi_{s'}$ to be the optimal policy of this MDP, i.e. the policy that maximizes this value function.
This implies that for any $s$, $V_{M_{s'}}^{\pi_{s'}}(s) \geq V_{M_{s'}}^{\pi_{s, s'}}(s)$. Then, 
\begin{align*}
  \Pr[\exists \ t, \ s.t. \ s_t = s' | \pi_{s'}, s_0 = s]
  &= V_{M_{s'}}^{\pi_{s'}}(s) \\
  &\geq V_{M_{s'}}^{\pi_{s, s'}}(s) \\
  &= \Pr[\exists \ t, \ s.t. \ s_t = s' | \pi_{s, s'}, s_0 = s] \\
  &\geq \Pr[s_n = s' | \pi_{s, s'}, s_0 = s] \\
  &> 0.
\end{align*}
Thus, $\Pr[\exists \ t, \ s.t. \ s_t = s' | \pi_{s'}, s_0 = s] > 0$.
Since this condition is satisfied for all starting states $s \in S$ and since we know that $Z$ is closed, we can use Lemma~\ref{lem:communicating_positive_prob} to show that in fact $\Pr[\exists \ t, \ s.t. \ s_t = s' | \pi_{s'}, s_0 = s] = 1$, proving the existence of $\pi_{s'}$ as claimed.


\end{proof}

\subparagraph{$\Zsafe$ is safe.}
Below, we prove that the set $\Zsafe$ defined in Definition~\ref{def:Zsafe} is indeed a safe set.

\begin{fact}
\label{fact:Zsafe}
$\Zsafe$ is a safe set.
\end{fact}
\begin{proof}
For any $(s,a) \in \Zsafe$, we have by definition that $R(s,a) \geq 0$.  Now, if for some $s'$ for which $T(s,a,s') > 0$, consider
 $a' = \pireturn(s')$. Then, by definition, if the agent were to start at $(s',a')$  and continue following $\pireturn$, with probability $1$, it would reach $Z_0$ while experiencing only positive rewards. Hence, $(s',a') \in \Zsafe$. Thus, $\Zsafe$ is closed.
\end{proof}

\section{Methodology}
\label{sec:methodology}
In this section, we provide detailed definitions of our algorithms and the notation used for our proofs.
We start by giving detailed a description of the main algorithm and the algorithms which it calls.
We then detail how our confidence intervals are computed and how we use these confidence intervals to construct the set of all candidate transition functions.
The next section details how these candidate transition functions are used for computing optimistic policies $\optpigoal$, $\optpiexplore$ and $\optpiswitch$.
Specifically, we define new MDPs for each of these policies and define how an optimistic policy is computed on an arbitrary MDP.
Finally, we formalize the discounted state distribution and discuss how this is computed in practice.

Here we provide a useful reference for some of the notation used throughout the proofs.

\begin{center}
\begin{tabular}{ |c|c| } 
 \hline
 True MDP & $M = \langle S, A, T, R, \gamma\rangle$  \\  \hline
 True safe, optimal goal policy & $\pisafe^* $ \\ \hline
 Empirical transition probability & $\hat{T}$ \\ \hline
 Empirical $L1$ confidence interval width &  $\hat{\epsilon}_T$ \\ \hline
 Arbitrary MDP & $M^\dagger = \langle S, A, T^\dagger, R^\dagger, \gamma^\dagger \rangle$ \\ \hline
 Optimal Q-function on the optimistic MDP $\optMgoal $ & $\optQgoal$ \\ \hline
 Optimal Q-function on the optimistic MDP $\optMexplore$ & $\optQexplore$ \\ \hline
 Optimal Q-function on the optimistic MDP $\optMswitch$ & $\optQswitch$ \\ \hline
 Optimistic goal policy (defined by $\optQgoal$) & $\optpigoal$ \\ \hline
 Optimistic explore policy (defined by $\optQgoal$) & $\optpiexplore$ \\ \hline
 Optimistic switching policy (defined by $\optQswitch$) & $\optpiswitch$ \\ \hline
 Analogy-based empirical probabilities and confidence intervals   & $\del{T}, \del{\epsilon}_T$ \\ \hline
 Optimal value function for some MDP $M^\dagger$ & $V^*_{M^\dagger}$ \\ \hline
\end{tabular}
\end{center}

\subsection{Algorithms}
\label{sec:algs}

\setcounter{algorithm}{0}

\begin{algorithm}[t]
\caption{Analogous Safe-state Exploration $(\alpha, \Delta, m, \delta_T, R, \gamma, \gamma_{\text{explore}}, \gamma_{\text{switch}}, \tau)$}
\begin{algorithmic}
\State Initialize: $\hatZsafe \gets Z_0$; $n(s,a),n(s,a,s') \gets 0$; $\optZgoal \gets S \times A$; $s_0 \gets \sinit$.
\State Initialize: $\hatZunsafe \gets \{(s, a) \in S \times A : R(s, a) < 0 \}$.
\State Compute confidence intervals, $\del{T}$ and $\del{\epsilon}_T$, using Alg~\ref{alg:tighter_ci} (Appendix~\ref{sec:methodology}) with 
\State \hspace{\algorithmicindent} state-action-state counts $n$, parameter $\delta_T$, and analogy function $\alpha, \Delta$.
\State Compute $\optpigoal, \optZgoal$, $\Zexplore$, $\hatZunsafe$ using Alg~\ref{alg:z_goal} and~\ref{alg:z_explore} with parameters $\gamma, \tau$
\State \hspace{\algorithmicindent} and reward function $R$.
\State Compute $\optpiexplore, \optpiswitch$ using value iteration (Appendix~\ref{sec:opt_policies}) with parameters $\gamma_{\text{explore}}, \gamma_{\text{switch}}$.
\For{$t = 1, 2, 3, \ldots$} 
    \State $a_t \gets \begin{cases} \optpigoal(s_t) & \textbf{if } s_t \in \optZgoal$ \& $\optZgoal \subset \hatZsafe \\
    \optpiexplore(s_t) & \textbf{if } \optZgoal \not\subset \hatZsafe \\
    \optpiswitch(s_t) & \text{\bf otherwise}.
    \end{cases}$
    \State Take action $a_t$ and observe next state $s_{t+1}$.
    \If{$n(s_t, a_t) < m$}
        \State $n(s_t, a_t) \pluseq 1$,  $n(s_t, a_t, s_{t+1}) \pluseq 1$.
        \State Recompute confidence intervals, then expand $\hatZsafe$ using Alg~\ref{alg:compute_safe_set} with parameter $\tau$.
        \State Recompute $\optpigoal, \optZgoal$, $\Zexplore$, $\hatZunsafe$, $\optpiexplore, \optpiswitch$ as above.
    \EndIf
\EndFor

\end{algorithmic}
\end{algorithm}

\begin{algorithm}[t]
\caption{Compute Safe Set $(R, \tau)$}
\label{alg:compute_safe_set}
\begin{algorithmic}
\Require Estimated safe set $\hatZsafe$ and confidence intervals $\del{T}$, $\del{\epsilon}_T$.
\State $Z_{\text{candidate}} \gets \{(s,a) \in (S \times A) \setminus \hatZsafe \ s.t. \del{\epsilon}_T(s, a) < \tau/2, R(s,a) \geq 0 \}$.

\While{$Z_{\text{candidate}} \neq \Zclosed$ in the last iteration}
\State $\Zreachable \gets \{ (s,a) \in \Zcandidate: s \in \hatZsafe \}$.
\While{$\Zreachable$ changed in the last iteration}
  \For{$(s, a) \in \Zreachable \cup \hatZsafe$}
    \State Add $\{(s', a') \in Z_{\text{candidate}} \ s.t. \ \del{T}(s, a, s') > 0\}$ to $\Zreachable$.
  \EndFor
\EndWhile
\State $\Zreturnable \gets \emptyset$.
\While{$\Zreturnable$ changed in the last iteration}
  \For{$(s, a) \in \Zreachable$}
    \If{$\exists \ (s', a') \in \Zreturnable \cup \hatZsafe \ s.t. \ \del{T}(s, a, s') > 0$}
        \State Add $(s, a)$ to $\Zreturnable$.
    \EndIf
  \EndFor
\EndWhile
\State $\Zclosed \gets \Zreturnable$.
\While{$\Zclosed$ changed in the last iteration}
  \For{$(s, a) \in \Zclosed$}
    \If{$\exists \ s' \in S \ s.t. \ \del{T}(s, a, s') > 0$ and \ $\forall \ a' \in A, \ (s', a') \not \in \Zclosed \cup \hatZsafe$}
        \State Remove $(s, a)$ from $\Zclosed$.
    \EndIf
  \EndFor
\EndWhile
\State $Z_{\text{candidate}} \gets \Zclosed$.
\EndWhile
\State $\hatZsafe \gets \Zclosed \cup \hatZsafe$.
\end{algorithmic}
\end{algorithm}

\begin{algorithm}[t]
\caption{Compute $\optpigoal$, $\optZgoal$ and $\Zexplore$ $(\alpha, \Delta, \tau)$}
\label{alg:z_goal}
\begin{algorithmic}
\Require Estimated safe and unsafe sets $\hatZsafe, \hatZunsafe$ and confidence intervals $\del{T}$, $\del{\epsilon}_T$.
\State Initialize: $\Zexplore \gets \emptyset$.
\For{$i = 1, 2, \ldots$}
    \State Compute $\optpigoal$ using Eq~\ref{eq:q-values}.
    \State Compute $\optZgoal$ using Alg~\ref{alg:compute_z_goal}.
    \If{$\optZgoal \subset \hatZsafe$}
      \State Break.
    \EndIf
    
    \State $Z_{\text{edge}} \gets \{(s,a) \in \optZgoal\setminus \hatZsafe \; | s \in \hatZsafe \}$.
    \State Compute $\Zexplore$ using Alg~\ref{alg:z_explore}.
    
    \If{$\Zexplore = \emptyset$}
      \State Add $Z_{\text{edge}}$ to $\hatZunsafe$.
    \Else
      \State Break.
\EndIf
\EndFor

\end{algorithmic}
\end{algorithm}

\begin{algorithm}[t]
\caption{Compute $\Zexplore$ $(\alpha, \Delta, \tau)$}
\label{alg:z_explore}
\begin{algorithmic}
\Require Sets of state-action pairs on the edge of the safe set, $Z_{\text{edge}}$, and on the goal path, $\optZgoal$, \\
\hspace{\algorithmicindent}\hspace{\algorithmicindent} along with $\hatZsafe, \hatZunsafe$ and $\del{T}$, $\del{\epsilon}_T$.
\State Initialize: $\Zexplore, \Zreturn \gets \emptyset$, $Z_{\text{candidate}} \gets \{ (s,a) \in \optZgoal : s \in \hatZsafe \}$, $L \gets 0$, $\Znext^0 \gets Z_{\text{edge}}$.
\While{$\Zexplore = \emptyset$ and $\Znext^L \neq \emptyset$}
    \State $\Znext^{L+1} \gets \emptyset$ .
    \For{$(s,a) \in \Znext^L$}
        \State Add $(s,a)$ to $\Zreturn$
        \If{$\del{\epsilon}_T(s, a) > \tau/2$}
            \State Add $\{\tilde{s}, \tilde{a} \in \hatZsafe : \Delta((s, a), (\tilde{s}, \tilde{a})) < \tau / 4 \}$ to $\Zexplore$.
        \Else
            \State Add $\{s', a' \in S \times A : \del{T}(s, a, s') > 0 \}\setminus (\Zreturn \cup \hatZsafe \cup \hatZunsafe)$ to $\Znext^{L+1}$.
        \EndIf
    \EndFor
    \State $L \gets L+1$.
\EndWhile
\end{algorithmic}
\end{algorithm}

\begin{algorithm}[t]
\caption{Compute $\optZgoal$ $(\alpha, \Delta, \tau)$}
\label{alg:compute_z_goal}
\begin{algorithmic}
\Require Confidence intervals $\del{T}$, $\del{\epsilon}_T$.
\State For convenience, denote $\rho_{\optpigoal, \sinit}^{\opt{M}_{\text{goal}}}$ as $\optrhogoal$. 
\State Initialize: $\optZgoal \gets \emptyset$, $\optrhogoal(\cdot, \cdot, 0)$ as in Eq~\ref{eq:init_state_distribution}.
\State Using dynamic programming, compute $\optrhogoal(\cdot, \cdot, t)$ as in Eq~\ref{eq:computing_state_distribution} for $t=1, 2, \hdots, |S|$ .
\State Add all $s, a \in S \times A$ where $\optrhogoal(s, a, |S|) > 0$ to $\optZgoal$
\end{algorithmic}
\end{algorithm}

We first restate Algorithm~\ref{alg:ase}, the full overview of our algorithm, with a bit more detail (most notably including $\hatZunsafe$).
Next, Algorithm~\ref{alg:compute_safe_set} details how the agent computes the current estimate of the safe set while ensuring reachability, returnability, and closedness.
The correctness and efficiency of this algorithm is proven in Section \ref{sec:proof:compute_safe_set}.
Algorithms \ref{alg:z_goal}, \ref{alg:z_explore}, and \ref{alg:compute_z_goal} together provide an overview of how ASE computes the goal and explore policies.

\subsection{Confidence intervals}
\label{sec:ci}
We let $\hat{T}$ denote the empirical transition probabilities. We then let $\hat{\epsilon}_T(s,a)$ denote the width of the $L1$ confidence interval for the empirical transition probability $\hat{T}(s,a)$.
As shown in \citet{strehl2008analysis}, by using the Hoeffding bound, we can ensure that if
\begin{equation}
    \label{eq:eps_T}
    \hat{\epsilon}_T(s,a) = \sqrt{\frac{2[\ln{(2^{|S|} -2)} - \ln{(\delta_T)}]}{n(s, a)}}
\end{equation}
where $n(s, a)$ is the number of times we have experienced state-action $(s, a)$, our $L1$ confidence interval hold with probability $\delta_T$.
Using the given analogies, we can then derive tighter confidence intervals of width $\del{\epsilon}_T$ (centered around an estimated $\delT$), especially for state-action pairs we have not experienced, as in Algorithm~\ref{alg:tighter_ci}.
The algorithm essentially transfers the confidence interval from a sufficiently similar, well-explored state-action pair to the under-explored state-action pair, using analogies.

\begin{algorithm}[t]
\caption{Compute Analogy-based Confidence Intervals $(\alpha, \Delta)$}
\label{alg:tighter_ci}
\begin{algorithmic}
\Require State-action and state-action-state counts $n(\cdot, \cdot)$, $n(\cdot, \cdot, \cdot)$.
\State Construct $\hat{T}$ (the empirical transition probabilities) using $n(\cdot, \cdot)$, $n(\cdot, \cdot, \cdot)$.
\State Compute $\hat{\epsilon}_T(s,a)$ with Eq~\ref{eq:eps_T} using $\delta_T$.
  \For{$(s, a) \in S \times A$}
    \State $(\tilde{s},\tilde{a}) \gets \arg\min \{ \hat{\epsilon}_T(s,a) , \ \min_{\tilde{s}, \tilde{a}} \epsilon_T(\tilde{s}, \tilde{a}) + \Delta((s, a), (\tilde{s}, \tilde{a})) \}$.
    \State $\del{\epsilon}_T(s,a) := \hat{\epsilon}_T(\tilde{s},\tilde{a})$.
    \For{$s' \in S$}
    \State $\tilde{s}' \gets \alpha((s,a,s'),(\tilde{s},\tilde{a}))$.
    \State $\delT(s,a,s') := \hat{T}(\tilde{s},\tilde{a},\tilde{s}')$.
    \EndFor
  \EndFor
 \end{algorithmic}
\end{algorithm}

Given these analogy-based $L1$ confidence intervals, we now define a slightly narrower space of candidate transition probabilities than the space defined by these confidence intervals in order to fully establish the support of certain transitions.
Specifically, we take into account Assumption~\ref{as:negligible_transitions}, to rule out candidates which do not have sufficiently large transition probabilities.
We also make sure that a transition probability is a candidate only if $Z_0$ is closed under it, as assumed in Assumption~\ref{as:z_0}.

\begin{definition}
\label{def:ci}
Given the transition probabilities $\delT$ and confidence interval widths $\del{\epsilon}_T:S \times A \to \mathbb{R}$, we say that $T^\dagger$ is a \textbf{candidate} transition if it satisfies the following for all $(s,a) \in S \times A$:
\begin{enumerate}
    \item $\|T^\dagger(s,a) - \del{T}(s,a)\|_1 \leq \del{\epsilon}_T(s,a)$.
    \item if for some $s'$, $\del{T}(s,a,s') = 0$ and $\del{\epsilon}_T(s,a) < \tau$, then  $T^\dagger(s,a,s') = 0$.
    \item if $(s,a) \in Z_0$, then $\forall s' \notin Z_0$, $T^\dagger(s,a,s') = 0$
\end{enumerate}
 Furthermore, we let $CI(\del{T})$ denote the space of all candidate transition probabilities.
\end{definition}

\subsection{Discounted future state distribution.}

Below we define the notion of a discounted future state distribution (originally defined in \citet{sutton2000policy}), and then describe how we compute it in practice.
We will need this notion in order to compute $\optZgoal$ (discussed in the section titled ``Goal MDP'').

Given an MDP $M^{\dagger}$, policy $\pi$, and state $s$, the discounted future state distribution is defined as follows:
\begin{equation} \label{eq:discounted_state_distribution}
  \rho_{\pi, s}^{M^\dagger}(s', a')= (1-\gamma) \sum_{t=0}^{\infty} \gamma^{t} \Pr\left(s_{t}=s', a_{t}=a' | \pi, s_0 = s \right)
\end{equation}

In words, for any state action pair $(s',a')$, $\rho_{\pi, s}(s', a')$ denotes the sum of discounted probabilities that $(s',a')$ is taken at any $t \geq 0$ following policy $\pi$ from $s$ in $M^\dagger$.

\subparagraph{Computing discounted future state distributions.}
We use a dynamic programming approach to approximate the discounted future state distribution.
Note that we are assuming that the policy $\pi$ is deterministic.

First, for all $\tilde{s} \in S$ and $\tilde{a} \in A$, we set:
\begin{align} \label{eq:init_state_distribution}
  \rho_{\pi, s}^{M^\dagger}(\tilde{s}, \tilde{a}, 0) =
  \begin{cases}
    1 - \gamma & \tilde{s} = s \text{ and } \tilde{a} = \pi(s) \\
    0 & \text{otherwise} \\
  \end{cases}
\end{align}

Then, at each step, $t+1$, we will set:
\begin{equation} \label{eq:computing_state_distribution}
  \rho_{\pi, s}^{M^\dagger}(\tilde{s}, \tilde{a}, t+1) = \rho_{\pi, s}^{M^\dagger}(\tilde{s}, \tilde{a}, t) + \gamma \sum_{\tilde{s}' \in S} T(\tilde{s}', \pi(\tilde{s}'), \tilde{s}) \rho_{\pi, s}^{M^\dagger}(\tilde{s}', \pi(\tilde{s}'), t)
\end{equation}

\subsection{Computing optimistic policies} 
\label{sec:opt_policies}

$\optpigoal$, $\optpiexplore$, and $\optpiswitch$ are the optimisitc policies for three different MDPs, $\Mgoal$, $\Mexplore$, and $\Mswitch$ (described below).
For the theory, we assume that these policies are the true optimistic policy, but in practice this is computed using finite-horizon optimistic form of value iteration introduced in \citet{strehl2008analysis}.
Here we describe this optimistic value iteration procedure.

\subparagraph{Optimistic Value Iteration}

Let $M^\dagger$ be an MDP that is the same as $M$ but with an arbitrary reward function ${R}^\dagger$ and discount factor $\gamma^\dagger$. Then, the optimistic state-action value function is computed as follows.
\begin{align*}
\opt{Q}^{\dagger}(s,a,0) &= 0 \\
\opt{Q}^{\dagger}(s,a,1) &= R^\dagger(s,a) \\
\opt{Q}^{\dagger}(s,a,t) & = R^\dagger(s,a) + \gamma^\dagger \max_{{T}^\dagger \in CI(\delT)} \sum_{s' \in S} T^\dagger(s,a,s') \max_{a' \in A} \opt{Q}^{\dagger}(s',a',t-1), \; \forall t > 0. \numberthis\label{eq:q-values}
\end{align*}

As $t \to \infty$, $\opt{Q}^{\dagger}(s,a, t)$ converges to a value $\opt{Q}^\dagger(s,a)$ since the above mapping is a contraction mapping.  For our theoretical discussion, we assume that we compute these  values for an infinite horizon i.e., we compute $\opt{Q}^{\dagger}(s,a)$.

We then let $\opt{T}^\dagger$ denote the transition probability from $CI(\del{T})$ that corresponds to the optimistic transitions that maximize $\opt{Q}^\dagger$ in Equation~\ref{eq:q-values}. Also, we let $\opt{M}^\dagger$ denote the `optimistic' MDP, $\langle S, A, \opt{T}^\dagger, R^\dagger, \gamma^\dagger \rangle$.


\subparagraph{Goal MDP.}

We define $\Mgoal$ to be an MDP that is the same as $M$, but without the state-action pairs from $\hatZunsafe$ (which is a set of state-action pairs that we will mark as unsafe). More concretely, $\Mgoal = \langle S, A, T, R_{\text{goal}}, \gamma_{\text{goal}} \rangle$, where:
\[
\Rgoal(s,a) =
\begin{cases}
-\infty & (s, a) \in \hatZunsafe \\
R(s,a) & \text{otherwise}.
\end{cases}
\]

We then define $\opt{Q}_{\text{goal}}$ to be the finite-horizon optimistic $Q$-value computed on $\Mgoal$, and $\optpigoal$ the policy dictated by the estimate of $\optQgoal$. Also, let $\opt{T}_{\text{goal}}$ denote the optimistic transition probability and $\opt{M}_{\text{goal}}$ the optimistic MDP.


Using the above quantities, we now describe \textbf{how to compute $\optZgoal$} (which we also summarize in Alg~\ref{alg:compute_z_goal}). Recall that we want $\optZgoal$ to be the set of all state-actions that would be visited with some non-zero probability by following $\optpigoal$ under the optimistic MDP $\opt{M}_{\text{goal}}$. More concretely, for  convenience, first define  $\optrhogoal:= \rho_{\optpigoal, \sinit}^{\opt{M}_{\text{goal}}}$, where $\rho_{\optpigoal, \sinit}^{\opt{M}_{\text{goal}}}$ is as defined in Equation~\ref{eq:discounted_state_distribution}. Then, we would like $\optZgoal$ to be the set $\{(s, a) \in S \times A : \optrhogoal(s, a) > 0 \}$. 

However, directly computing this infinite-horizon estimate in practice is impractical. Instead, here we make use of Lemma \ref{lem:rho_horizon_bound} and Corollary~\ref{cor:optzgoal_correctness}, which allow us to exactly compute $\optZgoal$ through a finite-horizon estimate of $\optrhogoal$.
Specifically, consider the finite-horizon estimate of $\optrhogoal$ (i.e., the finite horizon estimate of $\rho_{\optpigoal, \sinit}^{\opt{M}_{\text{goal}}}$ as defined in Equation~\ref{eq:init_state_distribution} and Equation~\ref{eq:computing_state_distribution}) which can be computed using dynamic programming. In Lemma \ref{lem:rho_horizon_bound}, we show that if we set the horizon $H \geq |S|$, then $\optrhogoal(s, a, H) > 0$ if and only if $\optrhogoal (s, a) > 0$. Hence, we fix $H$ to be any value greater than or equal to $|S|$ and then compute
\begin{equation}\label{eq:optzgoal}
\optZgoal := \{(s, a) \in S \times A : \optrhogoal(s, a, H) > 0. \}
\end{equation}

As stated in Corollary~\ref{cor:optzgoal_correctness}, this will guarantee what we need, namely that $\optZgoal = \{(s, a) \in S \times A : \optrhogoal(s, a) > 0 \}$. We summarize this algorithm in Algorithm~\ref{alg:compute_z_goal}

\subparagraph{Explore MDP.}

We define $\Mexplore = \langle S, A, T, \Rexplore, \gamma_{\text{explore}} \rangle$ to be an MDP with the same states, actions, and transition function as $M$, but with a different reward function, $\Rexplore$ (computed in Algorithm \ref{alg:z_explore}), and discount factor, $\gamma_{\text{explore}}$.
$\Rexplore$ is defined as follows:

\[
\Rexplore(s,a) = \begin{cases}
  1 & (s,a) \in \Zexplore \\
  0 & (s,a) \in \hatZsafe \setminus \Zexplore \\
  -\infty & \text{otherwise}.
\end{cases}
\]

\subparagraph{Switch MDP.}

We define $\Mswitch = \langle S, A, T, \Rswitch, \gamma_{\text{switch}} \rangle$ to be an MDP with the same states, actions, and transition function as $M$, but with a different reward function, $\Rswitch$, and discount factor, $\gamma_{\text{switch}}$.
More specifically, $\Rswitch$ is defined as follows:

\[
\Rswitch(s,a) = \begin{cases}
  1 & (s,a) \in \optZgoal \\
  0 & (s,a) \in \hatZsafe \setminus \optZgoal \\
  -\infty & \text{otherwise}.
\end{cases}
\]

\subsection{Regarding Safe Islands}
\label{sec:safe_islands}
Here, we elaborate on the motivation behind involving the notion of returnability (a) in Assumption~\ref{as:sim} and (b) in the definition of $\Zsafe$ in Definition~\ref{def:Zsafe}. 

\textbf{Returnability in Assumption~\ref{as:sim}.} Recall that a key motivation behind Assumption~\ref{as:sim} was that in order to add a safe subset of state-actions to the current safe set, it is necessary for the agent to establish that subset's returnability i.e., establish a return path from the to-be-added subset to the current safe set. 
Here we explain why this is necessary.
Consider a hypothetical agent that tries to expand its safe set without ensuring that whatever it adds to the safe set is returnable.  Such an agent might venture into a safe island: although the agent knows that the subset of state-action pairs it has entered into is safe, the agent does not know of any safe path from that subset back to the original safe set.
There are two distinct kinds of such safe islands.
The first is where there is truly no safe return path; the second is where there does, in fact, exist some safe return path, but the agent has not yet established that this path is indeed safe.
We will refer to these islands as True Safe Islands and False Safe Islands. 

Although entering into a True Safe Island is not a problem for ensuring optimality in the PAC-MDP sense, entering into a False Safe Island creates trouble. More concretely, in a True Safe Island, since there is no safe way to leave such an island, even the safe-optimal policy {\em must} remain on this True Safe Island. Thus, the agent that has ventured into a True Safe Island, can potentially find the $\epsilon$-optimal policy, even though it may be forever stuck in this island.
However, in a False Safe Island, 
since there is indeed a safe path to leave this island, it can be the case that the safe-optimal policy from this island will leave the island (and then achieve far higher future reward, than a policy confined to the False Safe Island). Hence, for the agent to be PAC-MDP optimal, it must first establish safety of this path.  However,
for an agent stuck inside this island, there may be no means to establish safety of that path simply by exploring that island --  unless the island is rich enough with analogous states like $Z_0$ is (which may not be the case if this happens to be a tiny island). 
Thus, the agent could be forever stuck in the False Safe Island and even worse, it might act $\epsilon$-sub-optimally forever (by choosing to remain instead of exiting). 
 Hence, it's necessary for the agent to establish returnability of any state-action pair before adding it (and Assumption~\ref{as:sim} enables us to do this).

\textbf{Returnability in Definition~\ref{def:Zsafe}.} Next, we explain the motivation behind defining  $\Zsafe$ in Definition~\ref{def:Zsafe} to be a ``returnable'' set. Specifically, recall that $\Zsafe$ is a safe subset of state-actions, and we would like to compete against the optimal policy on this subset; more importantly, we defined  this in a way that any state-action pair in this set is returnable, meaning that it has a return path to $Z_0$. 

Consider the hypothetical scenario where $\Zsafe$ is defined to allow non-returnable state-actions. 
Here, we argue that the agent will have to navigate some impractical complications. To begin with, this alternative definition of $\Zsafe$ could mean that the safe-optimal policy may lead one into True Safe Islands i.e., safe subsets of state-actions from which there is no safe path back to $Z_0$. This in turn could potentially require the agent to enter into a True Safe Island in order to be PAC-MDP-optimal. Therefore, when the agent expands its safe set, it is necessary for it to find True Safe Islands and add them to the safe set; while doing so, crucially, as discussed earlier, the agent must also avoid adding False Safe Islands to the safe set. Then, in order to meet these two objectives, the agent should consider every possible safe island and consider all its possible return paths, and establish their safety. If it can be established that no safe return paths exist for a particular safe island, the agent can label the island as a True Safe Island and add it to its safe set.

Thus, in theory, the above fairly exhaustive algorithm can address the more liberal definition of $\Zsafe$; however, in many practical settings, it may be expensive to fully determine the safety of every state-action pair. Hence, we choose to ignore this situation by enforcing that $\Zsafe$ is returnable. With this framework, our algorithm can grow the safe set by establishing return paths from the edges of the safe set (as against having to also look for safe islands and establish safety of all their possible return paths).

\subsection{MDP metrics and our analogy function}
\label{sec:metrics}
State-action similarities have been used outside of the safety literature in order to improve computation time of planning and sample complexity of exploration.
Bisimulation seeks to aggregate states into groupings of states that have similar dynamics or similar Q-values \citep{givan2003equivalence, taylor2009bounding, abel2017near, kakade2003exploration}.
These state aggregations allow for more efficient planning and exploration \citep{kakade2003exploration}.
Other work has used pseudo-counts to learn approximate state aggregations \citep{taiga2018approximate}.
The reason we do not use state-aggregation methods for transferring dynamics knowledge is that we want to include environments where similarities cannot easily partition the state space, such as situations where the similarity between two states is proportional to their distance.

\section{Proof Outline}
\label{sec:proof_outline}

The following subsections describe the overall techniques and intuition, and serve as a rough sketch of the proof of Theorem~\ref{thm:pac_and_safe}.

\subsection{Establishing Safety} \label{sec:hatZsafe}

We now highlight the key algorithmic aspects which ensure provably safe learning, in other words, that (w.h.p) the agent always experiences only non-negative rewards. Recall that our agent maintains a safe set $\hatZsafe$, and in order to add new state-action pairs to $\hatZsafe$ while ensuring that $\hatZsafe$ is closed, we must be able to determine a ``safe return policy'' to $\hatZsafe$.
However, doing this in a setting with unknown stochastic dynamics poses a significant challenge:  we must be able to find a return policy where, for every state-action pair in the return path, we know the exact support of its next state; furthermore, all these state-action pairs should return to $\hatZsafe$ with probability $1$.  Below, we lay out the key aspects of our approach to tackling this.



\textbf{``Transfer'' of confidence intervals.} As a first step, we start by establishing confidence intervals on the transition distributions of all state-action pairs as described below. 
Let $\hat{T}$ denote the empirical transition probabilities.
Just as in \citet{strehl2008analysis}, we can compute $L1$ confidence intervals of these estimates using the Hoeffding bound (details in Appendix~\ref{sec:ci}).
Let $\hat{\epsilon}_T(s,a)$ denote the $L1$ confidence interval for the empirical transition probability $\hat{T}(s,a)$.
Using the provided distance and analogy function $\Delta$ and $\alpha$, and using simple triangle inequalities, we can then derive tighter confidence intervals $\del{\epsilon}_T$ (centered around an estimated $\delT$) as in Alg~\ref{alg:tighter_ci}.
The idea here is to ``transfer'' the confidence interval from a sufficiently similar, well-explored state-action pair to an under-explored state-action pair, using analogies.

\textbf{Learning the next-state support.}
Crucially, we can use these transferred confidence intervals to infer the support of state-action pairs we have not experienced. More concretely, in Lemma~\ref{lem:del-support} we show that, when a confidence interval is sufficiently tight, specifically when $\del{\epsilon}(s,a) \leq \tau/2$ for some $(s,a)$ (where $\tau$ is the smallest non-zero transition probability defined in Assumption~\ref{as:negligible_transitions}), we can exactly recover the support of the next state distribution of $(s,a)$.
This fact is then exploited by Algorithm~\ref{alg:compute_safe_set} to expand the safe set whenever the confidence intervals are updated.

\textbf{Correctness of $\hatZsafe$.}
To expand $\hatZsafe$ while ensuring that it is safe and communicating, Algorithm~\ref{alg:compute_safe_set} first creates a candidate set, $\Zcandidate$, of all state-action pairs $(s, a)$ with sufficiently tight confidence intervals and non-negative rewards (and so, we know their next state supports).
The algorithm then executes three (inner) loops each of which prunes this candidate set. To ensure communicatingness, the first loop
eliminates candidates that have no probability of reachability from $\hatZsafe$, and the second loop eliminates those from which there is no probability of return to $\hatZsafe$. In order to ensure closedness, the third loop eliminates those that potentially lead us outside of $\hatZsafe$ or the remaining candidates.  We repeat these three loops until convergence. We prove in Lemmas~\ref{lem:hatZsafe_communicating} and \ref{lem:hatZsafe_safe} that Algorithm~\ref{alg:compute_safe_set} correctly maintains the safety and communicatingness of $\hatZsafe$ and in Lemma~\ref{lem:compute_safeset_terminates} that the algorithm terminates in polynomial time. Note that in Theorem~\ref{thm:pac_and_safe}, we prove that (w.h.p.) our agent always picks actions only from $\hatZsafe$.


\textbf{Completeness of $\hatZsafe$.}
While the above aspects ensure correctness of Algorithm~\ref{alg:compute_safe_set}, these would be satisfied even by a trivial algorithm that always only returns $Z_0$. Hence, it is important to establish that for given set of confidence intervals, $\hatZsafe$ is ``as large as it can be''. More concretely, consider any state on the edge of $\hatZsafe$ for which there exists a return policy to $\hatZsafe$ which passes only through (non-negative reward) state-action pairs with confidence intervals at most $\tau/2$; this means that we know all possible trajectories in this policy, and all of these lead to $\hatZsafe$. In such a case, we show in Lemma~\ref{lem:safe_set_large_enough} that Algorithm~\ref{alg:compute_safe_set} does indeed add this edge action {\em and all of the actions in every possible return trajectory} to $\hatZsafe$.

\subsection{Guided Exploration}
\label{app:guided-exploration}

To be able to add a state-action pair to our conservative estimate of the safe-set, $\hatZsafe$, we not only need to tighten the confidence intervals of that state-action pair but also that of every state-action in all its return trajectories to $\hatZsafe$. Observe that this can be accomplished by exploring state-action pairs inside $\hatZsafe$ that are similar to this return path, and using analogies to transfer their confidence intervals. However, this raises two main algorithmic challenges.

\textbf{Selecting unexplored actions for establishing safety.}
First, for which unexplored state-action pairs outside $\hatZsafe$ do we want to establish safety? Instead of expanding $\hatZsafe$ arbitrarily, we will keep in mind the objective mentioned in our outline of ASE: we want to expand $\hatZsafe$ so that we can get to a stage where every possible trajectory when following the optimistic goal policy, $\optpigoal$, is guaranteed to be safe, allowing the agent to safely follow $\optpigoal$.
By letting $\optZgoal$ denote the set of all state-action pairs on any path following $\optpigoal$ from the initial state $\sinit$, this condition can be equivalently stated as $\optZgoal \subset \hatZsafe$.

So, to carefully select such unexplored state-action pairs, ASE calls Algorithm~\ref{alg:z_goal}, which is an iterative procedure: in each iteration, it first (re)computes the optimistic goal policy $\optpigoal$ and the set $\optZgoal$. Using this, it then creates a set $\Zedge$, which is the intersection of $\optZgoal$ and the set of all edge state-action pairs of $\hatZsafe$. We then hope to establish safety of $\Zedge$, so that, intuitively, we can {\em expand the frontier of our safe set only along the direction of the optimistic path}. To this end, Algorithm~\ref{alg:z_goal} calls Algorithm~\ref{alg:z_explore} to compute a corresponding $\Zexplore \subset \hatZsafe$ to explore (we will describe Algorithm~\ref{alg:z_explore} shortly).

Now, in the case $\Zexplore$ is non-empty, Algorithm~\ref{alg:z_goal} returns control back to ASE, for it to pursue $\optpiexplore$ -- and Lemma~\ref{lem:escape} shows that $\optpiexplore$ indeed explores $\Zexplore$ in poly-time. But if $\Zexplore$ is empty, Algorithm~\ref{alg:z_goal} adds all of $\Zedge$ to $\hatZunsafe$; in the next iteration, $\optpigoal$ is updated to ignore $\hatZunsafe$.
In Lemma~\ref{lem:correctness_z_explore}, we use Assumption~\ref{as:sim}, to prove that the elements added to $\hatZunsafe$ are indeed elements that do not belong to $\Zsafe$ (and so we can confidently ignore $\hatZunsafe$ while computing $\optpigoal$).
In Lemma~\ref{lem:correctness_z_goal}, we show that this iterative approach terminates in poly-time and either returns a non-empty $\Zexplore$ that can be explored by $\optpiexplore$, or updates $\optpigoal$ in a way that $\optZgoal \subset \hatZsafe$.
In the case that 
 $\optZgoal \subset \hatZsafe$, using Lemma~\ref{lem:escape}, we show that the agent first takes $\optpiswitch$ to enter into $\optZgoal$ in finite time, so that the agent can pursue $\optpigoal$.




\textbf{Selecting safe actions for exploration.}
To establish safety of an unexplored $(s,a) \in \Zedge$, we must explore state-action pairs from $\hatZsafe$ that are similar to state-actions along an {\em unknown} return policy from $(s,a)$ in order to learn that {\em unknown} policy. While such a policy does exist if $(s, a) \in \Zsafe$ (according to Assumption~\ref{as:sim}), the challenge is to resolve this circularity, without exploring $\hatZsafe$ exhaustively.



Instead, Algorithm~\ref{alg:z_explore} uses a breadth-first-search (BFS) from $(s,a)$ which essentially enumerates a superset of trajectories that contains the true return trajectories. Specifically, it first enumerates a list of state-action pairs that are a 1-hop distance away and if any of them have a loose confidence interval, it adds  to $\Zexplore$ a corresponding similar state-action pair from $\hatZsafe$ (if any exist). If $\Zexplore$ is empty at this point, Algorithm~\ref{alg:z_explore} repeats this process for 2-hop distance, 3-hop distance and so on, until either $\Zexplore$ is non-empty or the BFS tree cannot be grown any further.  Lemma~\ref{lem:correctness_z_explore} argues that this procedure does populate $\Zexplore$ with all the state-action pairs necessary to establish the required return paths;  Lemma~\ref{lem:z_explore_terminates} demonstrates its polynomial run-time. Although we cannot guarantee that this method does not explore all of $\hatZsafe$, we do see this empirically, as we show in our experiments (see Section~\ref{sec:experiments}).

\section{Proofs}
\label{sec:proof}

This section details our proof of Theorem~\ref{thm:pac_and_safe}, that ASE is guaranteed to be safe with high probability and is optimal in the PAC-MDP sense.
We start by restating Theorem~\ref{thm:pac_and_safe} and proving it.
We then examine proofs for the correctness and polynomial computation time of Algorithm~\ref{alg:compute_safe_set}.
Specifically, we show that the computed $\hatZsafe$ is closed and communicating.
Next we show that Algorithms~\ref{alg:z_explore} and \ref{alg:z_goal} correctly compute the desired $\Zexplore$ and an estimate of $\optZgoal$ in polynomial time.
The following section shows that the computed estimate of $\optZgoal$ is in fact correct, under certain conditions.
Subsection \ref{sec:proofs:policies} provides the key lemmas for proving PAC-MDP, namely that that our agent, following $\optpigoal$, $\optpiexplore$, or $\optpiswitch$ either performs the desired behavior (acting optimally or reaching certain state-action pairs) or learns something new about the transition function.
By bounding the number of times our agent learns something new, we can show that the agent follows the optimal after a polynomial number of steps.
The final subsection provides and proves additional supporting lemmas.

\pacmdp*

  
   
   

\begin{proof}
\item 
\subparagraph{Proof of admissibility.}
We first establish the probability with which our confidence intervals remain admissible throughout the entire execution of the algorithm.
Note that we only calculate each confidence interval $m$ times for every state-action pair. 
Thus, by the union bound and our choice of $\delta_T$, the confidence intervals defined by $\hat{T}$ and $\hat{\epsilon}_T$ hold with probability $1 - \delta/2$. Then, by the triangle inequality, even the tighter confidence intervals computed by Algorithm~\ref{alg:tighter_ci}  -- defined by $\del{T}$ and $\del{\epsilon}_{T}$ -- are admissible. 

\subparagraph{Proof of safety.}
Next we will show that, given that the confidence intervals are admissible, the algorithm never takes a state-action pair outside of $\Zsafe$.
Corollary~\ref{cor:hatZsafe_subset}, Lemma~\ref{lem:hatZsafe_safe} and \ref{lem:hatZsafe_communicating} together show that $\hatZsafe$ is a safe (which also implies, closed), communicating subset of $\Zsafe$. 
Using this, we will inductively show that the agent is always safe under our algorithm. 
Specifically, assume that at any time instant, starting from $s_0$, the agent has so far only taken actions from $\hatZsafe$. Since $\hatZsafe$ is closed and safe, this means that the agent has so far been safe, and is currently at $s_t \in \hatZsafe$. We must establish that even now the agent takes an action $a_t$ such that $(s_t,a_t) \in \hatZsafe$.

Now, at each step, recall that according to Algorithm~\ref{alg:ase}, the agent follows either $\optpiexplore$, $\optpiswitch$, or $\optpigoal$. Consider the case when the agent follows either $\optpiexplore$ or $\optpiswitch$. 

In this case, 
for all $(s,a) \notin \hatZsafe$ the rewards are set to be $-\infty$, and for all $(s,a) \in \hatZsafe$ the rewards are set to be non-negative. To address both $\optpiexplore$ and $\optpiswitch$ together, let the assigned rewards  be $R^\dagger$.

Now, for any $t \geq 0$, recall that $\opt{Q}^\dagger(s,a,t)$ denotes the estimate of $\opt{Q}^\dagger(s,a)$ after $t$ iterations of dynamic programming (not to be confused with the finite-horizon value of the optimistic policy).
We first claim that, for all $(s,a) \in \hatZsafe$ and for all iterations $t \geq 1$, the resulting optimistic Q-values are such that $\opt{Q}^\dagger(s,a,t) \geq 0$ if $(s,a) \in \hatZsafe$ and $\opt{Q}^\dagger(s,a,t) = -\infty$ otherwise. 

We prove this claim by induction on $t$, assuming that $\opt{Q}^\dagger(s,a,t)$ is initialized to some non-negative value for $t=0$.
For $t=1$, our claim is satisfied because the Q-values equal to the sum of the reward function and some positive quantity. 

Consider any $t > 1$ and $(s,a) \in \hatZsafe$. We know from Equation~\ref{eq:q-values}
that the Q-value for this horizon can be decomposed into a sum of the reward  and the maximum Q-value of the next states (with a positive, multiplicative discount factor).
For $(s,a) \in \hatZsafe$, in Equation~\ref{eq:q-values}, we will have that the first term, which is the reward function, is non-negative. The second term is an expectation over the maximum $Q$-values (for a horizon of $t-1$), where the expectation corresponds to the probability distribution of $\opt{T}^\dagger(s,a,\cdot)$ over the next states.
Since $(s,a)\in \hatZsafe$ and since $\opt{T}^\dagger$ is a candidate transition function, by Corollary~\ref{cor:closed_under_candidate}, all the next states according to this transition function, belong to $\hatZsafe$.
Now, for any $s' \in \hatZsafe$, there exists $a'$ such that $(s',a') \in \hatZsafe$. By induction, we have $\opt{Q}^\dagger(s', a', t-1) \geq 0$, and hence $\max_{a''} \opt{Q}^\dagger(s', a'', t-1) \geq 0$.
Hence, even the second term in the expansion of $\opt{Q}^\dagger(s,a,t)$ is non-negative, implying that $\opt{Q}^\dagger(s,a,t) \geq 0$. Now, for any $(s,a) \notin \hatZsafe$, it follows trivially from Equation~\ref{eq:q-values} that $\opt{Q}^\dagger(s,a,t) = -\infty$ since the reward is set to be $-\infty$.

Thus, for any state $s \in \hatZsafe$, we have established that there exists $a$ such that $(s,a) \in \hatZsafe$ and $\opt{Q}^\dagger(s,a) \geq 0$ . Furthermore, for any action $a'$ such that $(s,a') \notin \hatZsafe$, $\opt{Q}^\dagger(s,a') = -\infty$. Therefore, $\optpi^\dagger(s)$ must be an action $a$ such that $(s,a) \in \hatZsafe$. This proves that $\optpi^\dagger \in \Pi(\hatZsafe)$. In other words, this means that at $s \in \hatZsafe$ the agent takes an action $a$ such that $(s,a) \in \hatZsafe$. This completes our argument for $\optpiexplore$ and $\optpiswitch$.

As the final case, consider a time instant when the agent follows $\pigoal$. By design of Algorithm~\ref{alg:ase}, we know that this happens only if $\optZgoal \subset \hatZsafe$ and $s \in \optZgoal$. Now, by definition of $\optZgoal$,  since $s \in \optZgoal$, we know that $(s, \optpigoal(s)) \in \optZgoal$. Then, since $\optZgoal \subset \hatZsafe$, $(s,\optpigoal(s)) \in \hatZsafe$, implying that the algorithm picks only safe actions, even when it follows $\optpigoal$.

\subparagraph{Proof of PAC-MDP.}

Now we will show that ASE is PAC-MDP.
To do this, we will show that at any step of the algorithm, assuming our confidence intervals are admissible, the agent will either act $\epsilon$-optimally or reach a state outside of the known set $K = \{(s, a) \in S \times A : n(s, a) \geq m \}$ in some polynomial number of steps with some positive polynomial probability.
To prove this, recall the agent follows either $\optpiexplore$, $\optpiswitch$, or $\optpigoal$ in three mutually exclusive cases; let us examine each of these three cases. 

\subparagraph{Case 1: $\optZgoal \not \subset \hatZsafe$. }
If $\optZgoal \not \subset \hatZsafe$, then the agent follows $\optpiexplore$. In this case, we will show that the agent will experience a state-action pair from $K^c$ (the complement of $K$) in the first $H_1$ steps, where $H_1 = O(H^2/c)$.

First note that the condition $\optZgoal \not \subset \hatZsafe$ can only change if $\optZgoal$ or $\hatZsafe$ are modified, which can only happen if the agent experiences a state action pair outside $K$; so, if before the agent takes its $H_1$th step, this condition changes, we know that the agent has experienced a state-action pair outside $K$, and hence, we are done. 

Consider the case when the agent does not experience any element of $K^c$ in the first $H_1-1$ steps;  hence the agent follows a fixed $\optpiexplore$ for these steps. By Lemma \ref{lem:correctness_z_goal}, since $\optZgoal \not \subset \hatZsafe$, there must exist some element in $\Zexplore$ (where $\Zexplore \subset \hatZsafe$ by design of Algorithm~\ref{alg:z_explore}). Now, recall that $\optpiexplore$ is computed using rewards $\Rexplore$, which are set to $1$ on $\Zexplore$, and either $0$ or $-\infty$ otherwise, depending on whether the state-action pair is in $\hatZsafe$ or not. If we define $\Mexplore = \langle S, A, T, \Rexplore, \gamma_{\text{explore}}\rangle$, then we can invoke Lemma~\ref{lem:escape} for $\Mexplore$ to establish that the agent reaches $\Zexplore$ or $K^c$. 

To do this, we must establish that all requirements of Lemma~\ref{lem:escape} hold.
In particular, we have $\Zexplore \subset \hatZsafe$ by design of Algorithm~\ref{alg:z_explore}.
We also have $\hatZsafe$ is communicating (Lemma~\ref{lem:hatZsafe_communicating}) and closed (Lemma~\ref{lem:hatZsafe_safe}) and that $\optpiexplore \in \Pi(\hatZsafe)$ (from our proof for safety of Algorithm~\ref{alg:ase}).
Finally, since $m$ is sufficiently large,
by Lemma \ref{lem:escape}, the agent will reach $\Zexplore$ or $K^c$ in $H_1 = O(H^2/c)$ steps with probability at least $1/2$ as long as $H \geq \Hcom {\log \frac{16\Hcom}{c}}/{\log \frac{1}{c}}$.
Note that although Lemma~\ref{lem:escape} guarantees this for the behavior of the agent on $\Mexplore$, the same would apply for $M$ as well, since both these MDPs share the same transition. Finally, note that since $\Zexplore \notin K$ (by Lemma~\ref{lem:z_explore_K}), this means that the agent escapes $K$ in $H_1$ steps.

\subparagraph{Case 2: $\optZgoal \subset \hatZsafe$, $s_t \not \in \optZgoal$. }
Now consider the next mutually exclusive case where $\optZgoal \subset \hatZsafe$ but the current state $s \not \in \optZgoal$. In this case our agent will attempt to return to $\optZgoal$ by following $\optpiswitch$. In this case, we argue that, in the next $H_2 = O(H^2/c)$ steps, the agent either does reach $\optZgoal$ or experiences a state-action pair in $K^c$. To see why, note that the current condition can only change if $\optZgoal$ or $\hatZsafe$ change or if the agent reaches a state $s \in \optZgoal$. As noted before, $\optZgoal$ or $\hatZsafe$ are modified only if the agent experiences a state action pair outside $K$; so, if before the agent takes its $H_2$th step, this condition changes, we know that the agent has either
experienced a state-action pair outside $K$ or has reached $\optZgoal$, and hence, we are done. 

Consider the case when the agent does not experience any element of $K^c$ in the first $H_2-1$ steps, and hence follows a fixed $\optpiswitch$ for these steps. Since $\optZgoal \ne \emptyset$ (which is trivially true since $s_0 \in \optZgoal$ always), using the same reasoning as the previous case, we can again use Lemma \ref{lem:escape} to show that the agent will reach a state-action pair in $\optZgoal$ or outside of $K$ in $H_2 = O(H^2/c)$ steps with probability at least $1/2$, since $m$ is sufficiently large and $H \geq \Hcom {\log \frac{16\Hcom}{c}}/{\log \frac{1}{c}}$.

\subparagraph{Case 3: $\optZgoal \subset \hatZsafe$, $s_t \in \optZgoal$. }
Finally, we consider the last case where $\optZgoal \subset \hatZsafe$ and the current state $s \in \optZgoal$. In this case, we argue that the agent either takes an action that is near-optimal, or in the next $H$ steps, it reaches $K^c$ with sufficiently large probability. 



Let $\Pr(A_M)$ be the probability that starting at this step, the Algorithm~\ref{alg:ase} leads the agent out of $K$ in $H$ steps, conditioned on the history $p_t$. 
Now, if $\Pr(A_M) \geq \epsilon(1-\gamma)/4$, the agent will escape $K$ in $H$ steps with sufficient probability. 
Hence, consider the case when $\Pr(A_M) \leq \epsilon(1-\gamma)/4$.

Then, assuming $H \geq O\left(\frac{1}{1-\gamma} \ln \frac{1}{\epsilon (1-\gamma)}\right)$ and sufficiently large $m$, we can use the above probability bounds and Lemma \ref{lem:reduction_to_mbie} to show that in this
case 
the state $s_t$ that the agent currently is in satisfies:
\begin{equation}
    V_{M}^{\mathcal{A}}(p_t) \geq V_{\Mgoal}^*(s_t) - {\epsilon}.
    \label{eq:thm1-A}
\end{equation}

To complete our discussion of this case, we need to lower bound the right hand side in terms of the value of the safe-optimal policy $\pisafe^*$ on the true MDP $M$.
Recall that $\pisafe^*$ is a policy that maximizes $V_{M}^{\pisafe^\star}$ subject to the constraint that $\pisafe^* \in \Pi(\Zsafe)$.
Now, consider an MDP $\Mgoal^*$ with the same transitions as $M$.
However it has rewards $\Rgoal^*$ such that for all $(s,a) \notin \Zsafe$, $\Rgoal^*(s,a) = -\infty$ and everywhere else $\Rgoal^*(s,a) = R(s,a)$.  Now, since $\pisafe^* \in \Pi(\Zsafe)$, from Fact~\ref{fact:closed_policy}, we have that following $\pisafe^*$ from any $s \in \Zsafe$, the agent would never exit $\Zsafe$. Hence, for any $s \in \Zsafe$, $V_{\Mgoal^*}^{\pisafe^*}(s) = V_{M}^{\pisafe^*}(s)$. Note that this equality applies to the current state $s_t$ since it is inside $\optZgoal$, and we know that $\optZgoal \subset \hatZsafe$ and $\hatZsafe \subset \Zsafe$ (by Corollary~\ref{cor:hatZsafe_subset}).

Now, let us compare $\Mgoal^*$ and $\Mgoal$. Recall that $\Mgoal$ has its rewards set to $-\infty$ only on $\hatZunsafe$ (and everywhere else, it equals $R(s,a)$). By Lemma~\ref{lem:correctness_z_explore}, we know that $\hatZunsafe \cap \Zsafe = \emptyset$, and therefore $\hatZunsafe \subset \Zsafe^c$. Thus, both $\Mgoal$ and $\Mgoal^\star$ have the same rewards as $M$, except $\Mgoal$ has the rewards set to $-\infty$ on a set $\hatZunsafe$, while $\Mgoal^\star$ has rewards set to $-\infty$ on a superset of $\hatZunsafe$, $\Zsafe^c$. In other words, the rewards of $\Mgoal$ are greater than or equal to the rewards of $\Mgoal^*$. Thus, the value of the optimal policy on $\Mgoal$ cannot be less than that of $\Mgoal^*$. Formally, for all $s$, $V^*_{\Mgoal}(s) \geq V_{\Mgoal^*}^{\pisafe^*}(s)$. Since we also have $V_{\Mgoal^*}^{\pisafe^*}(s_t) = V_{M}^{\pisafe^*}(s_t)$, we get:
\begin{equation}
    V_{\Mgoal}^*(s_t) \geq V_{M}^{\pisafe^*}(s_t).
        \label{eq:thm1-B}
\end{equation}

Thus, from Equations~\ref{eq:thm1-A} and ~\ref{eq:thm1-B}, we get that at this state, $    V_{\Mgoal}^{\mathcal{A}}(p_t) \geq V_{M}^{\pisafe^*}(s_t) - \epsilon.$\\


In summary, the agent does at least one of the following at any timestep:\\

\begin{enumerate}
    \item reach $K^c$ in $H_1$ steps (starting from Case 1) with probability at least $1/2$
    \item reach $K^c$ in $H_2$ steps (starting from Case 2),  with probability at least $1/4$
    \item reach $K^c$ in $H$ steps (starting from Case 3) with probability at least $\Omega(\epsilon ( 1 - \gamma))$.
    \item reach $\optZgoal$ in $H_2$ steps (starting from Case 2) with probability at least $1/4$.
    \item take a nearly-optimal action (in Case 3).\\
\end{enumerate}

Note that in the above list, we have slightly modified the guarantee from Case 2. In particular,
Case 2 guaranteed that with a probability of $1/2$ the agent would reach either $K^c$ or $\optZgoal$; from this we have concluded that at least one of these events would have a probability of at least $1/4$. 

We first upper bound the number of sub-optimal steps corresponding to the first three events. Observe that the agent can only experience a state-action pair outside of $K$ a total of $m |S| |A|$ times.
Now, by the Hoeffding bound, we have that 
it takes at most  $O \left (m |S| |A|  \frac{1}{ \epsilon ( 1 - \gamma) } \ln \frac{1}{\delta} \right )$ independent trials to  see $m |S| |A|$ heads in a coin that has a probability of at least $\Omega(\min\left(1/4, \epsilon(1-\gamma)\right))=\Omega(\epsilon(1-\gamma))$ as turning out to be  heads. Hence, these trajectories would correspond to at most  $O \left (\max(H_1, H_2,H) \cdot m |S| |A| \cdot  \frac{1}{ \epsilon ( 1 - \gamma) } \ln \frac{1}{\delta} \right )$ many sub-optimal steps.

Next, we upper bound the number of sub-optimal steps corresponding to the fourth event. Consider two successful occurrences of this event. That is, in both instances,  the agent did indeed reach $\optZgoal$. In such a case, there must be at least one time instant between these two occurrences when the agent experienced $K^c$; if not, after the first occurrence, by design of Algorithm~\ref{alg:ase}, the agent would have remained in $\optZgoal$, thereby precluding the second occurrence of the event from happening. Thus, there can be at most as many successful occurrences of this event as $m|S| |A| +1$. Again, by a Hoeffding bound, this corresponds to $O \left (H_2 \times m |S| |A|  \frac{1}{ \epsilon ( 1 - \gamma) } \ln \frac{1}{\delta} \right )$ many sub-optimal steps, with probability $1 - \delta/2$. 

Combining the above two bounds and plugging in the bound on $m$ gives our final bound on the number of sub-optimal steps.

\end{proof}

\subsection{Proofs about Algorithm~\ref{alg:compute_safe_set}}
\label{sec:proof:compute_safe_set}
Recall that Algorithm~\ref{alg:compute_safe_set} expands the set of safe state-action pairs $\hatZsafe$ (by making use of an updated set of confidence intervals) by first creating a candidate set and then iteratively pruning the set until it stops changing.


First, in Lemma~\ref{lem:compute_safeset_terminates}, we establish this procedure terminates in polynomial time. 
In the lemmas that follow after that, we prove soundness and completeness. In particular, in Lemma~\ref{lem:hatZsafe_safe}, we establish that the updated $\hatZsafe$ is indeed safe. In Lemma~\ref{lem:hatZsafe_communicating}, we establish that $\hatZsafe$ is communicating and as a result of which in Corollary~\ref{cor:hatZsafe_subset}, we establish that $\hatZsafe \subset \Zsafe$. Finally, in Lemma~\ref{lem:safe_set_large_enough} we prove completeness in that, if there exists a state-action pair on the edge of $\hatZsafe$, and if there exists a return path from that edge that only takes state-action pairs whose confidence intervals are sufficiently small, then that state-action pair and all of that return path is added to $\hatZsafe$.



\begin{lemma}
  \label{lem:compute_safeset_terminates}
  Algorithm~\ref{alg:compute_safe_set} terminates in $poly(|S|, |A|)$ time.
\end{lemma}

\begin{proof}
The algorithm begins with a set $\Zcandidate$ with $|S| \cdot |A|$  many elements. Now in each outer iteration, $\Zcandidate$ either ends up losing some elements, or remains the same. If it does remain the same, then we break out of the outer loop. Thus, there can be at most $\mathcal{O}(|S| \cdot |A|)$ many iterations of the outer loop. 

Next, consider the first inner while block. In each iteration of the while loop, we either add an element to $\Zreachable$ or break out of the loop if $\Zreachable$ does not change. Thus, there can be at most $\mathcal{O}(|S| \cdot |A|)$ many iterations of this while loop. As for the time complexity of the inner while loops, since $\hatZsafe \cup \Zreachable$ is a finite set, it is easy to see that these iterations terminate in polynomial time. A similar argument holds for the next while block too. For the third while block, we must apply a slightly different version of this argument where we make use of the fact that in each run of the while loop, we either remove a state-action pair from $\Zclosed$ or break out of the loop. From the above arguments, it follows that the algorithm terminates in polynomial time.
\end{proof}

\textbf{Note:} In the following discussions, unless otherwise specified, $\Zclosed$ denotes the set as it is in the last step of Algorithm \ref{alg:compute_safe_set}.

Before we prove our other lemmas about Algorithm~\ref{alg:compute_safe_set}, we first establish a result about Algorithm~\ref{alg:tighter_ci} that computes the tighter confidence intervals. Specifically we show that these confidence intervals are computed in a way that if a particular interval is sufficiently tight, then every candidate transition probability in that interval has the same support of next state-action pairs as the true support.

\begin{lemma}
\label{lem:del-support}
  Assume the confidence intervals are admissible. Then, for any $(s,a)$ such that $\del{\epsilon}_T(s,a) < \tau/2$, for all $s' \in S$, $\del{T}(s,a,s') > 0$ if and only if $T(s,a,s') > 0$. 
\end{lemma}
\begin{proof}
First consider the case when $T(s,a,s') = 0$; we will show that $\del{T}(s,a,s') = 0$. 
To see why, consider the $(\tilde{s}, \tilde{a})$ that contributed to this confidence interval as computed in the step of Algorithm~\ref{alg:tighter_ci}.
Since $\del{\epsilon}_T(s,a) < \tau/2$, from Algorithm~\ref{alg:tighter_ci}, we have that $\Delta((s,a), (\tilde{s}, \tilde{a})) \leq \tau/2$.
This implies that if $\tilde{s}' := \alpha((s,a,s'),(\tilde{s},\tilde{a}))$, then $|T(\tilde{s}, \tilde{a}, \tilde{s}') - T(s,a,s')| \leq \tau/2$.
Since we assumed $T(s,a,s') = 0$, we have that $|T(\tilde{s},\tilde{a},\tilde{s}')| \leq \tau/2$. However, by Assumption~\ref{as:negligible_transitions}, we have that the range of $T$ lies in $[0] \cup [\tau,1]$, therefore, to satisfy the above inequality, we must have that $T(\tilde{s},\tilde{a},\tilde{s}') = 0$. This would imply that the empirical probability $\hat{T}(\tilde{s},\tilde{a},\tilde{s}')$ too is zero, which then is assigned to $\delT(s,a,s')$ in Algorithm~\ref{alg:tighter_ci}. Thus,  $\delT(s,a,s')=0$.

Now consider the case when $T(s,a,s') > 0$. We will show that $\delT(s,a,s') > 0$.
First, since $T(s,a,s') > 0$, by Assumption~\ref{as:negligible_transitions}, it means that $T(s,a,s') \geq \tau$. As argued in the previous case, since $|T(\tilde{s}, \tilde{a}, \tilde{s}') - T(s,a,s')| \leq \tau/2$, we will also have $T(\tilde{s}, \tilde{a}, \tilde{s}') > \tau/2$. Again, by Assumption~\ref{as:negligible_transitions}, this would imply $T(\tilde{s}, \tilde{a}, \tilde{s}') \geq \tau$. Note that, from Algorithm~\ref{alg:tighter_ci}, we have that, since $\del{\epsilon}_T(s,a) < \tau/2$, $\hat{\epsilon}_T(\tilde{s},\tilde{a}) < \tau/2$. And since confidence intervals are admissible, this means that $|T(\tilde{s}, \tilde{a}, \tilde{s}') - \hat{T}(\tilde{s}, \tilde{a}, \tilde{s}')| \leq \tau/2$. This, together with the fact that $T(\tilde{s}, \tilde{a}, \tilde{s}') > \tau/2$, means that $\hat{T}(\tilde{s}, \tilde{a}, \tilde{s}') > \tau/2$. In Algorithm~\ref{alg:tighter_ci}, we would assign this value to $\del{T}(\tilde{s}, \tilde{a}, \tilde{s}')$, thus, resulting in $\delT(s,a,s') > 0$. 
\end{proof}

This lemma allows us to compute the support of state-action pairs we have never experienced.
Algorithm~\ref{alg:compute_safe_set} uses this idea to expand $\hatZsafe$ in a way that retains closedness, thus, safety.

\begin{lemma} \label{lem:hatZsafe_safe}
Assume our confidence intervals are admissible. Whenever Algorithm~\ref{alg:ase} calls Algorithm \ref{alg:compute_safe_set}, in the final step of Algorithm \ref{alg:compute_safe_set}, $\hatZsafe \cup \Zclosed$ is a safe set.
\end{lemma}

\begin{proof}

We will prove this statement using induction. That is we assume that, before every call to Algorithm \ref{alg:compute_safe_set}, $\hatZsafe$ is a safe set. As the base case, this is satisfied in the first call because then, $\hatZsafe = Z_0$ and we have assumed $Z_0$ to be a safe set in Assumption~\ref{as:z_0}.

Since we populate $Z_{\text{candidate}}$ with only those state-action pairs with non-negative rewards, it follows directly from the run of the algorithm that all state-action pairs that are eventually found in $\Zclosed$ have non-negative rewards. Thus, to show that $\Zclosed \cup \hatZsafe$ is safe, we only need to show that $\Zclosed \cup \hatZsafe$ is a closed set.  That is, we need to show that for any $(s,a) \in \Zclosed \cup \hatZsafe$, 
every possible next state has an action in $ \hatZsafe \cup \Zclosed$ i.e., for all $s' \in \{s' \in S : T(s, a, s') > 0 \}$, there exists an $a' \in A$ where $(s', a') \in \hatZsafe \cup \Zclosed$. From the induction assumption (that $\hatZsafe$ is closed), this trivially holds for all $(s,a) \in \hatZsafe$.

Hence, consider any $(s, a) \in \Zclosed$. From our choice of $Z_{\text{candidate}}$ in the first step of the algorithm, we know that $\del{\epsilon}_T(s, a) < \tau/2$. Then from Lemma~\ref{lem:del-support}, we know that the set $\del{S}' = \{s' \in S : \del{T}(s, a, s') > 0 \}$ is identical to the true support of the next state-action pairs of $(s,a)$. 
But from the third inner while loop of our algorithm, we have that for all $s' \in \del{S}'$, we ensure that there exists an $a' \in A$ where $(s', a') \in \Zclosed \cup \hatZsafe$, implying that all possible next states of $(s,a)$ have a corresponding action in $\Zclosed \cup \hatZsafe$. 
Thus, $\Zclosed \cup \hatZsafe$ is closed, and also, safe.
\end{proof}

As a corollary of the above result, we can also show that $\hatZsafe$ is closed even if we replaced the true transitions by some candidate transition.

\begin{corollary}
\label{cor:closed_under_candidate}
Assume the confidence intervals are admissible. During the run of Algorithm~\ref{alg:ase}, we always have that, for any $T^\dagger \in CI(\del{T})$ and for any $(s,a) \in \hatZsafe$, if there exists $s' \in S$ such that $T^\dagger(s,a,s') > 0$, then $s' \in \hatZsafe$.
\end{corollary}

\begin{proof}
By design of Algorithm~\ref{alg:compute_safe_set}, we know that for any $(s,a) \in \hatZsafe$, either $(s,a) \in Z_0$ or 
$\del{\epsilon}(s,a) < \tau/2$. 

Consider the case where $(s,a) \in Z_0$. Since the confidence intervals are admissible, by the third requirement in the definition of the candidate transition set (Definition~\ref{def:ci}), we have that if $T^\dagger(s,a,s') > 0$, then $s' \in Z_0$.
Since $Z_0 \subset \hatZsafe$, $s' \in \hatZsafe$.

Consider the case where $\del{\epsilon}(s,a) < \tau/2$.
Since the confidence intervals are admissible, we have from Lemma~\ref{lem:del-support} that if $T^\dagger(s,a,s') > 0$, then $T(s,a,s') > 0$. Since we have established in Lemma~\ref{lem:hatZsafe_safe} that $\hatZsafe$ is closed, this means that $s' \in \hatZsafe$.
\end{proof}

In order to make sure that our agent can continue exploring without ever getting stuck, we must ensure that whenever Algorithm~\ref{alg:compute_safe_set} expands $\hatZsafe$, it remains communicating.

\begin{lemma} \label{lem:hatZsafe_communicating}
Assume our confidence intervals are admissible. Whenever Algorithm~\ref{alg:ase} calls Algorithm \ref{alg:compute_safe_set}, in the final step of Algorithm \ref{alg:compute_safe_set}, $\hatZsafe \cup \Zclosed$ is communicating.
\end{lemma}

\begin{proof}
We will prove this statement using induction. We first assume that, before every call to Algorithm \ref{alg:compute_safe_set}, $\hatZsafe$ is an communicating set, and using this, prove that the updated safe set, namely $\hatZsafe \cup \Zclosed$ computed at the end of Algorithm \ref{alg:compute_safe_set}, is also communicating. As the base case, this is satisfied in the first call because $\hatZsafe = Z_0$ and we have assumed $Z_0$ to be an communicating set in Assumption~\ref{as:z_0}.

Informally, to show that $\hatZsafe$ is communicating, we will first show that for every state in $\hatZsafe$, the agent has a return policy which ensures that from anywhere in $\hatZsafe \cup \Zclosed$, it can reach that state with non-zero probability. As a second step, we will show that for every state in $\Zclosed$, the agent has a `reach' policy which ensures that from anywhere in $\hatZsafe \cup \Zclosed$
it can reach that state with non-zero probability. Finally, we will put these together to establish communicatingness.


\subparagraph{Proof that $\Zclosed$ is returnable.} 
As the first part of the proof, we will show that, 
\begin{equation}
 \forall \tilde{s} \in \hatZsafe \ \exists \pireturn \in \Pi(\hatZsafe \cup \Zclosed), \  \ s.t. \ \forall s \in \hatZsafe \cup \Zclosed \ \Pr\left[\left. {\exists t, \; s_t=\tilde{s}}\right| \pireturn, s_0=s\right] > 0.
\label{eq:returnable}
\end{equation}


Fix an $\tilde{s} \in \hatZsafe$. 
We will prove the existence of a suitable $\pireturn$ by induction. 
Consider the last outer iteration of our algorithm during which we know that $\Zreturnable$ at the end of the second inner while block is identical to $\Zclosed$ at the end of the third inner while block. In this round, consider some $(s,a)$ that is about to be added to $\Zreturnable$. For the induction hypothesis, we will consider a hypothesis that is stronger than the one above. 
In particular, assume that there exists $\pireturn$ such that for every initial state $s'$ currently in $\hatZsafe \cup \Zreturnable$, there is non-zero probability of returning to $\tilde{s}$, {\em while visiting only those states in $\hatZsafe \cup \Zreturnable$ on the way to $\tilde{s}$.} Formally, assume that $\exists \pireturn \in \Pi(\hatZsafe \cup \Zclosed)$ such that:
\begin{equation}
  \forall s' \in \hatZsafe \cup \Zreturnable, \ \Pr\left[\left. {\exists t, \; \substack{s_t=\tilde{s}, \\ \forall t' < t \ s_{t'} \in \hatZsafe \cup \Zreturnable}}\right| \pireturn, s_0=s'\right] > 0.
\end{equation}


This assumption is of course true initially when $\hatZsafe \cup \Zreturnable = \hatZsafe$, by communicatingness of $\hatZsafe$.
 
Now, by the manner in which the second while block works, we know that there exists an $s'$ such that $\del{T}(s,a,s') > 0$ and there exists $a'$ such that $(s',a') \in \hatZsafe \cup \Zreturnable$. Note that since $\del{\epsilon}_T(s,a) < \tau/2$ (by our initial choice of $\Zcandidate$) and
since $\del{T}(s,a,s') > 0$, it follows from Lemma~\ref{lem:del-support} that ${T}(s,a,s') > 0$. 



Next, consider $\pireturn'$ that is identical to $\pi$ everywhere, except $\pireturn'(s) = a$.  Note that since $(s,a) \in \hatZsafe \cup \Zclosed$, $\pireturn' \in \Pi(\hatZsafe \cup \Zclosed)$.

First, we have that the induction assumption still holds for every $s' \in \hatZsafe \cup \Zreturnable$. That is, for every $s' \in \hatZsafe \cup \Zreturnable$, with non-zero probability, $\pireturn'$ starts from $s'$ to return to $\tilde{s}$, without visiting any state outside $\hatZsafe \cup \Zreturnable$. This is because, for a given random seed, if the agent were to follow $\pireturn$ from $s$ to return to $s''$, without visiting any states outside of $\hatZsafe \cup \Zreturnable$, the agent would do the same under $\pireturn'$ since the two policies would agree on all the visited states. 

Next, we have the following when starting from $s$:
\begin{align*}
\Pr\left[\left. {\exists t, \; \substack{s_t=\tilde{s} \\ \forall t' < t \ s_{t'} \in \hatZsafe \cup \Zreturnable}}\right| \pireturn', s_0 = s\right]& \geq T(s,a,s ') \Pr\left[\left. {\exists t, \; \substack{s_t=\tilde{s} \\ \forall t' < t \ s_{t'} \in \hatZsafe \cup \Zreturnable }}\right| \pireturn', s_0 = s'\right] \\
& \geq T(s,a,s') \Pr\left[\left. {\exists t, \; \substack{s_t=\tilde{s} \\ \forall t' < t \ s_{t'} \in \hatZsafe \cup \Zreturnable}}\right| \pireturn, s_0 = s'\right] \\
& > 0. 
\end{align*}

Here, the first inequality simply follows from the fact that one possible way to reach $\tilde{s}$ from $s$, is by first taking a step to $s'$. In the second step, we were able to replace $\pireturn'$ with $\pireturn$ by a similar logic as before. Specifically, for a given random seed, if an agent starts from
$s'$ to reach $\tilde{s}$ following $\pireturn$ without ever visiting $s$, it should do the same under $\pireturn'$ too. 

Finally, we have that both the above terms are strictly positive. Hence, we establish that $\pireturn'$, with non-zero probability, allows the agent to return to $\tilde{s}$, while never visiting any state outside $\hatZsafe \cup \Zreturnable \cup \{(s,a)\}$.

\subparagraph{Proof for $\Zclosed$ is reachable.} 

For the second part of the proof, we will show that $\forall s \in \Zclosed$,
\begin{equation}
\exists \pireach \in \Pi(\Zclosed \cup \hatZsafe),  \ s.t. \forall \tilde{s} \in \hatZsafe \cup \Zclosed  \  \Pr[\exists t, s_t = s | s_0 = \tilde{s} ,\pireach] > 0.
\label{eq:reachable}
\end{equation}



We will prove this by induction. Consider the last outer iteration of the algorithm and consider the first inner while loop where we populate $\Zreachable$. Note that since this is the last outer iteration, at the end of this while block, $\Zreachable$ is exactly equal to $\Zclosed$ that is output at the end of the algorithm. (Thus, during the run of this while loop, we always have that $\Zreachable \subset \Zclosed$.) 
In this while loop, consider the instant at which  some $(s',a')$ is about to be added to $\Zreachable$.
We will assume by induction that for all $\tilde{s}$ that are currently in $\hatZsafe \cup \Zreachable$,
Equation~\ref{eq:reachable} holds i.e., $\exists \pireach$ such that from anywhere in $\hatZsafe \cup \Zreachable$, we can reach $s$ with non-zero probability.  

As the base case, because of how we have initialized $\Zreachable$, we have that for all $s \in \hatZsafe \cup \Zreachable$,  $s \in \hatZsafe$. Thus, for this case, the induction assumption holds from the fact that we have proven returnability to $\hatZsafe$. 

Now, let us turn to the point when the algorithm is about to add $(s',a')$ to $\Zreachable$. Then, note that this means that there exists $(s,a) \in \hatZsafe \cup \Zreachable$ such that $\del{T}(s,a,s') > 0$. Since $\del{\epsilon}_T(s,a) < \tau/2$ (by our initial choice of $\Zcandidate$) and since $\del{T}(s,a,s') > 0$, by Lemma~\ref{lem:del-support}, we have that $T(s,a,s') > 0$. 

Next, consider the policy $\pireach$ guaranteed by our induction assumption, to reach $s$ from anywhere in $\hatZsafe \cup \Zclosed$.  Then, define a policy $\pireach'$ which is identical to $\pireach$ on all states, except that $\pireach'(s) = a$. Note that since $(s,a) \in \hatZsafe \cup \Zreachable \subset \hatZsafe \cup \Zclosed$, and since $\pireach \in \Pi(\hatZsafe \cup \Zclosed)$, we have $\pi'_{\text{reach}} \in \Pi(\hatZsafe \cup \Zclosed)$.

Now, we can show that $\pi'_{\text{reach}}$ reaches $s'$ from any $\tilde{s} \in \hatZsafe \cup \Zclosed$  because of the following:
\begin{align*}
\Pr\left[\left. \exists t, s_t=s'\right| \pireach', s_0 = \tilde{s}\right] &\geq T(s,a,s')Pr\left[\left. \exists t, s_{t-1}=s \right| \pireach', s_0=\tilde{s}\right] \\
&\geq T(s,a,s')Pr\left[\left. \exists t, s_{t-1}=s \right| \pireach, s_0=\tilde{s}\right] \\
&>  0.
\end{align*}

Here, the first inequality comes from the fact that one way to reach $s'$ is by traveling to $s$ and then taking the action $a$. In the next inequality, we make use of the fact that, for a given random seed, if the agent follows $\pireach$ from $\tilde{s}$ to visit $s$ for the first time, it would follow the same steps to reach $s$ even under $\pireach'$, since the policies would agree until then. Finally, we have from our induction assumption that the probability term in the penultimate line is strictly positive; since the transition probability is strictly positive too, the last inequality holds.

This proves that $\pireach' \in \Pi(\hatZsafe \cup \Zclosed)$ reaches $s'$ from anywhere in $\hatZsafe \cup \Zclosed$ with non-zero probability. 


\subparagraph{Proof for communicatingness.} Finally, we will wrap the above results to establish communicatingness. From the above results, we have that $\forall s' \in \hatZsafe \cup \Zclosed$:
\[
   \exists \pivisit \in \Pi(\hatZsafe \cup \Zclosed), \ s.t. \forall s \in \hatZsafe \cup \Zclosed \ \ \Pr\left[\left. {\exists t, \; s_t=s'}\right| \pivisit, s_0=s\right] > 0.
\]

To establish communicatingness, we need to show that this probability is in fact $1$.
To do this, we will first note that $\pivisit \in \Pi(\hatZsafe \cup \Zclosed)$ and since $\hatZsafe \cup \Zclosed$ is closed (as proven in Lemma~\ref{lem:hatZsafe_safe}), then use Lemma~\ref{lem:communicating_positive_prob} to establish communicatingness.

\end{proof}

\begin{corollary}
\label{cor:hatZsafe_subset}
Assume our confidence intervals are admissible.
 Whenever Algorithm~\ref{alg:ase} calls Algorithm \ref{alg:compute_safe_set}, in the final step of Algorithm \ref{alg:compute_safe_set}, $\hatZsafe \cup \Zclosed \subset \Zsafe$.
\end{corollary}

\begin{proof}
Note that $Z_0 \subset \hatZsafe \cup \Zclosed$ by construction. Then,  since we have established communicatingness of $\hatZsafe \cup \Zclosed$ in Lemma~\ref{lem:hatZsafe_communicating}, for every $(s,a) \in \hatZsafe \cup \Zclosed$, there should exist a policy $\pireturn$ that travels from $(s,a)$ and goes to any states in $Z_0$ with probability $1$. Furthermore since we established safety in Lemma~\ref{lem:hatZsafe_safe}, this also means that every state-action pair visited in this path has non-negative reward. Thus, by definition of $\Zsafe$, the claim follows. 
\end{proof}

Although our definition of communicating, Def~\ref{def:communicating}, seems strict since its guarantee must hold with probability $1$, we show here that this is no stronger than having the guarantee simply hold with positive probability.
This lemma helps us prove that our definition of communicating is equivalent to that of the standard definition (see Fact~\ref{fact:communicating_single_policy}) as well as help prove that the safe set we construct is indeed communicating (see Lemma~\ref{lem:hatZsafe_communicating}).

\begin{lemma} \label{lem:communicating_positive_prob}
If there exists a closed set of state-action pairs $Z$ such that $\forall s' \in Z$:
\[
   \exists \ \pivisit \in \Pi(Z), \ s.t. \forall s \in Z \ \ \Pr\left[\left. {\exists t, \; s_t=s'}\right| \pivisit, s_0=s\right] > 0
\]
then it must be the case that this probability is, in fact, equal to $1$; specifically, $\forall s' \in Z$:
\[
   \exists \ \pivisit \in \Pi(Z), \ s.t. \forall s \in Z \ \ \Pr\left[\left. {\exists t, \; s_t=s'}\right| \pivisit, s_0=s\right] = 1
\]
\end{lemma}
\begin{proof}
From the above, we have that for a given $s' \in Z$, there must exists a constant $H \geq 0$ and a constant $p > 0$ such that:
\begin{equation}
    \forall s \in Z \ \ \Pr\left[\left. {\exists t \leq H, \; s_t=s'}\right| \pivisit, s_0=s\right] \geq p.
\label{eq:p}
\end{equation}

We will start by equating the probability of never visiting $s'$ by decomposing the trajectory into a prefix of $H$ steps and the rest, and then applying the Markov property, as follows:
\begin{align*}
\Pr[\forall t, s_t \neq s' \ | \pivisit, s_0 = s] & = \sum_{s'' \in S} \Pr[\forall t, s_t \neq s' \ | \pivisit, s_0 = s''] \Pr\left[\left. \substack{\forall t \leq H, s_t \neq s' \\ s_H = s''} \ \right| \pivisit, s_0 = s\right]. 
\end{align*}
Note that since $\pivisit \in \Pi(Z)$ and since $Z$ is closed, we have from Fact~\ref{fact:closed_policy} that $s''$, which is the $H$th state in the trajectory, satisfies $s'' \in Z$.
Therefore, we can restrict the summation to $Z$ (the remaining terms would zero out).
Hence, 
\begin{align*}
& \Pr[\forall t, s_t \neq s' \ | \pivisit, s_0 = s] \\
& = \sum_{s'' \in Z} \Pr[\forall t, s_t \neq s' \ | \pivisit, s_0 = s''] \Pr\left[\left. \substack{\forall t \leq H, s_t \neq s' \\ s_H = s''} \ \right| \pivisit, s_0 = s\right] \\
& \leq  \left(\max_{s'' \in Z} \Pr[\forall t, s_t \neq s' \ | \pivisit, s_0 = s'']\right) \left( \sum_{s'' \in Z} \Pr\left[\left. \substack{\forall t \leq H, s_t \neq s' \\ s_H = s''} \ \right| \pivisit, s_0 = s\right]   \right) \\
& =  \left(\max_{s'' \in Z} \Pr[\forall t, s_t \neq s' \ | \pivisit, s_0 = s'']\right) \left(  \Pr\left[\left. {\forall t \leq H, s_t \neq s' } \ \right| \pivisit, s_0 = s\right]   \right) \\
& \leq  \left(\max_{s'' \in Z} \Pr[\forall t, s_t \neq s' \ | \pivisit, s_0 = s'']\right) (1-p)
\end{align*}

Since the above inequality holds for any $s \in Z$, we can apply a $\max_{s \in Z}$ on the left hand side and rearrange to get:
\[
\max_{s'' \in Z} \Pr[\forall t, s_t \neq s' \ | \pivisit, s_0 = s''] \cdot p \leq 0 
\]

Since $p > 0$ (see Equation~\ref{eq:p}), this means that the first term here is equal to zero.
In other words, with probability $1$, $\exists t$ such that $s_t = s'$ when the agent starts from any $s$ and follows $\pivisit$, as claimed.
\end{proof}

Finally, we want to show that Algorithm~\ref{alg:compute_safe_set} computes the largest possible $\hatZsafe$ that still retains safety and communicatingness.
This is necessary since, if this were not true, we could get into a situation where we know enough to ensure we can perform the optimal policy, but our agent remains trapped in $\hatZsafe$ forever.

\begin{lemma}
   \label{lem:safe_set_large_enough} Assume that the confidence intervals are admissible. 
   Consider some call of  Algorithm \ref{alg:compute_safe_set} while executing Algorithm~\ref{alg:ase}.
   Consider $(\tilde{s},\tilde{a}) \notin \hatZsafe$ such that $\tilde{s} \in \hatZsafe$. Let $\exists \pireturn$ such that starting at $(\tilde{s},\tilde{a})$, $\pireturn$ reaches a state in $\hatZsafe$ with probability $1$. Let $\tilde{Z}$ be the set of state-action pairs visited by the agent starting $(\tilde{s},\tilde{a})$ following $\pireturn$ before reaching $\hatZsafe$. Formally, let:
   \[
   \tilde{Z} = \{(s,a) \in S \times A : \Pr\left[\left. \exists t \geq 0 \ \substack{(s_t, a_t) = (s,a) \\ \forall t' < t \ s_{t'} \notin \hatZsafe} \ \right| (s_0, a_0) = (\tilde{s}, \tilde{a}), \pireturn\right] > 0 \}.
    \]
    In the final step of Algorithm \ref{alg:compute_safe_set}, if $\forall (s,a) \in \tilde{Z} \setminus \hatZsafe$, $\del{\epsilon}_T(s,a)  < \tau/2$, then $\tilde{Z} \subset \hatZsafe \cup \Zclosed$.
\end{lemma}

\begin{proof}
Informally, we will show that in each iteration of the algorithm, after the first inner while block, we have $\tilde{Z} \subset \hatZsafe \cup \Zreachable$; and using this, we will show that after the second inner while block, we have $\tilde{Z} \subset \hatZsafe \cup \Zreturnable$; and finally, because of this, after the third inner while block, $\tilde{Z} \subset \hatZsafe \cup \Zclosed$. Below, we prove these three statements, and finally wrap them up to prove the main claim.

\subparagraph{Proof for $\tilde{Z} \subset \hatZsafe \cup \Zreachable$.} 
Assume there exists $(\tilde{s},\tilde{a})$ as specified in the lemma statement. Then, formally, we will show that if $\tilde{Z} \subset \hatZsafe \cup \Zcandidate$ before the beginning of the first inner while loop, then $\tilde{Z} \subset \hatZsafe \cup \Zreachable$ at the end of the loop. 

We will prove this by contradiction. Assume for the sake of contradiction that there exists a non-empty $\tildeZbad \subset \tilde{Z}$ such that $\tildeZbad$ has no intersection with $\hatZsafe \cup \Zreachable$.
Let $\tildeZgood = \tilde{Z} \setminus \tildeZbad$.
First we note why $\tildeZgood$ is non-empty.
Since $\tilde{Z} \setminus \hatZsafe \subset \Zcandidate$ (from our initial assumption),  since $(\tilde{s}, \tilde{a}) \in \tilde{Z}\setminus \hatZsafe$,  and since $\tilde{s} \in \hatZsafe$ (from lemma statement), 
when initializing $\Zreachable$ with  $\{(s,a) \in \Zcandidate: s \in \hatZsafe\}$,  we would add $(\tilde{s}, \tilde{a})$ to $\Zreachable$.  Thus, $\tildeZgood$ must contain at least $(\tilde{s},\tilde{a})$.

Next, we argue that there must exist some $(s',a') \in \tilde{Z}_{\text{bad}}$ such that there exists an $(s,a) \in \tilde{Z}_{\text{good}}$ for which $T(s,a,s') > 0$ and $s \notin \hatZsafe$. If not, then starting from $(\tilde{s},\tilde{a}) \in \tilde{Z}_{\text{good}}$, the agent can never hope to reach $\tildeZbad$ before visiting a state $\hatZsafe$, contradicting the fact that $\tildeZbad$ is a subset of state-action pairs that it visits before reaching $\hatZsafe$.

Now, consider such an $(s',a')$ (which was not added to $\Zreachable$) and its predecessor $(s,a)$, which belongs to  $\hatZsafe \cup \Zreachable$ because it belongs to $\tildeZgood$. In some iteration of this while loop,
we must have examined $(s,a) \in \hatZsafe \cup \Zreachable$. Since $s \notin \hatZsafe$, we also have $(s,a) \notin \hatZsafe$. Since $(s,a) \in \tildeZgood$, and $\tildeZgood \subset \tilde{Z}$, this further means that $(s,a) \in \tilde{Z} \setminus \hatZsafe$. From our lemma statement, we then have that the confidence interval of $(s,a)$ is less than $\tau/2$. Hence $\del{\epsilon}_T(s,a) < \tau/2$. Then, from Lemma~\ref{lem:del-support}, since $T(s,a,s') > 0$, we have that $\del{T}(s,a,s') > 0$.
Then, since $(s',a') \in \Zcandidate$, we would have added $(s',a')$ to $\Zreachable$ in this iteration, contradicting the fact that 
we never added it to $\Zreachable$ in the first place. Thus, $\tildeZbad$ must be empty, implying that $\tilde{Z} \subset \hatZsafe \cup \Zreachable$.

\subparagraph{Proof for $\tilde{Z} \subset \hatZsafe \cup \Zreturnable$.} Formally, we will show that if $\tilde{Z} \subset \hatZsafe \cup \Zreachable$ before the beginning of the second while block, then $\tilde{Z} \subset \hatZsafe \cup \Zreturnable$ at the end of the block. 

At the end of the second block, let us define $\tildeZbad := \tilde{Z} \setminus (\hatZsafe \cup \Zreturnable)$. We want to show that $\tildeZbad$ is empty, but assume on the contrary it is not. First we argue that there must exist $(s,a) \in \tildeZbad$ such that one of its next states belongs to $\hatZsafe \cup \Zreturnable$. If this was not the case, then whenever the agent enters $\tildeZbad$, it will never be able to return to $\hatZsafe$.  This is however in contradiction to the definition of $\tilde{Z}$.

For the rest of the discussion, consider such an $(s,a) \in \tildeZbad$ such that $\exists (s',a') \in \hatZsafe \cup \Zreturnable$ for which $T(s,a,s') > 0$. Since $(s,a) \in \tildeZbad$ and $\Zbad \subset \tilde{Z} \setminus \hatZsafe$, it means that $(s,a) \in \tilde{Z} \setminus \hatZsafe$. Then,
from our lemma statement we have $\del{\epsilon}_T(s,a) \leq \tau/2$. 
Now, from Lemma~\ref{lem:del-support}, we have that since the confidence intervals are admissible and since $T(s,a,s') > 0$, it must be the case that $\del{T}(s,a,s') > 0$. Then, in the iteration of the while loop during which $(s',a')$ is present in $\hatZsafe \cup \Zreturnable$, $(s,a)$ would in fact be added to $\Zreturnable$. This, however, contradicts our assumption that $(s,a) \in \tildeZbad$. Therefore, $\tildeZbad$ should in fact be empty. Thus, $\tilde{Z} \subset \hatZsafe \cup \Zreturnable$ at the end of this while block.

\subparagraph{Proof for $\tilde{Z} \subset \hatZsafe \cup \Zclosed$.} Formally, we will show that if $\tilde{Z} \subset \hatZsafe \cup \Zreturnable$ before the beginning of the third while block, then $\tilde{Z} \subset \hatZsafe \cup \Zclosed$ at the end of the block. 

Assume on the contrary that there exists $(s,a) \in \tilde{Z}$  such that $(s,a) \notin \hatZsafe \cup \Zclosed$ at the end of the third block. Since $\Zclosed$ is initialized with all of $\tilde{Z}$ contained in it,  consider the \textit{first} such $(s,a)$ that is removed from $\Zclosed$ during the course of this second inner iteration. Now, just before the moment at which $(s,a)$ is removed, 
by design of the algorithm, we would have that there exists $s'$ such that $\del{T}(s,a,s') > 0$ and $s' \notin \hatZsafe \cup \Zclosed$. Note that at this point, we also still have $\tilde{Z} \subset \hatZsafe \cup \Zclosed$. Therefore, this means that $s' \notin Z$.
Now, we have that $(s,a) \in \tilde{Z} \setminus \hatZsafe$, which means, by the lemma statement, $\del{\epsilon}_T(s,a) \leq \tau/2$. Then,
since $\del{T}(s,a,s') > 0$, from Lemma~\ref{lem:del-support}, we have $T(s,a,s') > 0$. However, this contradicts the fact, if $(s,a) \in \tilde{Z}$, then every next state of $(s,a)$ must lie in $\tilde{Z}$. Thus, no $(s,a)$ belonging to $\tilde{Z}$ must have been removed from $\hatZsafe \cup \Zclosed$ during this while block, implying that $\tilde{Z} \subset \hatZsafe \cup \Zclosed$ at the end of this block.

\subparagraph{Proof of main claim.} From the above arguments, we have that whenever the outer iteration begins with $\tilde{Z} \subset \Zcandidate \cup \hatZsafe$ it ends with $\tilde{Z} \subset \Zcandidate \cup \Zclosed$.  
Now,  at the beginning of Algorithm~\ref{alg:compute_safe_set}, we must have $\tilde{Z} \setminus \hatZsafe \subset \Zcandidate$ due to the fact that all elements of $\tilde{Z}$ have confidence intervals at most $\tau/2$. In other words, $\tilde{Z} \subset \Zcandidate \cup \hatZsafe$. 
Then, from the above arguments, we have that $\tilde{Z} \subset \hatZsafe \cup \Zclosed$ at the end of the first iteration. Since $\Zcandidate$ at the beginning of the second outer iteration is equal to $\Zclosed$ from the end of the previous outer iteration, we again have $\tilde{Z} \subset \Zcandidate \cup \hatZsafe$ at the beginning of the second iteration. Thus, by a similar argument we have that the algorithm preserves the condition that $\tilde{Z} \subset \Zclosed \cup \hatZsafe$ at the end of every outer iteration, proving our main claim.

\end{proof}
\subsection{Proofs about Algorithm~\ref{alg:z_explore} and Algorithm~\ref{alg:z_goal}}

In the next few lemmas, we prove results about Algorithm~\ref{alg:z_explore} and Algorithm~\ref{alg:z_goal}. Recall that Algorithm~\ref{alg:z_explore} takes as input a set of edge state-action pairs (which are state-action pairs that do not belong to $\hatZsafe$ but whose state belongs to $\hatZsafe$) and outputs a set of elements from $\hatZsafe$ which need to be explored in order to learn the return paths from the edges. Also recall that the idea of Algorithm~\ref{alg:z_goal} is to return an updated $\optZgoal$ (the set of states in the optimistic goal path) and $\Zexplore$ (by making a call to Algorithm~\ref{alg:z_explore}).

In Lemma~\ref{lem:z_explore_terminates} we demonstrate that Algorithm~\ref{alg:z_explore} terminates in polynomial time.

Lemma~\ref{lem:z_explore_K} helps establish that running Algorithm~\ref{alg:z_explore} allows the agent progress without getting stuck. More concretely, in Lemma~\ref{lem:z_explore_K} we argue that the elements of $\Zexplore$ do not belong to $K$. That is, the elements of $\Zexplore$ are those that have not already been explored (as otherwise, the agent may be stuck perpetually in exploring what has already been explored).

In Lemma~\ref{lem:correctness_z_explore}, we establish that Algorithm~\ref{alg:z_explore} works correctly (and hence, so does
Algorithm~\ref{alg:z_goal}). In particular, we show that
Algorithm~\ref{alg:z_explore} does not terminate with an empty $\Zexplore$ when some of the edge state-action pairs of $\hatZsafe$ are indeed safe. If we did not have this guarantee, then it's possible that even though there are some edge state-action pairs are safe, our agent may be stuck without exploring any state-action pair within $\hatZsafe$. On the other hand, with this guarantee, we can be confident that our agent will explore to learn the return path of such edge states, and then establish their safety, using which it can then expand $\hatZsafe$ in the future.  As a corollary of this guarantee, we show that Algorithm~\ref{alg:z_goal} always ensures that $\hatZunsafe$ contains no element that actually belongs to $\Zsafe$.

Finally, in Lemma~\ref{lem:correctness_z_goal}, we prove that Algorithm~\ref{alg:z_goal} terminates in polynomial time, guaranteeing that either $\Zexplore$ is non-empty (which means the agent can explore $\hatZsafe$ to learn a return path and expand $\hatZsafe$) or that $\optZgoal$ has been updated in a way that $\optZgoal \subset \hatZsafe$ (which means that the agent can stop exploring, and instead, start exploiting).



Below, we prove our result about the run-time complexity of Algorithm \ref{alg:z_explore}.

\begin{lemma} \label{lem:z_explore_terminates}
 Algorithm \ref{alg:z_explore} terminates in $\text{poly}(|S|, |A|)$ time.
\end{lemma}

\begin{proof}
Recall that Algorithm~\ref{alg:z_explore} executes an iteration of a while loop whenever $Z^L_{\text{next}} \neq \emptyset$ and $\Zexplore = \emptyset$. Also recall that inside the while loop, the algorithm executes a for loop that iterates over all elements of $Z^L_{\text{next}}$.
Hence, during the while loop, since $Z^L_{\text{next}}$ is non-empty, we must also execute at least one iteration of the inner for loop. Now, note that, by design of the algorithm, $\Znext^L$ is populated only with elements that do not belong to $\Zreturn$. Since the for loop adds all these elements to $\Zreturn$, every call to the for loop corresponds to increasing the cardinality of $\Zreturn$ by at least one. Thus, there can be at most as many executions of the for loop as there are state-action pairs, $|S| \times |A|$. By extension, the while loop can be executed at most $|S| \times |A|$ times, after which it should terminate.
\end{proof}

Next, we show that the state-action pairs marked for exploration by Algorithm~\ref{alg:z_explore} have not already been explored well before.

\begin{lemma} \label{lem:z_explore_K}
 Algorithm \ref{alg:z_explore} returns $\Zexplore$ such that for every $(\tilde{s},\tilde{a}) \in \Zexplore$, $(\tilde{s}, \tilde{a}) \in K^c$ where $K = \{(s, a) \in S \times A : n(s, a) \geq m \}$ when $m \geq O\left( \frac{|S|}{\tau^2} + \frac{1}{\tau^2} \ln \frac{|S||A|}{\tau^2 \delta}\right)$.
\end{lemma}

\begin{proof}
  Assume on the contrary that there exists $(\tilde{s},\tilde{a}) \in \Zexplore$ such that $(\tilde{s}, \tilde{a}) \in K$. By design of  Algorithm \ref{alg:z_explore}, we have that there exists $(s,a) \notin \hatZsafe$ which resulted in the addition of $(\tilde{s}, \tilde{a})$ to $\Zexplore$. In particular, we would have that $\del{\epsilon}_T(s,a) > \tau/2$ and $\Delta((s,a), (\tilde{s}, \tilde{a})) < \tau/2$. Since $(\tilde{s}, \tilde{a}) \in K$, we also have $\hat{\epsilon}_T(\tilde{s},\tilde{a}) \leq \tau/4$ (and this follows from Lemma~\ref{lem:size_of_m}). Then, we would have $\del{\epsilon}_T(s,a) \leq \hat{\epsilon}_T(\tilde{s},\tilde{a})+ \Delta((s,a), (\tilde{s}, \tilde{a}))$, implying $\del{\epsilon}_T(s,a) \leq \tau/4$ which is a contradiction. Thus, the above claim is correct.
\end{proof}

Below, we show that as long as there is a edge to $\hatZsafe$ with a safe return path to $\hatZsafe$, 
Algorithm \ref{alg:z_explore} will return with a non-empty $\Zexplore$.

\begin{lemma} \label{lem:correctness_z_explore}
Assume our confidence intervals are admissible. During a run of Algorithm~\ref{alg:ase}, at the end of every call to Algorithm \ref{alg:z_explore}, 
we will have $\Zexplore = \emptyset$ only if $\forall \ (\tilde{s}, \tilde{a}) \in Z_{\text{edge}}$, $(\tilde{s},\tilde{a}) \notin \Zsafe$. As a corollary of this, we always have that $\hatZunsafe \cap \Zsafe = \emptyset$.
\end{lemma}

\begin{proof}
We will prove this by induction. Specifically, since Algorithm~\ref{alg:z_goal} is the only block where we call Algorithm \ref{alg:z_explore} and modify $\hatZunsafe$,
we will consider a particular iteration of the while loop in Algorithm~\ref{alg:z_goal}. Then, we will assume that in all the previous iterations of this while loop, when 
Algorithm \ref{alg:z_explore} was called, it satisfied the above guarantee. Additionally, we will assume that $\hatZunsafe$ satisfies $\hatZunsafe \cap \Zsafe = \emptyset$ in the beginning of this loop. Then, we will show that the guarantee about Algorithm \ref{alg:z_explore} is satisfied even when we call it in this loop, and that, by the end of this loop, $\hatZunsafe$ continues to satisfy $\hatZunsafe \cap \Zsafe = \emptyset$.

As the base case, in the first iteration, since we have never called Algorithm~\ref{alg:z_explore} before, the induction hypothesis about Algorithm~\ref{alg:z_explore} is trivially satisfied. Furthermore, since $\hatZunsafe$ is initialized to be empty set, we again trivially have $\hatZunsafe \cap \Zsafe = \emptyset$ in the beginning of this loop. 

Now, consider any arbitrary iteration of the while loop of Algorithm~\ref{alg:z_goal}. Through the next few paragraphs below, we will argue why the call to Algorithm \ref{alg:z_explore} in this loop, satisfies the above specified guarantee. In the final paragraph, we provide a simple argument showing why $\hatZunsafe \cap \Zsafe = \emptyset$  at the end of the loop.

\subparagraph{Proof for claim about Algorithm \ref{alg:z_explore}.}
Assume that during this particular call to Algorithm \ref{alg:z_explore}, there exists $(\tilde{s}, \tilde{a}) \in Z_{\text{edge}}$ such that $(\tilde{s}, \tilde{a}) \in \Zsafe$. 
To prove the above claim, we only need to show that in this case Algorithm \ref{alg:z_explore} will result in a non-empty $\Zexplore$. So, for the sake of contradiction, we will assume that $\Zexplore$ is empty after the execution of Algorithm \ref{alg:z_explore}. 

The outline of our idea is to make use of the fact that by Assumption~\ref{as:sim} we are guaranteed a return path from $(\tilde{s},\tilde{a})$ that is sufficiently similar to state-action pairs in $\hatZsafe$. We will then show that we can pick a particular trajectory from this return path which visits a `bad' state-action pair -- a state-action pair whose counterpart in $\hatZsafe$ has not been explored sufficiently. Then, under the assumption that $\Zexplore$ remains empty, we will argue that Algorithm \ref{alg:z_explore} will visit all the state-action pairs in this trajectory, and finally, when it encounters the bad state-action pair, the algorithm will add the counterpart of this bad pair to $\Zexplore$, reaching a contradiction.


First, let us apply Assumption~\ref{as:sim} to $(\tilde{s}, \tilde{a})$ which guarantees a return path from $(\tilde{s},\tilde{a})$ because it is an edge state-action i.e., $\tilde{s} \in \hatZsafe$ and $(\tilde{s}, \tilde{a}) \notin \hatZsafe$. But to apply this assumption, we must establish that $Z_0 \subset \hatZsafe$. This is indeed true as it follows from how Algorithm~\ref{alg:ase} initializes $\hatZsafe$ with $Z_0$ and every call to Algorithm~\ref{alg:compute_safe_set} only adds elements to $\hatZsafe$. 

Now, consider the $\pireturn$ guaranteed by Assumption \ref{as:sim}. Starting from $(s,a)$, 
and following $\pireturn$, the agent returns to $Z_0$ with probability $1$. Furthermore, if we define the set of state-action pairs visited by $\pireturn$ on its way to $\hatZsafe$ as:
\[\tilde{Z} = \{(s,a) \in S \times A : \Pr\left[\left. \exists t \geq 0 \ \substack{(s_t, a_t) = (s,a) \\ \forall t' < t \ s_{t'} \notin \hatZsafe} \ \right| (s_0, a_0) = (\tilde{s}, \tilde{a}), \pireturn\right] > 0 \},\] 
then we are given that every element of $\tilde{Z} \setminus \hatZsafe$ corresponds to an element $(s',a') \in \hatZsafe$ such that $\Delta((s,a), (s',a')) \leq \tau/4$.

To make our discussion easier, let us partition the elements of $\tilde{Z} \setminus \hatZsafe$ into two sets $\tildeZgood$ and $\tildeZbad$ as follows, depending on whether or not the corresponding $(s',a')$ has been explored sufficiently well or not:
\[
\tildeZgood = \{ (s,a) \in \tilde{Z} \setminus \hatZsafe \ | \ \exists (s',a') \in S \times A \text{ s.t. } \Delta((s,a), (s',a')) + \hat{\epsilon}_T(s',a') \leq \tau/2 \}
\]
and
\[
\tildeZbad = \{ (s,a) \in \tilde{Z} \setminus \hatZsafe \ | \ \forall (s',a') \in S \times A \text{ s.t. } \Delta((s,a), (s',a')) + \hat{\epsilon}_T(s',a') > \tau/2 \}.
\]
Note that every element of $\tilde{Z} \setminus \hatZsafe$ belongs to exactly one of $\tildeZgood$ and $\tildeZbad$.  Also note that for all $(s,a) \in \tildeZgood$, $\del{\epsilon}_T(s,a) \leq \tau/2$ since $\Delta((s,a), (s',a')) + \hat{\epsilon}_T(s',a') \leq \tau/2$. However, for all $(s,a) \in \tildeZbad$,  since there exists no sufficiently explored $(s',a')$ that is also sufficiently smaller, $\del{\epsilon}_T(s,a) > \tau/2$. 

Next, we argue that there must exist at least one element in $\tildeZbad$. If this was not the case, then we would have that all elements in $\tilde{Z} \setminus \hatZsafe$ belong to $\tildeZgood$ and therefore have a confidence interval at most $\tau/2$. Then, from Lemma~\ref{lem:safe_set_large_enough},  we will have that when Algorithm~\ref{alg:ase} executed Algorithm~\ref{alg:compute_safe_set} just before calling Algorithm~\ref{alg:z_goal}, $\hatZsafe$ is updated in a way that $\tilde{Z} \subset \hatZsafe$, which would imply that $(\tilde{s},\tilde{a}) \in \hatZsafe$. However, this contradicts our assumption in the beginning that that $(\tilde{s},\tilde{a})$ is an edge state-action pair that does not belong to $\hatZsafe$.

Since $\tildeZbad$ is non-empty, and since $\tildeZbad \subset \tilde{Z}$, there should exist a trajectory of $\pireturn$ starting from $(\tilde{s}, \tilde{a})$ that passes through an element of $\tildeZbad$ before visiting $\hatZsafe$. Let $(s_0,a_0), (s_1, a_1), \hdots, (s_n,a_n)$ be one such trajectory, where $(s_0, a_0) = (\tilde{s}, \tilde{a})$. Let $(s_n, a_n)$ be the first element in this trajectory that belongs to $\tildeZbad$. Since all the elements in this trajectory preceding $(s_n, a_n)$ belong to $\tilde{Z}$ but not $\hatZsafe$ or $\tildeZbad$, all these elements must belong to $\tildeZgood$. 

We will now use the fact that, under our initial assumption, Algorithm~\ref{alg:z_explore} returns with an empty $\Zexplore$ to argue by induction that, during that run of Algorithm~\ref{alg:z_explore}, all state-action pairs in this trajectory up until and including $(s_n, a_n)$ are added to $\Zreturn$.
(After this, we will reach a contradiction).

For the base case,  consider $(s,a)$. When the while loop condition is executed the first time, $\Zexplore = \emptyset$ by initialization, and $\Znext^L \neq \emptyset$ because it is equal to $Z_{\text{edge}}$ which has at least one state-action pair, namely $(s,a)$. Thus, the while loop will be executed, and every element in $\Znext^L$ will be added to $\Zreturn$. Since $\Znext^L = \Zedge$ at this point, this implies that $(s,a)$ will be added to $\Zreturn$.

Next, for some $i \in [1,n]$, assume by induction that all state-action pairs preceding $(s_k, a_k)$, where $k < i$, have been added to $\Zreturn$; we must prove the same happens to $(s_i, a_i)$. 
Consider the loop when $(s_{i-1}, a_{i-1})$ is examined and added to $\Zreturn$. 
Since $(s_{i-1}, a_{i-1}) \in \tildeZgood$, its confidence interval is at most as large as $\tau/2$; thus, in this loop, we would execute the else branch of the if-condition. As a result of this, we can argue that $(s_i, a_i)$ is added to $\Znext^{L+1}$ from the following four observations. 

First, since the considered trajectory has non-zero probability,  we have  $T(s_{i-1},a_{i-1}, s_i) > 0$. Furthermore, since $\del{\epsilon}_T(s_{i-1}, a_{i-1}) \leq \tau/2$, and since the confidence intervals are admissible, by Lemma~\ref{lem:del-support}, we have $\del{T}(s_{i-1},a_{i-1}, s_i) > 0$. Second, if $(s_i, a_i)$ was part of $\Zreturn$ at this point, we are already done; so let us consider the case that currently $(s_i, a_i) \notin \Zreturn$. Thirdly, since $(s_i, a_i) \in \tilde{Z} \setminus \hatZsafe$, $(s_i, a_i) \notin \hatZsafe$. Finally, from our induction assumption, we have that $\hatZunsafe \cap \Zsafe = \emptyset$; and since $\pireturn$ is a safe policy, it follows that $(s_i, a_i) \notin \hatZunsafe$.
As a result of these four observations, in this else branch, we would add $(s_i, a_i)$ to $\Znext^{L}$. 

Now, consider the instant when the algorithm evaluates the while condition after exiting the for loop that examined $(s_{i-1}, a_{i-1})$. At this point, by our initial assumption, $\Zexplore$ is still empty while $\Znext^L$ is not as it contains $(s_i, a_i)$.
Thus the algorithm would proceed with executing this while loop (as against exiting from it then). Now since $(s_i, a_i) \in \Znext^L$, inside the inner for loop, there must be an iteration when $(s_i, a_i)$ is examined and added to $\Zreturn$, proving our induction statement.

Thus, consider the for loop iteration when $(s_n, a_n)$ is added to $\Zreturn$. Since $(s_n, a_n) \in \tildeZbad$, $\del{\epsilon}_T(s_n, a_n) > \tau/2$. Hence, we would enter the if-branch of the if-else condition. Again, since  $(s_n, a_n) \in \tilde{Z} \setminus \hatZsafe$, from Assumption~\ref{as:sim}, we know there must exist $(s,a) \in \hatZsafe$ such that $\Delta((s,a), (s_n,a_n)) \leq \tau/4$.
As a result, $(s,a)$ will be added to $\Zexplore$ contradicting the fact that Algorithm~\ref{alg:z_explore} exited the while loop without adding any element to $\Zexplore$. Thus our initial assumption must be wrong, proving the main claim about Algorithm \ref{alg:z_explore}.

\subparagraph{Proof for $\hatZunsafe \cap \Zsafe =\emptyset$.}

From the induction assumption, we have that $\hatZunsafe \cap \Zsafe =\emptyset$ in the beginning of this while loop. During this loop, $\hatZunsafe$ is modified by Algorithm~\ref{alg:z_goal} only when the call to Algorithm~\ref{alg:z_explore} returns with an empty $\Zexplore$. In such a case, $\hatZunsafe$ is modified by adding $\Zedge$ to it. Fortunately, from the above discussion, we know that when $\Zexplore$ is empty, $\Zedge$ contains no element from $\Zsafe$. Thus, $\hatZunsafe \cap \Zsafe = \emptyset$ even at the end of the while loop.

\end{proof}


Finally, we show that Algorithm~\ref{alg:z_goal} terminates in finite time ensuring that $\optZgoal$ has been updated in a way that all of it has been established to be safe, or $\Zexplore$ is non-empty. 

\begin{lemma} \label{lem:correctness_z_goal}
Algorithm \ref{alg:z_goal} terminates in $\text{poly}(|S|, |A|)$ time, after which either $\optZgoal \subset \hatZsafe$ or $\Zexplore \neq \emptyset$.
\end{lemma}

\begin{proof}
First, we show that, in Algorithm \ref{alg:z_goal}, whenever the condition $\optZgoal \subset \hatZsafe$ fails, the subsequently computed $Z_{\text{edge}}$ is non-empty.  Assume for the sake of contradiction that even though $\optZgoal \not\subset \hatZsafe$, $Z_{\text{edge}}$ is empty. Now, recall that $\optZgoal$ is the set of all state-action pairs visited starting from $s_0$ following $\optpigoal$ in the MDP $\optMgoal$, with transition probabilities $\opt{T}_{\text{goal}}$. 

Then, consider any trajectory $(s_0, a_0), (s_1, a_1), \hdots$ of non-zero probability under this corresponding policy and transition function. 
Since $s_0 \in \hatZsafe$ and $\Zedge$ is empty, $(s_0,a_0) \in \hatZsafe$. 
Then, since $\opt{T}_{\text{goal}} \in CI(\del{T})$, and since the confidence intervals are admissible, we have from Corollary~\ref{cor:closed_under_candidate} that $s_1 \in \hatZsafe$.
Since $\Zedge$ is empty, by a similar argument, we can establish that $(s_1, a_1) \in \hatZsafe$, and so on for all $(s_t, a_t)$. Since this holds for any trajectory under this policy and transition function, it would mean that $\optZgoal \subset \hatZsafe$, which is a contradiction. Thus, $Z_{\text{edge}}$ is indeed non-empty. 

Next, we show that Algorithm \ref{alg:z_goal} terminates in polynomial time.
Whenever Algorithm \ref{alg:z_goal} does not break out of a loop, then by the design of the algorithm, $\optZgoal \not\subset \hatZsafe$ and $\Zexplore = \emptyset$. In addition to this, $Z_{\text{edge}}$ must have been added to $\hatZunsafe$. Furthermore, from the above argument, since $\optZgoal \not\subset \hatZsafe$, $Z_{\text{edge}}$ must be non-empty. In other words, in each loop that does not break, we take a non-empty subset of $\optZgoal$, namely $Z_{\text{edge}}$ and add it to $\hatZunsafe$. Note that since we computed $\optZgoal$ in a way that it does not include any of $\hatZunsafe$, this also means that $Z_{\text{edge}} \cap \hatZunsafe \neq \emptyset$. Thus, by the end of this loop, we increase the cardinality of $\hatZunsafe$. Since $\hatZunsafe$ cannot be any larger than the finite quantity $|S| \times |A|$, we are guaranteed that no more than $O(|S| \times |A|)$ for loops are run when Algorithm~\ref{alg:z_goal}  is executed. (In fact, we can say something stronger:  no more than $O(|S| \times |A|)$ for loops are run, across multiple calls to Algorithm \ref{alg:z_goal} during the whole run of Algorithm~\ref{alg:ase}). 

Finally, observe that, by design of the algorithm, whenever the algorithm terminates, it must have broken out of the for loop. This is possible only if either $\optZgoal \subset \hatZsafe$ or $\Zexplore \neq \emptyset$, thus proving all of our claim.
\end{proof}



\subsection{Proofs about computing the set of state-actions along the goal path}

Recall that $\optZgoal$ is intended to be the set of state-action pairs visited by the optimistic goal policy under optimistic transitions. To compute this set, we need to know which states have a positive probability of being reached from $\sinit$ following $\optpigoal$, or equivalently the state-actions $(s, a)$ for which $\optrhogoal(s, a) > 0$. Formally, we must enumerate the set
\[
\{(s, a) \in S \times A : \optrhogoal(s, a) > 0\}.
\]

However, computing this as defined is not feasible in finite time (as we must enumerate infinite length trajectories). 
Instead, recall from Equation~\ref{eq:optzgoal} that we can approximate the above set by  only computing the finite-horizon estimate $\optrhogoal(\cdot, \cdot, H)$ for some horizon $H$.
Fortunately, computing $\optZgoal$ does not require a good estimate of $\optrhogoal(\cdot, \cdot)$, but only requires knowing when the $\optrhogoal(\cdot, \cdot)$ is positive or $0$.
Lemma~\ref{lem:rho_horizon_bound} shows that as long as $H \geq |S|$, $\optrhogoal(\cdot, \cdot, H) > 0$ if and only if $\optrhogoal(s, a) > 0$; as a corollary of which we have that $\optZgoal$ is exactly what we intend it to be.


\begin{lemma} \label{lem:rho_horizon_bound}
  For any policy $\pi$, starting state $s \in S$, and state-action pair $(s', a') \in S \times A$, $\rho_{\pi, s}^M(s', a', H) > 0$ if and only if $\rho_{\pi, s}^M(s', a') > 0$ as long as $H \geq |S|$. 
\end{lemma}
\begin{proof}

First we establish sufficiency. That is, if the finite-horizon estimate is positive, then, so is the infinite horizon estimate. Consider the following inequality relating these two quantities:
\begin{align*}
    \rho_{\pi, s}^M(s', a', H)
    &= (1-\gamma) \sum_{t=0}^{H} \gamma^{t} \Pr\left(s_{t}=s', a_{t}=a' | \pi, s_0 = s \right) \\
    &\leq \lim_{H \to \infty} (1-\gamma) \sum_{t=0}^{H} \gamma^{t} \Pr\left(s_{t}=s', a_{t}=a' | \pi, s_0 = s \right) \\
    &= \rho_{\pi, s}^M(s', a')
\end{align*}
The second step uses the fact that $\gamma > 0$ and $\Pr\left(s_{t}=s', a_{t}=a' | \pi, s_0 = s \right) \geq 0$ for all $t$.
Thus, if $\rho_{\pi, s}^M(s', a', H) > 0$, then $\rho_{\pi, s}^M(s', a') > 0$, establishing sufficiency.

Now we will establish necessity i.e., if the infinite horizon estimate was positive, then the same must hold for the finite-horizon estimate. If $\rho_{\pi, s}^M(s', a') > 0$, then there exists at least one sequence of states $(s_0, s_1, \ldots s_n)$ where $s_0 = s$, $s_n = s'$, and $T(s_i, \pi(s_i), s_{i+1}) > 0$ for all $0 \leq i < n$.
Without loss of generality, consider the shortest such sequence.
Now, for the sake of contradiction, assume that $n > |S|$.
By the pigeonhole principle, there exists at least one state that is repeated at least twice. That is, for two indices $j, k$ ($j < k$), we have $s_j = s_k$.
Since $s_j = s_k$ and $T(s_k, \pi(s_k), s_{k+1}) > 0$, then $T(s_j, \pi(s_j), s_{k+1}) > 0$.
Thus, we can construct a shorter sequence  by removing all indices $i$ such that $j < i \leq k$ and this sequence still satisfies the fact that every transition observed has non-zero probability. This contradicts our assertion that this is the shortest such sequence. Thus, $n \leq |S|$.
Given $H \geq |S| \geq n$, we have $\rho_{\pi, s}^M(s', a', H) > 0$, establishing necessity.
\end{proof}

As a straightforward corollary of the above, we have:

\begin{corollary}
\label{cor:optzgoal_correctness}
$\optZgoal$ as computed in Equation~\ref{eq:optzgoal} satisfies:
\[
\optZgoal = \{(s, a) \in S \times A : \optrhogoal(s, a) > 0\}.
\]
when $H \geq |S|$.
\end{corollary}

\subsection{Proofs about goal, explore, and switching policies} 
\label{sec:proofs:policies}
This subsection details the key lemmas for proving that ASE is PAC-MDP.
The main idea is to show that, under these policies, our agent either will perform the desired behavior, i.e. act $\epsilon$-optimally or reach a desired state-action set, or reach an insufficiently explored state-action pair (i.e., not in $K$).
Since a state-action pair that is experienced $m$ times is added to $K$, we can bound the number of times we reach a state-action pair outside of $K$.
With this bound and the following lemmas, we can bound the number of times the agent performs undesired behaviors, e.g. acting sub-optimally.
We start by proving this claim for $\optpiexplore$ and $\optpiswitch$ (Lemma \ref{lem:escape}), then for $\optpigoal$ (Lemma~\ref{lem:pigoal_escape} and Lemma \ref{lem:reduction_to_mbie}).

Recall that the $\optpiexplore$ is based on a reward system where the rewards are non-zero only on state-action pairs that are in $\hatZsafe$ and are not sufficiently explored (i.e., not in $K$).
We first show that if we were to follow the $\optpiexplore$ policy, in polynomially many steps, we are guaranteed to obtain a non-zero reward i.e., we are guaranteed to reach a state-action pair not in $K$.  In other words, if we were to follow $\optpiexplore$,  we will definitively obtain a useful sample and learn something new.
Similarly, if we were to follow $\optpiswitch$,  we will definitively return to $\optZgoal$ or learn something new.

\begin{lemma} \label{lem:escape}
Assume that the confidence intervals are admissible.
Consider an MDP $M^\dagger = \langle S, A, T, R^\dagger, \gamma^\dagger \rangle$, which is the same as the true MDP but with different rewards and discount factor. Let $Z$ be a closed, communicating subset of $S \times A$ and $Z^\dagger$ a non-empty subset of $Z$. Let $R^\dagger$ be defined such that:
\[
R^\dagger(s,a) = \begin{cases}
1 & (s,a) \in Z^\dagger \\
0 & (s,a) \in Z \setminus Z^\dagger \\
-\infty & (s,a) \notin Z
\end{cases}
\]
Let $H = \Hcom {\log \frac{16\Hcom}{c}}/{\log \frac{1}{c}}$  and let $\gamma^\dagger = c^{1/H}$ where $c \in (0,1/4]$ is a constant.
Let $\tilde{H} = \max\left(H \frac{1}{\sqrt{8c}},\frac{1}{\sqrt{\tau}} \right) $.
Let $\opt{Q}^\dagger$ denote the optimistic value function of this MDP, and $\optpi^\dagger$ be the optimistic policy i.e., $\optpi^\dagger(s,a) = \arg\max_{a \in A} \opt{Q}^\dagger(s,a)$.  

Then,  for any $\delta > 0$ and $\epsilon \in (0,c/8]$, starting from any state in $Z$ and following $\optpi^\dagger$, the agent will reach a state-action pair either in $Z^\dagger$ or outside of $K = \{(s, a) \in S \times A : n(s, a) \geq m \}$, where $m \geq O \left({\tilde{H}^4}|S|+ {\tilde{H}^4} \ln \frac{|S| | A | \tilde{H}^2}{\delta} \right)$, in at most $O(\frac{H^2}{c})$ time steps, with probability at least $1/2$, provided $\optpi^\dagger \in \Pi(Z)$.
\end{lemma}
\begin{proof}
Consider any $s \in Z$. We will first upper bound the optimistic value $\opt{V}^\dagger(s)$ and then  derive a lower bound on it, and then relate these two bounds together to prove our claim. Note that if we let $\opt{M}^\dagger$ be the same MDP as $M^\dagger$, but with the optimistic transitions (the transitions $\opt{T}^\dagger \in CI(\del{T})$ which maximize the optimistic Q-values), then $\opt{V}^\dagger(\cdot) = \opt{V}_{\opt{M}^\dagger}^{\optpi^\dagger}(\cdot)$. 

We will begin by upper bounding the finite-horizon value of $\optpi^\dagger$ on
$M^\dagger$ (and then relate it to its value on $\opt{M}^\dagger$).
Let $s_0, s_1, s_2, \hdots,$ denote the random sequence of states visited by the agent by following $\optpi^\dagger$ from $s_0=s$ on $M^\dagger$.   Then, we have:
\begin{align*}
V_{M^\dagger}^{\optpi^\dagger}(s, H)
&= \mathbb{E} \left [ \sum_{i=0}^{H} (\gamma^\dagger)^i R^\dagger(s_i, \optpi^\dagger(s_i))\right ] \\
&\leq \mathbb{E} \left [ \sum_{i=0}^{H}  R^\dagger(s_i, \optpi^\dagger(s_i))\right ]\\
&\leq \Pr(\exists \ i \leq H : (s_i, \optpi^\dagger(s_i)) \in Z^\dagger) \cdot H
\end{align*}
The first inequality follows from the fact that $\gamma^\dagger < 1$. The second inequality follows from the fact that every trajectory of $\optpi^\dagger$ that experiences a positive cumulative reward, must experience some state-action pair in $Z^\dagger$; and such a trajectory can at best experience a reward of $1$ at each timestep.

In the next step we will upper bound the finite-horizon optmistic value of following $\optpi^\dagger$ in $\opt{M}^\dagger$. To do this, define $M'$ to be an MDP that is identical to $M^\dagger$ on $(s,a) \in K$, and identical to $\opt{M}^\dagger$ everywhere else. Then,
\begin{align*}
V_{\opt{M}^\dagger}^{\optpi^\dagger}(s, H)
&\leq V_{M'}^{\optpi^\dagger}(s, H) + \frac{c}{8} \\
&\leq V_{M^\dagger}^{\optpi^\dagger}\left(s, H\right) + H \Pr(\exists \ i \leq H : (s_i, \optpi^\dagger(s_i)) \not \in K) + \frac{c}{8} \\
&\leq H \Pr(\exists \ i \leq H : (s_i, \optpi^\dagger(s_i)) \in Z^\dagger) +  H \Pr(\exists \ i \leq H : (s_i,\optpi^\dagger(s_i)) \not \in K) + \frac{c}{8} \\
&\leq 2 H \Pr(\exists \ i \leq H : (s_i, \optpi^\dagger(s_i)) \in Z^\dagger \cup K^c) + \frac{c}{8}.
\end{align*}
Here, the first inequality follows from Lemma \ref{lem:epsilon_close} and \ref{lem:size_of_m}.
Specifically, from Lemma \ref{lem:size_of_m}, we have when  $m \geq O \left({\tilde{H}^4}|S|+ {\tilde{H}^4} \ln \frac{|S| | A | \tilde{H}^2}{\delta} \right)$, the width of the confidence interval of $(s,a) \in K$ is at most $1/\tilde{H}^2 \leq c/(8H^2)$.
Then, Lemma~\ref{lem:epsilon_close} can be used to bound the difference in their value functions by $c/8$. 

Note that to apply Lemma~\ref{lem:epsilon_close} we must also ensure that for all $(s,a) \in K$, the support of the next state distribution is the same under $M'$ and $\opt{M}^\dagger$. To see why this is true, observe that for all $(s,a) \in K$, the confidence interval is at most $1/\tilde{H}^2 \leq \tau/2$.
Then, since the transition probabilities for $(s,a) \in K$ in $\opt{M}^\dagger$ and $M'$ correspond to $\opt{T}^\dagger \in CI(\del{T})$ and $T$ respectively, Lemma~\ref{lem:del-support} implies that the support of these state-action pairs are indeed the same for these two transition functions. Also note that Lemma~\ref{lem:epsilon_close} implies that these value functions are either close to each other or both equal to $-\infty$; even in the latter case, the above inequalities would hold (although, we will show in the remaining part of the proof that these quantities are lower bounded by some positive value).

The second inequality follows from Lemma \ref{lem:escape_prob}. Note that in order to apply 
Lemma \ref{lem:escape_prob}, we must establish that, with probability $1$, the agent  experiences only non-negative rewards, when it starts from $s$ and follows $\optpi^\dagger$ for $H$ steps. This is indeed true because we know that $Z$ is closed and $\optpi^\dagger \in \Pi(Z)$. Then, we know from Fact~\ref{fact:closed_policy} that the agent always remains in $Z$, which means it experiences only non-negative rewards.

The third inequality above uses the upper bound on $V_{M^\dagger}^{\pi}(s, H)$ that we derived previously.

Having established the above inequalities, we are now ready to upper bound the optimistic value using Lemma~\ref{lem:horizon_bound} as follows:
\begin{align*}
\opt{V}_{\opt{M}^\dagger}^{\optpi^\dagger}(s) & \leq \opt{V}_{\opt{M}^\dagger}^{\optpi^\dagger}(s, H) + \frac{(\gamma^{\dagger})^{H+1}}{1-\gamma^{\dagger}}  \\
&\leq  2 H \Pr(\exists \ i \leq H : (s_i, \optpi^\dagger(s_i)) \in Z^\dagger \cup K^c) + \frac{c}{8} +  \frac{(\gamma^{\dagger})^{H+1}}{1-\gamma^{\dagger}} 
\numberthis\label{eq:escape:upper-bound}
\end{align*}

Next, as the second part of our proof, we will derive a lower bound on the optimistic value using Assumption \ref{as:h_communicating}.
Choose some $(s', a') \in Z^\dagger$ (which is given to be non-empty).
Since $Z$ is given to be communicating, by Assumption 3, we know that 
there exists a policy $\picom$ which has a probability of at least $\frac{1}{2}$ of reaching $s'$ from $s$ in $\Hcom$ steps, while visiting only state-action pairs in $Z$. Without loss of generality, let us assume that $\picom(s') = a'$ (since, regardless of what the action at $s'$ is, it guarantees reachability of $s'$ from everywhere else.).

Let us lower bound the value of this policy. Let $s_0, s_1, s_2, \hdots,$ denote the random sequence of states visited by the agent by following $\picom$ from $s_0=s$ on $M^\dagger$.   Then:
\begin{align*}
V^{\picom}_{M^\dagger}(s) & \geq  
V^{\picom}_{M^\dagger}(s, \Hcom) \\
& = \mathbb{E} \left [ \sum_{i=0}^{\Hcom} (\gamma^\dagger)^i R^\dagger(s_i, \picom(s_i)) \right] \\
& \geq  (\gamma^\dagger)^{\Hcom} \mathbb{E} \left [ \sum_{i=0}^{\Hcom} R^\dagger(s_i, \picom(s_i)) \right]  \\
&\geq (\gamma^\dagger)^{\Hcom} \Pr (\exists \ i : (s_i,\picom(s_i)) \in Z^\dagger) \\
&\geq (\gamma^\dagger)^{\Hcom} \frac{1}{2}.
\end{align*}

Here, the first step follows from Fact~\ref{fact:closed_policy} which says that, since 
$\picom \in \Pi(Z)$ and $Z$ is closed, $\picom$ only visits state-action pairs in $Z$, all of which have non-negative $R^\dagger$ reward; as a result, truncating the value function to $\Hcom$ steps only maintains/decreases the value. 

The third step follows from the fact that $\gamma^\dagger < 1$. The fourth step
comes from the fact that, since $\picom$ takes only state-action pairs in $Z$, $R^\dagger(s_i, \picom(s_i)) \in \{0, 1\}$; then, every trajectory with a total non-zero reward has a reward of at least $1$. The last step follows from the guarantee of Assumption~\ref{as:h_communicating}.

Now, as the final step in our proof we note that if our confidence intervals are admissible, by Lemma \ref{lem:optimistic_q}, we know that $V_{\opt{M}^\dagger}^{\optpi^\dagger}(s) \geq V_{M^\dagger}^{*}(s)$. Furthermore, $V_{M^\dagger}^{*}(s)$ must be lower bounded by the value of $\picom$ on $M^\dagger$, which we just lower bounded. Now, equating this with the upper bound from Equation~\ref{eq:escape:upper-bound}, we get:
\begin{align*}
2 H \Pr(\exists \ i \leq H : (s_i, \optpi^\dagger(s_i)) \in Z^\dagger \cup K^c) + \frac{c}{8} +  \frac{(\gamma^{\dagger})^{H+1}}{1-\gamma^{\dagger}}  &\geq (\gamma^\dagger)^{\Hcom} \frac{1}{2} \\
\Pr(\exists \ i \leq H : (s_i, \optpi^\dagger(s_i)) \in Z^\dagger \cup K^c) &\geq \frac{1}{2H}\left( (\gamma^\dagger)^{\Hcom} \frac{1}{2} - \frac{c}{8} - \frac{(\gamma^{\dagger})^{H+1}}{1-\gamma^{\dagger}} \right) \\
& \geq \frac{1}{2H}\left( \frac{3c}{8}-  \frac{(\gamma^{\dagger})^{H+1}}{1-\gamma^{\dagger}} \right) \numberthis\label{eq:escape:prob-bound} \\
\end{align*}

Here, in the last step, we make use of the fact that  $\gamma^\dagger = c^{1/\Hcom}$. Next, we will upper bound the last term by making use of the inequality: if $c < 1$, then $\forall x > 0, c^{x} \leq 1- x(1-c)$. Then, we get: 
\begin{align*}
    \frac{1}{1-\gamma^\dagger} = \frac{1}{1-c^{1/\Hcom}} \leq \frac{\Hcom}{1-c}
\end{align*}
Furthermore, since $\log (16\Hcom/c) = (H/\Hcom) \log(1/c)$, by applying $\exp(\cdot)$ on both sides, we have:

\begin{align*}
\frac{c}{16\Hcom} =  c^{H/\Hcom}=  
 (\gamma^\dagger)^{H}
\end{align*}

From the above two inequalities, we have:
\[
\frac{(\gamma^\dagger)^{H+1}}{1-\gamma^\dagger} \leq \frac{(\gamma^\dagger)^{H}}{1-\gamma^\dagger} \leq \frac{c}{16(1-c)} \leq \frac{c}{8}
\]

In the first step above, we make use of $\gamma^\dagger < 1$ and in the second step, $c < 1/2$. Plugging this back in Equation~\ref{eq:escape:prob-bound}, we get:

  \begin{align*}
    \Pr(\exists \ i \leq H : (s_i, \optpi^\dagger(s_i)) \in Z^\dagger \cup K^c)\geq \frac{c}{8H} \\
  \end{align*}

  In other words, following policy $\optpi^\dagger$ for $H$ steps in $M^\dagger$, the agent will reach a state-action pair either in $Z^\dagger$ or outside $K$ with probability at least $c/(8H)$. Then, by the Hoeffding bound applied to multiple subsequent trajectories each of $H$ timesteps, we would have that with probability at least $1/2$, following policy $\optpigoal$ for $O(H \cdot \frac{H}{c})$ timesteps, the agent will reach an element either in $Z$ or outside $K$.

\end{proof}

In Lemma~\ref{lem:pigoal_escape}, we show that once our algorithm begins following the optimistic goal policy $\optpigoal$, it will continue to do so until it learns something new.
(Since we can bound the number of times the agent learns something new, observe that this also means that, eventually, the agent will follow $\optpigoal$ for all time.)

\begin{lemma}
\label{lem:pigoal_escape}
Assume the confidence intervals are admissible. Then, during the run of Algorithm~\ref{alg:ase}, if the agent is currently following $\optpigoal$, then it will continue to do so until it experiences a state-action pair outside $K = \{(s, a) \in S \times A : n(s, a) \geq m \}$ when $m \geq O\left( \frac{|S|}{\tau^2} + \frac{1}{\tau^2} \ln \frac{|S||A|}{\tau^2 \delta}\right)$.
\end{lemma}

\begin{proof}

Assume that during a run of Algorithm~\ref{alg:ase}, the agent is currently at $s$ and takes the action $\optpigoal(s)$. By design of Algorithm~\ref{alg:ase}, we have $s \in \optZgoal$. 
Recall that $\optZgoal$ is the set of all state-action pairs that can be visited by the agent if it were to following $\optpigoal$ starting from $s_0$ under the optimistic transitions $\optTgoal$. More formally, we have from Corollary~\ref{cor:optzgoal_correctness} that $\optZgoal =  \{(s, a) \in S \times A : \optrhogoal(s, a) > 0 \}$.

Now, to prove our claim, we only need to argue that if $(s, \optpigoal(s)) \in K$, then the next state $s'$ belongs to $\optZgoal$. Then, by design of Algorithm~\ref{alg:ase}, the agent will take $\optpigoal$ even in the next state, proving our claim.   
To argue this, observe that since $(s, \optpigoal(s)) \in K$ and $m \geq O\left( \frac{|S|}{\tau^2} + \frac{1}{\tau^2} \ln \frac{|S||A|}{\tau^2 \delta}\right)$, by Lemma~\ref{lem:size_of_m}, the confidence interval of $(s, \optpigoal(s))$ has width at most $\tau/2$. Then, from Lemma~\ref{lem:del-support}, since $T(s, \optpigoal(s), s') > 0$ and since $\optTgoal \in CI(\del{T})$,  we have $\optTgoal(s, \optpigoal(s), s') > 0$. Thus, since we know that $s \in \optZgoal$, by definition of $\optZgoal$,  $s'$ should also belong to $\optZgoal$. 
\end{proof}

In the following lemma, we show that in the MDP $\Mgoal$ (which is the same as the original MDP but with the unsafe state-action pairs set to $-\infty$ rewards), when we follow the optimistic goal policy $\optpigoal$, we either take a near-optimal action (with respect to $\Mgoal$) or we experience an action outside of  $K$ with sufficient probability in the next $H$ steps.

\begin{lemma} \label{lem:reduction_to_mbie}
Assume the confidence intervals are admissible. 
Consider any instant when the agent has taken a trajectory $p_t$ and is at state $s_t$, and the Algorithm~\ref{alg:ase} instructs the agent to follow $\optpigoal$.
Let $\Pr(A_M)$ be the probability that starting at this step, the Algorithm~\ref{alg:ase} leads the agent out of $K =\{(s,  a) \in S \times A : n(s, a) \geq m \}$ in $H$ steps, conditioned on $p_t$.
Then, for any $\epsilon, \delta \in (0, 1)$, and for $H=O\left(\frac{1}{1-\gamma} \ln \frac{1}{\epsilon (1-\gamma)}\right)$ and $m \geq O\left( \frac{1}{\epsilon^{2}(1-\gamma)^{4}} \left ( \frac{|S|}{\tilde{\epsilon}}+ \frac{1}{\tilde{\epsilon}^2}\ln \frac{ |S||A|}{\tilde{\epsilon}^2 \delta} \right) \right)$, where $\tilde{\epsilon} =O\left( \min\left(\frac{\tau}{2}, \frac{\epsilon (1-\gamma)^2}{3}\right)  \right)$, we have:
\[  V_{M}^{\mathcal{A}}(p_t) \geq V_{\Mgoal}^*(s_t) - \frac{\epsilon}{2} - 2\frac{\Pr(A_M)}{1-\gamma}. \]
\end{lemma}
\begin{proof}
While we follow the general outline of the proof of Theorem 1 from \citet{strehl2008analysis}, we note that there are crucial differences for incorporating safety (such as dealing with rewards of $-\infty$).

At the outset, we establish two useful inequalities. 
First, since $H=O\left(\frac{1}{1-\gamma} \ln \frac{1}{\epsilon(1-\gamma)}\right)$, by Lemma \ref{lem:horizon_bound}, we have that:
\begin{equation} \label{eq:epsilon_optimal:truncate}
 V_{\optMgoal}^{\optpigoal}(s_t, H)\geq V_{\optMgoal}^{\optpigoal}(s_t) - \frac{\epsilon}{3}
\end{equation}

Secondly, let $\Mgoal'$ be an MDP that is equivalent to $\Mgoal$ for all state-action pairs in $K$ and equal to $\optMgoal$ otherwise. (Note that all these MDPs have the same reward function, namely $\Rgoal$.)
We claim that for all $a$:
%
\begin{equation} \label{eq:epsilon_optimal:finite-approximate}
Q_{\Mgoal'}^{\optpigoal}(s_t, a, H) \geq Q_{\optMgoal}^{\optpigoal}(s_t, a, H) - \frac{\epsilon}{3}
\end{equation}

Let us see why this inequality holds.
From Lemma \ref{lem:size_of_m}, we have when $m \geq O\left( \frac{1}{\epsilon^{2}(1-\gamma)^{4}} \left ( \frac{|S|}{\tilde{\epsilon}}+ \frac{1}{\tilde{\epsilon}^2}\ln \frac{ |S||A|}{\tilde{\epsilon}^2 \delta} \right) \right)$, the width of the confidence interval of $(s,a) \in K$ is at most $O(\tilde{\epsilon}) = \frac{\epsilon}{3} \frac{(1-\gamma)^2}{\gamma}$. Then, 
Lemma~\ref{lem:epsilon_close} can be used to bound the difference in the value functions by $\epsilon/3$ as above.

Note that to apply Lemma~\ref{lem:epsilon_close},  we must also ensure that for all $(s,a) \in K$, the support of the next state distribution is the same under $\Mgoal'$ and $\optMgoal$. To see why this is true, observe that for all $(s,a) \in K$, the width of the confidence interval is at most $O(\tilde{\epsilon}) = \tau/2$. Then, since the transition probability for $(s,a) \in K$ in $\opt{M}^\dagger$ and $M'$ correspond to $\opt{T}^\dagger \in CI(\del{T})$ and $T$ respectively, Lemma~\ref{lem:del-support} implies that the support of these state-action pairs are indeed the same for these two transition functions. Also note that Lemma~\ref{lem:epsilon_close} implies that these value functions are either $\epsilon_1$-close to each other or both equal to $-\infty$; even in the latter case, the above inequality would hold (although, to be precise, this latter case does not really matter since Theorem~\ref{thm:pac_and_safe} eventually shows that these quantities are lower bounded by a positive quantity).\\

Having established the above inequalities, we now begin lower bounding the value of the algorithm.
First, since the rewards $R(\cdot, \cdot)$ are bounded below by $-1$, and since $H=O\left(\frac{1}{1-\gamma} \ln \frac{1}{\epsilon (1-\gamma)}\right)$, by Lemma~\ref{lem:horizon_bound}, we can lower bound the infinite-horizon value of the algorithm by its finite-horizon value as
\begin{equation}\label{eq:epsilon_optimal:algorithm-truncate}
  V_{M}^{\mathcal{A}}(p_t) \geq V_{M}^{\mathcal{A}}(p_t, H) - \frac{\epsilon}{3}.
\end{equation}

Next, we claim to bound the value of following the algorithm for the next $H$ steps as follows:
%
%
\begin{align*} 
  V_{M}^{\mathcal{A}}(p_t, H) & \geq V_{\Mgoal^{\prime}}^{\optpigoal}(s, H) - 2\frac{\Pr\left(A_{M}\right) }{1-\gamma} \numberthis\label{eq:epsilon_optimal:escape}.
\end{align*}


In other words, we have lower bounded the value of following the algorithm on $M$, in terms of following $\optpigoal$ on $\Mgoal'$. Let us see why this is true. For the sake of convenience, let us define two cases corresponding to the above inequality: Case A, where the agent follows Algorithm~\ref{alg:ase} to take actions for $H$ steps starting from $s_t$ in MDP $M$ and Case B, where the agent follows the fixed policy $\optpigoal$ to take actions for $H$ steps starting from $s_t$ in MDP $\Mgoal'$. 

Now, recall that we are considering a state $s_t$ where Algorithm~\ref{alg:ase} currently follows $\optpigoal$, and will continue to follow it until it takes an action in $K^c$. Then, consider a particular random seed for which, in Case A, the agent does not reach $K^c$. 

We argue that for this random seed, the agent will see the same cumulative discounted reward over $H$ steps in both Case A and Case B. To see why, recall that $\Mgoal'$ and $M$ share the same transition functions on $K$. Next, since in Case A, the agent does not escape $K$, by Lemma~\ref{lem:pigoal_escape}, we know that the agent follows only $\optpigoal$ for $H$ steps. Thus in both these cases, the agent experiences the same sequence of state-action pairs for $H$ steps. It only remains to argue that these state-action pairs have the same rewards in both cases. To see why, recall that by design of Algorithm~\ref{alg:ase}, since the agent does not escape $K$, all these state-action pairs would belong to $\optZgoal$ which in turn is a subset of $\hatZsafe$. Now the MDP $M$ and $\Mgoal'$ share the same rewards on $\hatZsafe$; this is because, their rewards differ only in $\hatZunsafe$, and we know $\hatZunsafe \cup \Zsafe = \emptyset$ (Lemma~\ref{lem:correctness_z_explore}).
Thus, in both cases, the agent experiences the same sequence of rewards. 

As a result, the value functions in Case A and B differ only due to trajectories where the algorithm either leads the agent to $K^c$ in $H$ steps or leads the agent to negative rewards any time in the future. Hence, we can upper bound the difference $V_{\Mgoal^{\prime}}^{\optpigoal}(s_t, H) - V_{M}^{\mathcal{A}}(p_t, H)$ in terms of  $\Pr(A_{M})$ multiplied by the maximum difference between the respective cumulative rewards. We know that the cumulative reward in Case A is at least $-1/(1-\gamma)$, because the rewards $R(\cdot, \cdot)$ are bounded below by $-1$. On the other hand, in Case B, the cumulative reward experienced is at most $1/(1-\gamma)$, assuming the agent receives a reward of $1$ for each step. From this, we establish Equation~\ref{eq:epsilon_optimal:escape}.  

Subsequently, we further lower bound $V_{\Mgoal^{\prime}}^{\optpigoal}(s_t, H)$ as follows:
\begin{align*}
  V_{\Mgoal'}^{\optpigoal}(s_t, H) & \geq V_{\optMgoal}^{\optpigoal}\left(s_t, H\right) - \frac{\epsilon}{3} \\
   & \geq V_{\optMgoal}^{\optpigoal}\left(s_t\right) - 2\frac{\epsilon}{3}\\
   & \geq V_{\Mgoal}^{*}\left(s_t\right) - \epsilon
\end{align*}
Here, the first inequality comes from Equation~\ref{eq:epsilon_optimal:finite-approximate}. 
The second inequality comes from Equation~\ref{eq:epsilon_optimal:truncate}.  The last step comes from Lemma \ref{lem:optimistic_q} (given admissibility).
Then, by combining the above inequality with Equations~\ref{eq:epsilon_optimal:escape} and~\ref{eq:epsilon_optimal:algorithm-truncate}, we get our final result.
\end{proof}



\subsection{Supporting lemmas for showing PAC-MDP}

The following lemmas are necessary for proving our algorithm is PAC-MDP.
Note that most of these lemmas are similar to lemmas from \citet{strehl2008analysis}.
However, because we construct MDPs with infinitely negative rewards, additional care must be taken to ensure that these properties still hold.

Here we provide a quick description of the lemmas detailed in this section.
We start with Lemma~\ref{lem:epsilon_close}, which shows that if two MDPs have sufficiently similar transition functions, the optimal $Q$-values on these two MDPs must also be similar.
This, together with Lemma~\ref{lem:size_of_m}, allows us to show that, if we have sufficiently explored the state-space, we can accurately estimate the optimal policy on any MDP.
Next we show, in Lemma~\ref{lem:escape_prob}, that the difference between the value functions of the true and an estimated MDP for a given policy is proportional to the probability of reaching an under-explored state-action pair, i.e. a state-action pair outside of $K$.
This allows us to claim that either the probability of reaching an element outside of $K$ is sufficiently large, or our estimated value function is sufficiently accurate.
Lemma~\ref{lem:horizon_bound} bounds the difference between the finite horizon and infinite horizon value functions, allowing us to consider only finite length trajectories. 
Lemma~\ref{lem:optimistic_q} simply shows that our optimistic value function always over-estimates the true value function (given that our confidence intervals are admissible).

\begin{lemma} \label{lem:epsilon_close}
  Let $M_1 = \langle S, A, T_1, R^\dagger, \gamma^\dagger \rangle$ and $M_2 = \langle S, A, T_2, R^\dagger, \gamma^\dagger \rangle$ be two MDPs with identical rewards that either belong to $[0,1]$ or equal $-\infty$ and $\gamma^\dagger<1$. Let $K$ be a subset of state-action pairs such that 
  \begin{enumerate}
  \item for all $(s,a) \notin K$, $T_1(s, a, \cdot) = T_2(s, a, \cdot)$,
  \item for all $(s,a) \in K$, $\|T_1(s,a,\cdot) - T_2(s,a,\cdot) \|_1 \leq \beta$ and
  \item for all $(s,a) \in K$, the next state distribution $(s,a)$ has identical support under both $T_1$ and $T_2$.
  \end{enumerate}
  Then, for any (stationary, deterministic) policy $\pi$, and for any $(s,a)$ and any $H \geq 0$, we have that, either:
  $$ \left|Q_{M_1}^{\pi}(s, a, H)-Q_{M_2}^{\pi}(s, a, H)\right| \leq \min \left( \frac{\gamma^\dagger \beta}{(1-\gamma^\dagger)^{2}}, \beta H^2 \right).$$
  or
    $$Q_{M_1}^{\pi}(s, a, H)=Q_{M_2}^{\pi}(s, a, H) = -\infty.$$

\end{lemma}

\begin{proof}
First, we note that for any $(s,a)$ such that $R^\dagger(s,a) = -\infty$, $Q_{M_1}^{\pi}(s, a, H) = Q_{M_2}^{\pi}(s, a, H) = -\infty$
 for all $H$. Hence, for the rest of the discussion, we will consider $(s,a)$ such that $R^\dagger(s,a) \neq -\infty$.

We prove our claim by induction on $H$. For $H=1$, for all $(s,a)$, $Q_{M_1}^{\pi}(s, a, H) = Q_{M_2}^{\pi}(s, a, H)= R^\dagger(s,a)$.
  
Consider any arbitrary $H$.
First, we show that, if there exists a next state $s'$ in the support of $T_1(s, a, \cdot)$ such that $Q_{M_1}^\pi(s',\pi(s'),H-1) = -\infty$, then $Q_{M_1}^{\pi}(s, a, H) = Q_{M_2}^{\pi}(s, a, H) = -\infty$.
Note that, by conditions 1 and 3 of the Lemma statement, we have that for all $(s,a)$, regardless of whether in $K$ or not, the support of the next state distribution is identical between $T_1$ and $T_2$. Hence, if there exists a next state $s'$ in the support of $T_1$ such that $ Q_{M_1}^\pi(s',\pi(s'),H-1) = -\infty$, then $s'$ would belong even to the support of $T_2$ and, thus, we would have that $ Q_{M_2}^\pi(s',\pi(s'),H-1) = -\infty$.
Hence, by definition of Q-values, we would have $Q_{M_1}^{\pi}(s, a, H) = Q_{M_2}^{\pi}(s, a, H) = -\infty$.

Now consider a case where none of the next states $s'$ in the support of $T_1$ (and $T_2$, without loss of generality) have Q-value $Q_{M_1}^\pi(s',\pi(s'),H-1) = -\infty$. We will prove by induction that in this case, 
\[ \left|Q_{M_1}^{\pi}(s, a, H)-Q_{M_2}^{\pi}(s, a, H)\right| \leq \frac{\gamma^\dagger \beta}{1-\gamma^\dagger} \cdot \frac{1-\gamma^{\dagger H}}{1-\gamma^\dagger}.\]

Note that when $H=0$, the right hand side above resolves to zero, which is indeed true. 

Consider any $H>0$. Now, observe that for all $(s,a)$, regardless of whether in $K$ or not, we have $\|T_1(s,a,\cdot) - T_2(s,a,\cdot) \|_1 \leq \beta$. Besides, since $R^\dagger(s,a) \neq -\infty$, we can upper bound the difference in the $Q$ values as follows:
\begin{align*}
  &|Q_{M_1}^\pi(s,a,H) - Q_{M_2}^\pi(s,a, H)| \\
  &= \gamma^\dagger \left| \sum_{s'} T_1(s,\pi(s'),s') Q_{M_1}^\pi(s',\pi(s'),H-1) -  \sum_{s'} T_2(s,\pi(s'),s')  Q_{M_2}^\pi(s',\pi(s'),H-1)   \right| \\
  & \leq \gamma^\dagger \left| \sum_{s'} T_1(s,\pi(s'),s') \left( Q_{M_1}^\pi(s',\pi(s'),H-1) -  Q_{M_2}^\pi(s',\pi(s'),H-1)   \right)\right| \\ 
  & \,\,\, + \gamma^\dagger \left| \sum_{s'} \left(T_1(s,\pi(s'),s') -  T_2(s,\pi(s'),s')\right)  Q_{M_2}^\pi(s',\pi(s'),H-1)   \right| \\
  & \leq \gamma^\dagger \left(  \frac{\gamma^\dagger \beta}{1-\gamma^\dagger} \cdot \frac{1-\gamma^{\dagger H-1}}{1-\gamma^\dagger}\right) + \gamma^\dagger \frac{\beta}{1-\gamma^\dagger} \\
  &= \frac{\gamma^\dagger\beta}{1-\gamma^\dagger} \left( \frac{\gamma^\dagger(1-\gamma^{\dagger H-1})}{1-\gamma^\dagger}+1\right) \\
  &= \frac{\gamma^\dagger \beta}{1-\gamma^\dagger} \cdot \frac{1-\gamma^{\dagger H}}{1-\gamma^\dagger}
\end{align*}

Here, the second step follows by a simple algebraic rearrangement that decomposes the difference in the Q-values, in terms of the difference in the transitions and the difference in the next-state Q-values. 
In the third step, for the first term, we make use of the fact that in this case, the next state $s'$ has a Q-value that is not $-\infty$; this means that $(s', \pi(s'))$ does not have any next states with Q-value $-\infty$  and therefore the induction assumption holds. By applying the induction assumption for $H-1$, we get the first term. For the second term, we make use of the fact that the total sum of transition probabilities equals $1$. Furthermore, we also make use of the fact the maximum magnitude of the Q-value is at most $\frac{1}{1-\gamma^\dagger}$ if it is not $-\infty$; this is because, if the Q-value is not $-\infty$, it is a discounted summation of expected rewards that lie between $0$ and $1$.

Hence, our induction hypothesis is true. Our main upper bound can then be established by noting that $1-\gamma^{\dagger H} < 1$.

To prove our other upper bound, we consider the induction hypothesis:
\[ \left|Q_{M_1}^{\pi}(s, a, H)-Q_{M_2}^{\pi}(s, a, H)\right| \leq \beta H^2.\]
Then, in the third step above, we would instead have:
\begin{align*}
  |Q_{M_1}^\pi(s,a,H) - Q_{M_2}^\pi(s,a, H)| 
  & \leq \gamma^\dagger \left(  \beta (H-1)^2 \right)     + \gamma^\dagger \beta H   \\
  & \leq  \beta (H-1)^2 + \beta H \\
  & \leq \beta H^2.
\end{align*}
  
To get the first term on the right hand side, we again make use of the induction assumption. For the second term, we simply upper bound the sum of the maximum discounted rewards to be $H$. Finally, we make use of the fact that $\gamma^\dagger < 1$ and $H^2 - (H-1)^2 \geq 2H-1 \geq H$ when $H>0$.
\end{proof}

\begin{lemma} \label{lem:size_of_m}
Suppose that as input to Algorithm~\ref{alg:ase}, we set $\delta_T = \delta / (2 |S| |A| m)$ and all confidence intervals computed by our algorithm are admissible.
   Then, for any $\beta > 0$, there exists an $m=O \left(\frac{|S|}{\beta^{2}}+\frac{1}{\beta^{2}} \ln \frac{|S| |A|}{\beta \delta} \right)$ such that $\| \hat{T}(s, a, \cdot) = T(s, a, \cdot) \|_1 \leq \beta$ holds for all state-action pairs $(s, a)$ that have been experienced at least $m$ times.
\end{lemma}
\begin{proof}
  See Lemma 5 from \citet{strehl2008analysis}.
\end{proof}

\begin{lemma} \label{lem:escape_prob}
   Let $M^\dagger = \langle S, A, R^\dagger, T, \gamma^\dagger \rangle$ be an MDP that is the same as the true MDP $M$, but with arbitrary rewards $R^\dagger$ (bounded above by $1$) and discount factor $\gamma^\dagger < 1$. Let $K$ be some set of state-action pairs. Let $M' = \langle S, A, R^\dagger, T', \gamma^\dagger \rangle$ be an MDP such that $T'$ is identical to $T$ on all elements inside $K$.  
   Consider a policy $\pi$. Let $A_{M^\dagger}$ be the event that a state-action pair not in $K$ is encountered in a trial generated by starting from state $s_1$ and following $\pi$ for $H$ steps in $M^\dagger$ (where $H$ is a positive constant).   If, with probability $1$, the agent starting at $s_1$ and following $\pi$ in $M^\dagger$ will receive only non-negative rewards then, 
   $$ V_{M}^{\pi}\left(s_{1}, H\right) \geq V_{M^{\prime}}^{\pi}\left(s_{1}, H\right)- \min\left(\frac{1}{(1 - \gamma^\dagger)}, H \right) \Pr(A_{M}).$$
\end{lemma}
\begin{proof}
  This proof is similar to that of Lemma 3 from \citet{strehl2008analysis}.
  The only aspect we need to be careful about is the magnitude of the rewards.
  
  The two MDPs $M^\dagger$ and $M'$ differ only in their transition functions, and moreover, only outside the set $K$. Then, observe that, for a fixed random seed, if the agent were to follow $\pi$
  starting from $s_1$, it would receive the same cumulative reward for $H$ steps in both $M^\dagger$ and $M'$, if it remained in $K$ for all those $H$ steps. In other words, the value of $\pi$ in these two MDPs differs only because of those random seeds which led the agent out of $K$ in $M^\dagger$. 
  Thus, the difference $V_{M^{\prime}}^{\pi}\left(s_{1}, H\right) - V_{M^\dagger}^{\pi}\left(s_{1}, H\right)$
  cannot be any larger than the respective cumulative rewards. Our claim then follows by lower bounding the cumulative reward in $M^\dagger$ and upper bounding the cumulative reward in $M'$. More concretely, note that cumulative reward in $M^\dagger$ can not be any lower than zero, since we assume that the agent receives rewards bounded in $[0,1]$. On the other hand, in $M'$, the agent can receive a reward of $1$ in all $H$ steps, so the value function can be upper bounded by $\min\left(\frac{1}{(1 - \gamma^\dagger)}, H \right)$.


\end{proof}

\begin{lemma} \label{lem:horizon_bound}
   Consider an MDP $M^\dagger = \langle S, A, T^\dagger, R^\dagger, \gamma^\dagger\rangle$ with rewards bounded above by $1$, and a stationary or non-stationary policy $\pi$ and state $s$. Then, for any $H \geq 0$, we have:
   \[
  V_{M^\dagger}^{\pi}(s, H) \geq V_{M^\dagger}^{\pi}(s) - \frac{(\gamma^{\dagger})^{H+1}}{1-\gamma^{\dagger}}.
  \]
  As a corollary of this, for $H \geq \frac{1}{1-\gamma^\dagger} \ln \frac{1}{\epsilon(1-\gamma^\dagger)}$, we have:  \[
  V_{M^\dagger}^{\pi}(s, H) \geq V_{M^\dagger}^{\pi}(s) - \epsilon.
  \]
  By the same argument, in the case that the rewards are bounded below by $-1$,
   \[
  V_{M^\dagger}^{\pi}(s, H) \leq V_{M^\dagger}^{\pi}(s) + \epsilon.
  \]
\end{lemma}
\begin{proof}
  This proof follows the proof of Lemma 2 from \citet{kearns2002near}.
  
  
  Observe that by truncating any trajectory to $H$ steps, the cumulative discounted reward for this trajectory can drop by a value of at most:
  \[
  \sum_{t=H+1}^{\infty}(\gamma^\dagger)^t = \frac{(\gamma^{\dagger})^{H+1}}{1-\gamma^{\dagger}},
  \]
  which happens when it receives a reward of $1$ at every time step after $H$. Thus for any $H$ such that the above quantity is lesser than or equal to $\epsilon$,  we will have $V^{\pi}(s, H) \geq V^{\pi}(s) - \epsilon$. This is indeed true for $H \geq \frac{1}{1-\gamma^\dagger} \ln \frac{1}{\epsilon(1-\gamma^\dagger)}$.
\end{proof}

\begin{lemma} \label{lem:optimistic_q}
  Suppose that all confidence intervals are admissible. Let $M^\dagger= \langle S, A, T, R^\dagger, \gamma^\dagger\rangle$ be the same MDP as $M$ except with arbitrary rewards that are upper bounded by $1$ and discount factor $\gamma^\dagger < 1$. Let $\optpi^\dagger$ denote the optimal policy i.e., $\forall s$, $\optpi^\dagger(s) = \arg\max_{a \in A} \opt{Q}^\dagger(s,a)$. Let $\opt{M}^\dagger$ denote the optimal MDP. Then, for all $H \geq 0$ and for all $(s,a)$, we have:
  \[
   V^{\optpi^\dagger}_{\opt{M}^{\dagger}}(s) \geq  V^{*}_{M^\dagger}(s)
  \]
  
\end{lemma}

\begin{proof}
  See Lemma 6 from \citet{strehl2008analysis}.
  Note that this proof also applies to MDPs with negative rewards.
  
  
  
  
\end{proof}

\section{Experiment Details} \label{appendix:experiment}

\begin{figure}
    \centering
    \includegraphics[width=0.7\textwidth]{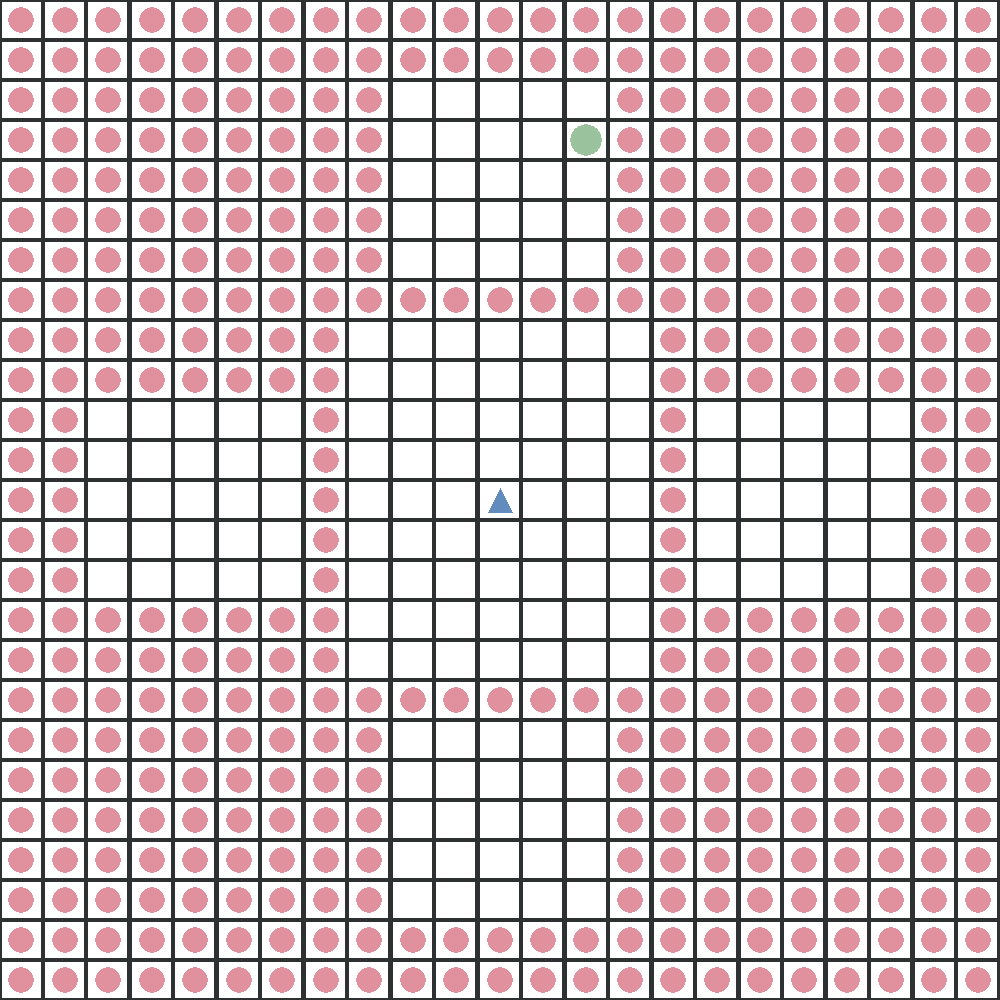}
    \caption{
    Full map of the Unsafe Grid World environment.
    The green circle marks the goal, the blue triangle marks the initial location of the agent $\sinit$, and red circles correspond to dangerous states.
    }
    \label{fig:grid_world_map}
\end{figure}

\subsection{Unsafe Grid World}

The first domain we consider is a grid world domain with dangerous states, where the agent receives a reward of $-1$ for any action and the episode terminates.
The agent starts on a $7 \times 7$ island of safe states and is surrounded by four $5 \times 5$ islands of safe states in all four directions, separated from the center island by a one-state-thick line of dangerous states (see Figure~\ref{fig:grid_world_map}).
The goal is placed on one of the surrounding islands.
The agent can take actions up, down, left, or right to move in those directions one step, or can take actions jump up, jump down, jump left, or jump right to move two steps, allowing the agent to jump over dangerous states.
There is a slipping probability of $60\%$, which causes the agent to fall left or right of the intended target ($30\%$ for either side).

The initial safe set provided to the agent is the whole center island (except for the corners) and all actions that with probability $1$ will keep the agent on the center island.
The distance function $\Delta$ provided to the agent is $\Delta((s, a), (\tilde{s}, \tilde{a})) = 0$ if $a = \tilde{a}$ and $s$ and $\tilde{s}$ are within $5$ steps from each other (in $L_\infty$ norm) and $\Delta((s, a), (\tilde{s}, \tilde{a})) = 1$ otherwise.
The analogous state function $\alpha$ is simply $\alpha((s, \cdot, s'), (\tilde{s}, \cdot)) = (x_{s'} + (x_{\tilde{s}} - x_{s}), y_{s'} + (y_{\tilde{s}} - y_{s}))$, where the subscripts denote the state to which the attribute belongs.

\subsection{Discrete Platformer}

We also consider a more complicated discrete Platformer domain.
The states space consists of tuples $(x, y, \dot{x}, \dot{y})$ where $x, y$ are the coordinates of the agent and $\dot{x}, \dot{y}$ are the directional velocities of the agent.
The actions provided to the agent are the tuple $(\dot{x}_{\text{desired}}, j)$ where $\dot{x}_{\text{desired}}$ is the desired $\dot{x}$ and ranges from $-2$ to $2$, and $j$ is a boolean indicating whether or not the agent should jump.
While on the ground, at every step $\dot{x}$ changes by at most $1$ in the direction of $\dot{x}_{\text{desired}}$ and $\dot{y} \in \{1, 2\}$ if $j = 1$ (otherwise $\dot{y}$ remains unchanged).
While in the air, however, the agent's actions have no effect and gravity decreases $\dot{y}$ by one at every step.
When the agent returns to the ground, $\dot{y}$ is set to $0$.

There are three types of surfaces in the environment: 1) concrete, 2) ice, and 3) sand.
These surfaces change how high the agent can jump.
On concrete, when the agent jumps, $\dot{y} = 2$ with probability $1$; on ice $\dot{y} = 2$ with probability $0.5$ and $\dot{y} = 1$ with probability $0.5$; and on sand $\dot{y} = 1$ with probability $1$.

The environment is arranged into three islands.
The first island has all three surface materials from left to right: sand, ice, then concrete.
The next two islands are just concrete, with the last one containing the goal state (where the reward is $1$).
The regions surrounding these island are unsafe, meaning they produce rewards of $-1$ and are terminal.
The islands are spaced apart such that the agent must be on concrete to make the full jump to the next islands (and visa versa).

The initial safe set provided to the agent is the whole first island and all actions that with probability $1$ will keep the agent on the center island.
The distance function $\Delta$ provided to the agent is $\Delta((s, a), (\tilde{s}, \tilde{a})) = 0$ if $a = \tilde{a}$ and $s$ and $\tilde{s}$ are either both in the air or both on the same type of surface and $\Delta((s, a), (\tilde{s}, \tilde{a})) = 1$ otherwise.
The analogous state function $\alpha$ is simply $\alpha((s, \cdot, s'), (\tilde{s}, \cdot)) = \tilde{s'}$ where $\tilde{s'}$ has the same $y$, $\dot{x}$, and $\dot{y}$ values as $s'$ with the $x$ value shifted by the $x$ difference between $s$ and $\tilde{s}$.

\subsection{Baselines}

Here we describe the details for the baselines we compare against.
We note that all these baselines make use of the distance metric and analogous state function to transfer information between different states, just like our algorithm.
For all of our ``unsafe'' algorithms, we set all negative rewards to be very large to ensure that they converged to the safe-optimal policy.
To improve the runtime of the experiments, the value functions and safe sets are only re-computed every $100$ time steps.

\paragraph{MBIE}
MBIE \citep{strehl2008analysis} is a guided exploration algorithm that always follows a policy that maximizes an optimistic estimate of the optimal value function.
As noted above, one of the motivations of our method was to construct a safe version of MBIE.

\paragraph{R-Max and Safe R-Max}
The next algorithm we compare against is R-Max \citep{brafman2002r}.
This algorithm sets the value function for all state-action pairs that have been seen fewer than $m$ times (for some integer $m$) to be equal to $V_{\max}$, the maximum value the agent can obtain.
In order to ensure that all states are sufficiently explored and still make use of the analogous state function, we set the value of any state-action pair, $(s, a)$, to $V_{\max} = 1$ (since all goal states are terminal) only if there is a state-action pair similar to $(s, a)$ with a transferred confidence interval length greater than some $\epsilon' > 0$.
In mathematical terms, all state-action pairs in $\{(s, a) \in S \times A : \exists (\tilde{s}, \tilde{a}) \in S \times A \text{ where } \del{\epsilon}_T(s,a) < \epsilon' \text{ and } \Delta((s, a), (\tilde{s}, \tilde{a})) < \tau/2 \}$ are set to $V_{\max}$.
Clearly, this requires at most every state to be explored $m$ times, but in most cases decreases the number of times each state-action pair needs to be explored.
In our experiments we set $\epsilon' = \tau/2$ to give R-Max the most generous comparison against ASE, since ASE requires that a state-action only have a confidence interval of $\tau/2$ before it can be marked as safe.
However, note that in many problems $\tau/2$ may be much larger than the desired confidence interval.

Our safe modification of this algorithm, ``Safe R-Max,'' simply restricts the allowable set of actions the agent can take to $\hatZsafe$.

\paragraph{$\epsilon$-greedy and Safe $\epsilon$-greedy}
Another classic algorithm we compare against is $\epsilon$-greedy.
This algorithm acts according to the optimal policy over its internal model at every time step with probability $1 - \epsilon$ and with probability $\epsilon$ the agent takes a random action.
For our experiments we anneal $\epsilon$ between $1$ and $0.1$ for the first $N$ number of steps ($N = \num{5000}$ for the unsafe grid world and $N = \num{20000}$ for the discrete platformer game).
Our safe modification of this algorithm, ``Safe $\epsilon$-greedy,'' simply restricts the allowable set of actions the agent can take to $\hatZsafe$.

\paragraph{Undirected ASE}
We also compare against a modified version of our algorithm ``Undirected ASE.''
This modification changes Algorithm~\ref{alg:z_goal} such that $Z_{\text{edge}} \gets \{(s,a) \in \hatZsafe^c \; | s \in \hatZsafe \}$, removing the use of $\optZgoal$. 
With this change, ``Undirected ASE'' simply tries to expand the safe set in all directions, instead of only along the direction of the optimistic goal policy.
This baseline is to illustrate the efficacy of using our directed exploration method.

\end{document}